\documentclass[letterpaper]{article} 
\usepackage{geometry}
\usepackage{natbib}
\usepackage{amsmath,amsfonts}
\usepackage{amsthm}
\usepackage{mathtools}
\usepackage{graphicx}
\usepackage{float}
\usepackage[colorlinks=true, citecolor=blue, linkcolor = red]{hyperref}
\usepackage{enumerate}
\usepackage{booktabs}
\usepackage{subcaption}
\usepackage[ruled,vlined]{algorithm2e}
\usepackage{xparse}
\usepackage[dvipsnames]{xcolor}
\usepackage{enumitem}
\usepackage{tikz}
\usepackage{bm}
\usepackage{authblk}
\usetikzlibrary{arrows.meta}
\usetikzlibrary{positioning}
\usetikzlibrary{decorations.pathreplacing}

\newcommand{\bbracket}[1]{\ensuremath{\left({#1}\right)}}
\newcommand{\bigbracket}[1]{\ensuremath{\big({#1}\big)}}
\newcommand{\Bigbracket}[1]{\ensuremath{\Big({#1}\Big)}}

\newcommand{\Bigbar}[1]{\ensuremath{\Big|{#1}\Big|}}

\newcommand{\ham}{\mathrm{H}}
\newcommand*\diff{\mathop{}\!\mathrm{d}}

\newcommand{\given}{\, | \,}
\newcommand{\biggiven}{\, \big{|} \,}
\newcommand{\Biggiven}{\, \Big{|} \,}

\newcommand{\VC}{\textnormal{VC}}
\newcommand{\ndim}{\textnormal{Ndim}}
\newcommand{\argmax}{\textnormal{argmax}}
\newcommand{\argmin}{\textnormal{argmin}}
\newcommand{\bmsigma}{{\bm \sigma}}
\newcommand{\iid}{\textnormal{\textbf{i.i.d.}}}

\DeclareMathOperator{\F}{\mathcal{F}}

\DeclareMathOperator{\M}{\mathbb{M}}

\DeclareMathOperator{\bR}{\mathbb{R}}
\DeclareMathOperator{\z}{\mathbf{z}}
\DeclareMathOperator{\w}{\mathbf{w}}
\DeclareMathOperator{\calW}{\mathcal{W}}
\DeclareMathOperator{\calX}{\mathcal{X}}

\DeclareMathOperator{\calH}{\mathcal{H}}
\DeclareMathOperator{\calZ}{\mathcal{Z}}
\DeclareMathOperator{\calB}{\mathcal{B}}
\DeclareMathOperator{\bE}{\mathbb{E}}
\DeclareMathOperator{\p}{\mathbb{P}}
\DeclareMathOperator{\one}{\mathbf{1}}

\newcommand{\hgamma}{\widehat{\Gamma}}
\newcommand{\chg}{C_T({h,g})}
\DeclareMathOperator{\bx}{\mathbf{x}}
\DeclareMathOperator{\PP}{\mathbb{P}}
\newcommand{\EE}{\mathbb{E}}

\providecommand{\keywords}[1]
{
  \small	
  \textbf{\textit{Keywords---}} #1
}

\newtheorem{assumption}{Assumption}
\newtheorem{definition}{Definition}
\newtheorem{theorem}{Theorem}
\newtheorem{lemma}{Lemma}
\newtheorem{corollary}{Corollary}[theorem]
\newtheorem{proposition}{Proposition}
\newtheorem{remark}{Remark}

\title{Policy Learning with Adaptively Collected Data}

\makeatletter
\newcommand{\printfnsymbol}[1]{%
  \textsuperscript{\@fnsymbol{#1}}%
}
\makeatother

\author[1]{Ruohan Zhan\thanks{Authors contributed equally.}}
\author[2]{Zhimei Ren\printfnsymbol{1}}
\author[1]{Susan Athey}
\author[3]{Zhengyuan Zhou}

\affil[1]{Graduate School of Business, Stanford University}
\affil[2]{Department of Statistics, University of Chicago}
\affil[3]{Stern School of Business, New York University}

\date{First Version: May, 2021;\\This Version: November, 2022}

\begin{document}


\maketitle

\begin{abstract}
    Learning optimal policies from historical data enables personalization in a wide variety of applications including healthcare, digital recommendations, and online education. The growing policy learning literature focuses on settings where the data collection rule stays fixed throughout the experiment. However, adaptive data collection is becoming more common in practice, from two primary sources: 1) data collected from adaptive experiments that are designed to improve inferential efficiency; 2) data collected from production systems that progressively evolve an operational policy to improve performance over time (e.g. contextual bandits). Yet adaptivity complicates the optimal policy identification ex post, since samples are dependent, and each treatment may not receive enough observations for each type of individual. In this paper, we make initial research inquiries into addressing the challenges of learning the optimal policy with adaptively collected data. \textcolor{black}{We propose an algorithm based on generalized augmented inverse propensity weighted (AIPW) estimators, which non-uniformly reweight the elements of a standard AIPW estimator to control worst-case estimation variance.} We establish a finite-sample regret upper bound for our algorithm and complement it with a regret lower bound that quantifies the fundamental difficulty of policy learning with adaptive data. \textcolor{black}{When equipped with the best weighting scheme, our algorithm achieves minimax rate optimal regret guarantees even with diminishing exploration.}
Finally, we demonstrate our algorithm's effectiveness using both synthetic data and public benchmark datasets.
\end{abstract}


\section{Introduction}
The growing availability of user-specific data has welcomed the exciting era of personalized decision making, a paradigm that exploits the heterogeneity in a given population so as to provide individualized service decisions that lead to improved outcomes. This paradigm has found applications in a wide variety of operations management domains. For instance, in healthcare, using electronic medical records, doctors can better prescribe heterogeneous treatments---different types of drugs/therapies or different dosage levels of the same drug---to different patients based on their medical characteristics~\citep{murphy2003optimal,kim2011battle,bertsimas2017personalized,fukuoka2018objectively}. In advertising~\citep{charles2013counterfactual,kallus2016dynamic,schnabel2016recommendations,farias2019learning,bastani2021predicting}, using the recorded clientele information, the retailer can send more targeted product promotions---either in mail or online---to different existing and potential customers. In news recommendation~\citep{li2010contextual,li2011unbiased,zeng2016online,karimi2018news,schnabel2019shaping,lee2020news,schnabel2020impact}, the content provider may stream different news articles and/or media content to users with different digital footprints and perceived interests. In online education~\citep{mandel2014offline,lan2016contextual,hoiles2016bounded,bassen2020reinforcement}, an institution may want to offer different education plans to different students based on their varied learning styles (visual learner v.s. aural learner v.s. verbal learner, etc.).

A key problem in achieving intelligent personalization through data lies in learning an effective policy (which maps individual characteristics to treatments/actions) in a sample-efficient manner (i.e., making the fullest use of a given dataset so as to learn---to the extent possible---a policy that yields the highest rewards and hence leading to the best outcome for each individual). Particular challenges arise for off-policy evaluation due to missing counterfactual outcomes.  Researchers from a variety of fields---including operations research, statistics and machine learning---have devoted extensive efforts to this problem in recent years~\citep{dudik2011doubly,zhang2012estimating,zhao2015doubly,swaminathan2015batch,swaminathan2015counterfactual,swaminathan2015self,swaminathan2016off,kitagawa2018should,levine2020offline,kallus2018confounding,zhou2022offline,joachims2018deep,chernozhukov2019semi,su2019cab,bennett2020efficient,sachdeva2020off,jin2020pessimism,athey2021policy} and satisfactorily addressed (discussed in more detail in Section~\ref{subsec:related}) 
various aspects of the policy learning problem when the underlying historical data has been collected with a \textit{fixed} exploration policy, which results in independent and identically distributed (\textbf{i.i.d.}) data over time. This is an important setting that includes several data-collection mechanisms: A/B testing, randomized control trials, and deploying a fixed operational policy that has built-in randomization.

However, much less is known on this problem when data is collected \textit{adaptively}, that is, where the policy used to select actions evolves over time in response to observed outcomes rather than stays fixed. The following two broad categories of adaptive data collection are common in practice: 
\begin{enumerate}
\item \textbf{Data from Adaptive Experiments.}
Such data are collected from experiments for statistical inference and/or hypothesis testing. Since experiments are costly to conduct, a fixed experiment design (hence a fixed randomization rule) is not efficient. In comparison, adaptive experiment designs can dramatically improve statistical efficiency and are often used instead~\citep{armitage1960sequential,simon1977adaptive,murphy2005experimental,collins2007multiphase,cai2011manifold,offer2019adaptive}. Experimental data initially collected to answer particular inferential questions can be used for policy learning that falls out of its original design.  

\item \textbf{Data from Operations Using Bandit Algorithms.}
Such data is generated and collected from an \textit{operational} policy in production systems. Production systems often adaptively choose their operational policies to improve the performance over time.
A common family of algorithms are bandit algorithms (particularly contextual bandit algorithms)~\citep{lai1985asymptotically,thompson1933likelihood,agrawal2013thompson,russo2017tutorial}, 
where treatment assignments are selected to balance exploration and exploitation to maximize the cumulative performance over time, thereby performing better compared to deploying a fixed policy.


\end{enumerate}

Policy learning using adaptively collected data is much more difficult, since there are complex intertemporal dependencies. To see this, suppose we have observational data $\{(X_t, W_t, Y_t)\}_{t=1}^T$ collected sequentially, where $X_t\stackrel{\textnormal{\textbf{i.i.d.}}}{\sim}P_X$ is the context,
$W_t\in\calW=\{1,\dots, K\}$ is the selected action, and $Y_t=\mu(X_t;W_t)+\epsilon_t$ is the outcome, with $\{\epsilon_t\}_{t=1}^T$ being \textbf{i.i.d.} zero-mean random variables. Importantly, the samples $\{(X_t, W_t, Y_t)\}_{t=1}^T$ are \emph{not} \textbf{i.i.d.} since $W_t$ is sampled according to probabilities $\big(e_t(X_t; 1), \dots, e_t(X_t; K)\big)$---also known  as propensity 
scores---that are time-varying and dependent on past observations $\{(X_s, W_s, Y_s)\}_{s=1}^{t-1}$. 
Note that $e_t$ is the (randomized) policy used at $t$ while data was being collected. 
In practice, these $e_t$ functions may be derived from complicated
functions with a large number of parameters (e.g., neural networks) and are quickly and constantly updated 
(e.g., ads serving engines), hence making it cumbersome to record those policies in their entirety.
Consequently, in reflecting this constraint, we assume that only $e_t(X_t, W_t)$---the probability of sampling the chosen action $W_t$ at $t$---is recorded for each $t$, but not the entire function $e_t(\cdot\,;\, \cdot)$\footnote{Nevertheless, in our proposal we require the knowledge of $e_t(X_t, W_t)$ instead of estimating it from the data, which may be difficult in adaptive experiments since $e_t$ can be time-varying.}.

With a dataset of the form described above, the goal of policy learning is to select a good policy from a given policy class. A policy $\pi$---mapping from contexts to actions---can be evaluated by its policy value $Q(\pi)=\mathbb{E} \big[\mu(X, \pi(X))\big]$, where the expectation is taken with respect to the randomness in $X$ over the target population. To perform effective policy learning, one needs to select a policy with  as large value as possible, or equivalently, as small regret as possible, where regret is defined to be the policy value loss relative to  the best value in the policy class. An effective policy learning algorithm should achieve a small regret as a function of a given finite $T$,  the quantitative measure of the algorithm's sample complexity. However, a moment of thought reveals that this is a challenging desideratum for the following three reasons.

First, unlike in fixed policy data collection where the propensities $e(\cdot \, ;\, \cdot)$ do not change over time (corresponding to a constant exploration bandwidth), a distinct feature in adaptive collection (when carried out by popular algorithms used in practice) is that $e_t(\cdot \,;\, \cdot)$ shrinks and goes to $0$ over time for certain actions (and contexts). This means that exploration is gradually reduced, thereby resulting in vanishing probabilities of selecting certain (poorly performing) actions. This crucial benefit of adaptive experiments---that poor arms can be dropped, while good arms are more likely to be pulled---creates difficulty for offline policy learning, a purely exploitation task. This is because the data not only have selection bias (as is already present in the fixed policy setting), but also become \textit{increasingly} more so overtime, thereby producing skewed data that makes it difficult to compare the quality (i.e., the outcome) of alternative actions under a given context.  

Second, although one might be tempted to think that smaller propensities (particularly towards the end of data-collection) indicate that certain actions are ``clearly bad" as a result of their vanishing probabilities being selected, it is important to keep in mind that the policy learner does not have access to the propensity functions $e_t(\cdot\,;\, \cdot)$ and hence cannot perform this ``action elimination" type of policy learning.
Consequently, when observing that $e_T(X_T; W_T)$ is small on the last timestep $T$ and---assuming that we are confident that the adaptive data collection process is well-designed and can hence conclude  that the action $W_T$ is indeed the wrong action for $X_T$---we would still not be able to know what other actions are wrong for $X_T$, nor what other contexts are bad for $W_T$. Further, since we only observe a single data point $e_t(X_t, W_t)$ for the propensity function $e_t(\cdot\,;\,\cdot)$, there is no hope to learn these evolving $e_t(\cdot \,;\, \cdot)$'s, a clear distinction from the fixed policy setting where one can learn the fixed propensity function $e(\cdot \,;\, \cdot)$ using the entire training dataset that it has generated.

Third, an adaptive data collection mechanism may itself be ``poorly" designed and hence result in wrong propensities: $e_t(x,w)$ is small when $w$ is in fact the best action for context $x$. This ``poor" design can arise for different reasons; for instance, it could be that an ineffective adaptive exploration scheme is used, \textcolor{black}{resulting in over-exploiting certain actions and under-exploring others}. Since contextual bandits rely on learning a complex outcome surface $\mu(\cdot,w)$ for each action $w$, most commonly used algorithms rely on specifying and estimating a parametric model for $\mu$, making misspecification a real possibility. Alternatively, it could be that the data were collected to answer a specific inferential
question and hence the data-collection process was steered towards a particular direction.
Regardless of the cause, when this occurs, the policy learner will see very few samples on the good actions but many samples on the bad actions, in effect reducing the overall useful samples and yielding larger uncertainty about which action is actually good for which context. 
This is a risk that the policy learner should address, because it has no control over what the data-collector does in collecting the data.\footnote{If the policy learner also plays the role of collecting data, it should instead focus on adaptive experimental design or online adaptive learning, rather than offline policy learning, although there are scenarios where the data-collector has multiple objectives and thus does not optimize for policy learning during the experiment.} Hence,  policy learning methods should be robust to a wide class of adaptive data-collection mechanisms, good or bad from the policy learner's own perspective.     

Situated in this challenging and under-explored landscape, we aim to make initial progress into confronting these challenges and focus on developing finite-sample regret bounds that shed light on the design and implementation of efficient policy learning using adaptively collected data.


\subsection{Our Contributions and Related Work}
Our contributions are twofold.
First, we study the fundamental difficulty of this problem by characterizing a lower bound for policy learning. In particular, let $\{g_t\}_{t=1}^T$ be any positive lower-bound sequence of propensities $e_t(\cdot \,;\, \cdot)$ (i.e., $e_t(\cdot\,;\,\cdot)$  is lower bounded by $g_t$ for each $t$).\footnote{The assumption that $g_t > 0$ ensures that each action is sampled with positive probability regardless of the contexts.} Then, \textcolor{black}{for a worst-case distribution satisfying the propensity lower bound $\{g_t\}_{t=1}^T$}, any policy learning algorithm will incur at least an expected regret of  $\Omega\big(\sqrt{{\ndim(\Pi)}/{\sum_{t=1}^T g_t}}\big)$, where $\ndim(\Pi)$ refers to the Natarajan dimension of the (multi-action) policy class $\Pi$---this is a generalization of the VC-dimension for the multi-action policy classes, and in particular $\VC(\Pi) = \ndim(\Pi)$ in settings of binary actions (i.e. only two actions are available).
Consider an example where $g_t$ decays at a rate with $g_t = t^{-\alpha}$ and $\alpha \in [0,1)$, then the worst-case expected regret is lower bounded by $\Omega(\sqrt{\ndim(\Pi)}\cdot T^{(\alpha - 1)/2})$.
This regret lower bound highlights a necessary boundary for the feasibility of
policy learning with adaptively collected data: if the adaptive data collection process is overly aggressive in exploitation, leading to an exploration rate that decreases faster than $\Theta(t^{-1})$, then the regret will be $\Omega(1)$---there is thus no hope of ever learning a near-optimal policy in the worst cases, no matter how large $T$ is. On the other hand, when $\alpha = 0$,
our regret bound recovers the $\Omega(T^{-{1}/{2}})$ lower bound for the non-adaptive setting where the data is collected by a fixed policy~\citep{kitagawa2018should}.

Second, building on the recent adaptive inference literature, we propose a policy learning algorithm and \textcolor{black}{establish its
expected regret bound, which is minimax optimal when the assignment probability lower bound $g_t$ is known to the policy learner.}
Our algorithm follows a two-step procedure: 1) construct a policy value estimator $\widehat{Q}(\pi)$ for any fixed $\pi\in\Pi$ using the collected data; 
2) output the policy in $\Pi$ that  maximizes the estimated value: $\widehat{\pi}=\argmax_{\pi\in\Pi}\widehat{Q}(\pi)$. 
The specific estimator we use is a variant of the family of generalized augmented
inverse propensity weights (AIPW) estimators considered 
in~\cite{luedtke2016statistical, hadad2021confidence,zhan2021off}, which takes the following form:
\begin{align*}
    & \widehat{Q}_T(\pi) = \frac{\sum_{t=1}^T h_t\widehat{\Gamma}_t(\pi)}{\sum_{t=1}^T h_t},
    \mbox{ where }\widehat{\Gamma}_t(\pi) = \widehat{\mu}_t\big(X_t; \pi(X_t)\big)+ 
    \frac{\one \big\{W_t=\pi(X_t)\big\}}{e_t\big(X_t;\pi(X_t)\big)}\Big(Y_t-\widehat{\mu}_t\big(X_t;\pi(X_t)\big)\Big).
\end{align*}
Above,  $\widehat{\mu}_t$  is the nuisance estimator of the expected outcome,
$\widehat{\Gamma}_t(\pi)$ is the \emph{AIPW score} \citep{robins1994estimation},  and $h_t$ is the pre-specified deterministic weight that remains the same for all $\pi\in\Pi$. 
The purpose of $h_t$ is to offset the (potentially) large  worst-case variance caused by vanishing assignment probabilities $e_t(X_t; \pi(X_t))$. Depending on whether $g_t$ is known or not, $h_t$ would be chosen differently, which results in different regret bounds (to be elaborated shortly in  subsequent paragraphs). These specific choices of $h_t$ are simple and different from those  variants adopted in~\cite{luedtke2016statistical, hadad2021confidence,zhan2021off}, which are concerned with devising $h_t$ to evaluate a specific target policy and achieve asymptotic normality for inference, while instead here we aim for controlling worst-case estimation variance and constructing finite-sample regret bounds. 

\textcolor{black}{With the weights $\{h_t\}_{t=1}^T$ plugged into the generalized AIPW estimator, we show that our algorithm has  
an expected regret upper bound of $\tilde{O}\big(\kappa(\Pi)\cdot \frac{\sqrt{\sum_{t=1}^Th_t^2/g_t} }{\sum_{t=1}^T h_t}
+ \frac{\sum^T_{t=1}h_t^4/g_t^3}{(\sum^T_{t=1}h_t^2/g_t)^2}\big)$~\footnote{We use $\tilde{O}(\cdot)$ to show rates after omitting logarithm factors.}, and $\kappa(\Pi)$ is the entropy integral of a policy class $\Pi$ based on the Hamming distance.
Since the data are no longer i.i.d., existing techniques in most policy learning literature no longer apply. 
Our analysis instead leverages the framework of sequential uniform concentration~\citep{rakhlin2015sequential}. 
When $\widehat{\mu}_t$ is fitted with observations up to time $t-1$, the AIPW score $\widehat{\Gamma}_t$
is unbiased conditional on the past observations, so we can write $\widehat{Q}_T(\pi) - Q(\pi)$ as the
sum of a martingale difference sequence. The supremum of the sum of this martingale difference sequence
cannot be bounded with the standard notion of the Rademader complexity used in existing policy learning literature \citep[e.g.,][]{kallus2018balanced,zhou2022offline}. 
We instead consider an analog of the Rademacher process in the adaptive data setting---the tree Radamecher process---and connect the martingale difference sequence with a tree Radamechar process, thereby developing a bound for $\max_{\pi\in\Pi}|\widehat{Q}(\pi)-Q(\pi)|$. However the uniform concentration results in \cite{rakhlin2015sequential} are not directly applicable here, since here martingale difference terms are not bounded (as required in~\citet{rakhlin2015sequential}) with risk that $\widehat{\Gamma}_t$  may diverge due to vanishing propensities. To address this issue, we introduce a high probability event
on which the quadratic variation of the martingale difference sequence is regularized; then we
   bound the supremum of the tree Radamechar process via a covering of the policy class on the event, by refining and sharpening a chaining technique in \cite{zhou2022offline}.
}

\textcolor{black}{The choice of weight $h_t$ largely decides the regret bound provided by our algorithm.  
We show that the optimal weight $h^*_t$ 
is  proportional to the assignment probability lower bound $g_t$,  
which yields the expected regret bound $\tilde{O}(\kappa(\Pi)/\sqrt{\sum_{t=1}^T g_t})$.
Note  that 
$\kappa(\Pi) = O\big(\sqrt{\log{p}\cdot\ndim(\Pi)}\big)$, where $p$ denotes 
the dimension of the context; in the binary-action case, $\kappa(\Pi) \le 2.5 \sqrt{\VC(\Pi)}$~\citep{jin2022upper}.
This implies that our upper bound with optimal weight $h^*_t$ is tight with respect to (abbreviated as w.r.t. hereafter) the sample size and the complexity of the policy class up to logarithmic factors, and is thus minimax optimal.
In cases where  $g_t$ is not disclosed to the policy learner, we recommend using uniform weights, that is, estimating policy values with standard AIPW estimator. This choice yields an expected regret bound $\tilde{O}\big(\kappa(\Pi) \cdot \frac{\sqrt{\sum_{t=1}^T g_t^{-1}}}{T} + \frac{\sum^T_{t=1} 1/g_t^3}{(\sum^T_{t=1} 1/g_t)^2}\big)$ 
that is in general looser than the minimax regret   but performs reasonably well in many cases.  In particular,  for the same example above in which $g_t=t^{-\alpha}$, using uniform weighting also achieves the minimax optimal regret bound $\tilde{O}(\kappa(\Pi)\cdot T^{(\alpha-1)/2})$, which---by setting  $\alpha$ to 0---recovers the minimax optimal regret bound for policy learning under i.i.d. data collection established in \cite{zhou2022offline}.}

\textcolor{black}{After we posted a version of this paper online, \cite{bibaut2021risk} made remarkable progress on this policy learning problem with adaptive data. They show that their algorithm (which amounts to a variant of our algorithm that by using uniform weighting $h_t=1$ and setting nuisance component $\hat{\mu}_t=0$)  meets our established lower bound in settings where $g_t=t^{-\alpha}$ (though their algorithm does not guarantee minimax optimality beyond those special cases). Encouraged by this positive result, we tighten the upper bound of our algorithm and show that it achieves minimax optimal regret guarantee  with optimal weighting $h_t^*=g_t$  in  general cases (that is, even when $g_t$ is not of the form $t^{-\alpha}$). 
Besides, the nuisance component $\widehat{\mu}_t$ in our algorithm (which is set to zero in \cite{bibaut2021risk}) would reduce variance  in  policy value estimation even with misspecification, which in turn improves  the value of learned policy. We defer detailed empirical comparison to   Appendix \ref{appendix:compare_with_erm}.}

Finally, leveraging a publicly available piece of software, we consider the policy class of (fixed-depth) decision trees and evaluate the algorithm on both synthetic data and public benchmark datasets. The empirical results show two important strengths of our policy learning algorithm.
First, when $g_t$ is unknown, our algorithm narrowly trails and has the same regret decay rate as the setting where $g_t$ is known (and hence the optimal weights $h_t^*$ can be computed exactly). This suggests that although it would be helpful to know $g_t$ and hence leverage it to achieve even better performance, our algorithm would still be functional even when such information is not available. Second, since online learning algorithms (e.g., Thompson sampling) often directly use outcome regression, they are prone to model misspecification, where they allocate small propensities to good actions, incurring large performance gap as the process goes on. Despite this, our offline learning algorithm can still find the ($\epsilon$-)optimal policy, demonstrating its effectiveness and robustness.

\subsection{Other Related Work}\label{subsec:related}

Policy learning with observational data is a growing field that has received increasing attention from different communities. As already mentioned, the existing literature in offline policy learning has primarily focused on data collected by a fixed policy, where they collectively proposed several statistical efficient and/or computationally efficient policy learning algorithms that achieve the minimax optimal regret bounds of $\Theta(T^{-{1}/{2}})$ on the expected regret. A sequence of refinements~\citep{zhang2012estimating,zhao2015doubly,kitagawa2018should,athey2021policy}  addressed this challenge in the settings of binary actions; and particularly   \cite{kitagawa2018should} established the tight regret bound with the knowledge of propensities, which was relaxed later in  \cite{athey2021policy} that established the optimal dependency with estimated propensities.  Extension to multi-action schemes has been successively investigated by \cite{swaminathan2015batch,zhou2017residual,kallus2018balanced,zhou2022offline};  \cite{zhou2022offline} established the minimax regret bound  using doubly robust AIPW estimator  when the propensities are unknown. An important distinction when using AIPW estimator on adaptive data is that one often assumes the knowledge of propensities~\citep{luedtke2016statistical,hadad2021confidence,zhan2021off}, as is the case in this paper, since they are typically time-varying and are difficult to approximate with limited batch size.

Another strongly related area is offline policy evaluation with adaptively collected data. Estimating policy values on adaptively collected data is much more challenging compared to that on \textbf{i.i.d.} data. For instance, direct methods that fit regression models will be biased~\citep{nie2018adaptively,shin2019sample}, while unbiased estimators such as inverse propensity weighted (IPW) estimator 
can suffer from huge variability~\citep{horvitz1952generalization,imbens2004nonparametric}. 
Note that IPW estimators already suffer from large variance with \textbf{i.i.d.} data (where the propensity has a fixed lower bound that does not change), and the problem becomes more acute in the adaptive data collection setting because the propensities are vanishing.
The AIPW estimator ~\citep{robins1994estimation,dudik2011doubly} combines 
the outcome modelling and IPW approaches, gaining efficiency and ``double robustness'' properties with \textbf{i.i.d.} data.
In particular, the AIPW estimator is consistent if either the propensity model or the outcome model is consistently estimated. However, AIPW still has challenges when the data is adaptively collected, since vanishing propensities will yield exploding variance and de-stablize the estimator.
To deal with this, the literature has seen two approaches for adapting AIPW estimators to offline policy evaluation with adaptively collected data. The first incorporates weight clipping into AIPW \citep{bembom2008data, charles2013counterfactual,wang2017optimal,su2020doubly}, where one controls variance by shrinking the weights at the cost of introducing a small bias. 
The second approach, described above, is to locally stabilize the elements of the AIPW estimator~\citep{luedtke2016statistical, hadad2021confidence,zhan2021off}. Our policy learning algorithm uses an estimator that falls into this second approach, where the weights $h_t$  are chosen with the consideration of the worst-case variance in order to be robust.

Finally, there is also an extensive literature on online contextual bandits~\citep{dani2008stochastic,besbes2009dynamic,rigollet2010nonparametric,abbasi2011improved,chu2011contextual,bubeck2012regret,abbasi2013online,agrawal2013thompson,goldenshluger2013linear,russo2014learning,li2017provably,dimakopoulou2017estimation,bastani2020online}, the online counterpart of policy learning. The literature focuses on analyzing bounds on the regret experienced by alternative algorithms. This problem is distinct from ours and from offline policy learning in general; in the online learning literature, the focus is on the adaptive algorithms and how to assign treatment to units in order to  balance the exploration-exploitation trade-off. We close by pointing out that no ``online-to-batch" conversion is feasible here: online learning algorithms cannot be directly converted to offline learning algorithms because the offline policy learner does not have the opportunity to choose a particular action. The counterfactual outcomes from alternative actions for a particular unit are unobserved.

\section{Offline Policy Learning}
Let $T$ be the time horizon, $\calX$ the covariate space, and $\calW$ the action space with $K$ actions. 
At time $t \in [T] \stackrel{\Delta}{=}\{1,\ldots, T\}$, the experimenter
observes a covariate $X_t \in \calX$ with $X_t\stackrel{\iid}{\sim}P_X$; then the experimenter  takes an action $W_t \in \calW$ sampled from a 
multinomial distribution with probability $\big(e_t(X_t; 1), \dots, e_t(X_t; K)\big)$; next 
she  receives a response $Y_t$ from the chosen action generated via:
\begin{align*}
    Y_t = \mu(X_t;W_t) + \varepsilon_t,
\end{align*}
where $\{\varepsilon_t\}_{t=1}^T$ is a sequence of \iid~zero-mean and $\sigma^2$-variance random variables. 
We assume that the assignment
probability $e_t(x,w)$ is updated for all $(x,w) \in \calX \times \calW$ using observations up to time $t-1$.
We use $\calH_{t}=\{(X_s, W_s, Y_s)\}_{s=1}^{t}$ to denote
samples up to time $t$. Throughout, we make the following assumptions on the data generating process.
\begin{assumption}
\label{assumption:dgp}
The data generating process has: 
\begin{enumerate}[label = (\alph*)]
    \item Bounded outcome: $\exists~M\ge 0$ such that $|Y_t| \le M$ and $|\mu(x;w)|\le M$ for all $(x,w)\in \calX\times \calW$.
    \item Known and bounded assignment probability: $e_t(X_t;W_t)$ is known; and 
    $e_t(x;w)$ is lower bounded \textcolor{black}{with probability one, over the sampling of $\{(X_s, W_s, Y_s)\}_{s=1}^{t-1}$,} by a positive and nonincreasing deterministic sequence $\{g_t\}$ for all $(x,w)\in \calX\times \calW$ and all
    $t \in [T]$.
\end{enumerate}
\end{assumption}
Assumption~\ref{assumption:dgp}(b) requires the experimenters to record data $(X_t, W_t, Y_t, e_t(X_t;W_t))$ at each time instead of the standard $(X_t, W_t, Y_t)$, which is feasible in many applications since  practitioners often have access to the algorithm
used to collect data, and hence have the knowledge of the assignment probability. 
\textcolor{black}{The requirement of a lower bound on  assignment probabilities for \emph{all} arms is a bit more restrictive---for instance standard Thompson
sampling may not fall into this category since the sampling probability on a clearly sub-optimal arm may decay to zero too fast as $T$ goes to infinity---but
we note that this is already a generalization of the previous settings with \textbf{i.i.d.} data.} In parallel, previous work on offline policy learning with \textbf{i.i.d.} data
(e.g.,~\cite{athey2021policy,zhou2022offline})  often assumes the ``overlap'' condition: the assignment probability $e_t(x;w)$ is 
lower bounded by a positive constant for all $(x,w)\in\calX \times \calW$ and all $t\in[T]$. This is in fact a special case of our 
condition by setting $g_{t}$ to be a positive constant.


\subsection{Offline Policies and Metrics}
\textcolor{black}{A policy $\pi$ is a mapping from the covariate
space to the action space, i.e., $\pi: \calX \mapsto \calW$.}
Given a policy $\pi$, we are interested in its \emph{policy value} defined as the expected reward incurred by taking actions according to $\pi$:
\begin{align*}
  Q(\pi) \stackrel{\Delta}{=}\bE_{X} \Big[\mu\big(X;\pi(X)\big)\Big],
\end{align*}
where $\bE_X[\cdot]$ denotes taking expectation w.r.t.~$P_X$.
For a class of policies $\Pi$, we wish to learn the optimal policy  $\pi^*$ within $\Pi$:
\begin{align*}
 \pi^* = \argmax_{\pi\in \Pi}~Q(\pi). 
\end{align*}
The {\em regret} of a policy $\pi \in \Pi$
is the policy value difference between  the 
optimal policy and itself:
\begin{equation*}
    R(\pi) \stackrel{\Delta}{=} Q(\pi^*)-Q(\pi). 
\end{equation*}
Conceptually, we would like to learn a policy whose regret is as small as possible (so that its policy value is close to that of the
optimal policy). To do this, we shall 1) construct a policy value estimator for each $\pi \in \Pi$,
and 2) choose the policy $\hat{\pi}$ that maximizes the estimated policy value over the class $\Pi$. Note that here $\hat{\pi}$ is data-driven,
and thus itself and $R(\hat{\pi})$ are both random variables. 
Our strategy of deriving an upper bound for $R(\hat{\pi})$
is to establish that our policy value estimator is close to the true value \emph{uniformly} across the policy class. 
For the rest of the paper, \textcolor{black}{when we say a regret bound for $\hat{\pi}$, 
we refer to upper bound of the expected regret.}


\subsection{Policy Class Complexity}
Policy class is a central element in the discussion of policy learning. The complexity of a policy class presents a trade-off: 
a richer policy class means that the corresponding optimal policy has a larger value, while the increasing complexity of 
the class makes the optimal policy learning  intrinsically harder. Ideally we would like to consider a policy class with 
the appropriate complexity such that the optimal policy has a reasonably high value, and meanwhile it is feasible to learn the 
optimal policy efficiently. Our first step towards this goal is to characterize the complexity of a policy class
based on the Hamming distance.
\begin{definition}[Hamming Distance and the Covering Number]
\label{definition:hamming}
Consider a covariate space $\calX$ and a policy class $\Pi$. We define:
\begin{enumerate}[label = (\alph*)]
    \item the Hamming distance between any two policies $\pi_1,\pi_2\in \Pi$ given a covariate set $\{x_{1:n}\}$: 
  $$\ham\big(\pi_1, \pi_2 ; \{x_{1:n}\}\big) \stackrel{\Delta}{=} \frac{1}{n}\sum_{j=1}^n\one\big\{\pi_1(x_j)\neq 
    \pi_2(x_j)\big\};$$
    \item $\epsilon$-Hamming covering number of $\Pi$ given $\{x_{1:n}\}$: $N_\ham\big(\epsilon, \Pi; \{x_{1:n}\}\big)=|\Pi_0|$,\footnote{For a set $A$, we let $|A|$ denote its cardinality.} where  $\Pi_0 \subset \Pi$ is the smallest policy set such that $\forall \pi\in\Pi$, there exists $\pi'\in\Pi_0$ satisfying 
    $\ham\big(\pi,\pi'; \{x_{1:n}\} \big) \leq \epsilon$;
    \item $\epsilon$-Hamming covering number of $\Pi$: $
        N_\ham(\epsilon, \Pi) \stackrel{\Delta}{=}\sup\big\{N_\ham(\epsilon, \Pi ; \{x_{1:m}\}) :  m\geq 1, x_1,\dots, x_m\in\calX\}$;
    \item entropy integral: $\kappa(\Pi)=\int_0^1 \sqrt{\log{N_\ham(\epsilon^2, \Pi)}}\diff{\epsilon}$.
\end{enumerate}
\end{definition}
The entropy integral is the key quantity characterizing the complexity of a policy class. We will show in Section~\ref{sec:upper_bnd}
that the regret of our estimator can be upper bounded by a function of the entropy integral, and in Section~\ref{sec:simulation} we provide a tree-based
policy class with finite entropy integral.

\section{Regret Lower Bound}
To characterize the fundamental limit of the offline policy learning problem, we establish a lower bound for the 
worst-case expected regret in this
section. The lower bound is stated in terms of the Natarajan dimension of a multi-action policy class,
the definition of which is as follows.



\begin{definition}
Given a $K$-action policy class $\Pi$,
we say a set of $m$ points
$\{x_1,x_2,\ldots,x_m\}$ is shattered by $\Pi$ if there exist two functions $f_{-1}$, 
$f_1: \{x_1,x_2,\ldots, x_m\} \mapsto [K]$ such that, 
\begin{enumerate}[label = (\alph*)]
    \item for any $j\in[m]$, $f_{-1}(x_j) \neq f_1(x_j)$;
    \item for any $\bmsigma = (\sigma_1, \sigma_2, \ldots, \sigma_m) \in \{\pm 1\}^m$, 
    there exists a policy $\pi \in \Pi$ such that for any $j \in [m]$,
    \begin{align*}
        \pi(x_j) = \begin{cases}
            f_{-1}(x_j) & \mbox{ if }\sigma_j = -1;\\
            f_1(x_j) & \mbox{ if }\sigma_j = 1.
        \end{cases}
    \end{align*}
\end{enumerate}
The Natarajan dimension of $\Pi$ is defined to be the size of the largest set of points that can be shattered by $\Pi$.
\end{definition}


\textcolor{black}{Returning to our problem, we let $\ndim(\Pi)$ denote the Natarajan dimension of $\Pi$, $\calZ_T  = \{(X_t,W_t,Y_t)\}_{t=1}^T$ represent the collected
offline data and $\p$ be the joint distribution of $\calZ_T$.
The following theorem shows that, there exists a $\p$
satisfying Assumption~\ref{assumption:dgp}, such that for any policy $\hat{\pi}$ that is learned from $\calZ_T\sim \p$,  
the expected regret is lower bounded by $\Omega\Big(\min\big(1, \sqrt{\ndim(\Pi)\big/\sum^T_{t=1} g_t}\big)\Big)$.}

\textcolor{black}{
\begin{theorem}
\label{theorem:lower_bound}
Fix a squence $\bar{g}=\{g_t\}_{t=1}^T$, 
and let $\mathcal{P}(\bar{g})$ be the collection of all laws of data generating processes for which Assumption \ref{assumption:dgp} 
holds with the assignment probability lower bound sequence $\bar{g}$. 
For any $\hat{\pi}\in\Pi$ learned from $\calZ_T$, there is 
\begin{align*}
   \sup_{\mathbb{P}\in \mathcal{P}(\bar{g})} \bE_{\calZ_T\sim 
   \mathbb{P}}\big[R(\hat{\pi})\big] \ge \frac{M}{8e^3}\cdot \min\bigg(\frac{1}{3}, \sqrt{\frac{\ndim(\Pi)}{\sum_{t=1}^T g_t}}\bigg).
\end{align*}
\end{theorem}
}

\textcolor{black}{As is self-explained in Theorem \ref{theorem:lower_bound}, when $\ndim(\Pi) > \frac{1}{9}\sum_{t=1} g_t$, the expected regret in the worst case is lower bounded by a constant, which makes the problem of learning optimal policies infeasible.  Hence in this paper, we shall focus on the settings where $\ndim(\Pi) \leq \frac{1}{9}\sum_{t=1} g_t$ and write the lower bound as $O\big(\sqrt{\ndim(\Pi)/\sum_{t=1}^T g_t}\big)$ elsewhere when no confusion arises.} 

\subsection{Proof of Theorem~\ref{theorem:lower_bound}}

Consider $K$ arms indexed by $[K]$,
and write $d = \ndim(\Pi)$. By the definition of the Natarajan dimension, there exists a set of $d$ points 
$\{x_1,x_2,\ldots,x_d\} \subset \calX$ that is shattered by $\Pi$, i.e., there exist two functions 
$f_{-1}$, $f_1: \{x_1,x_2,\ldots,x_d\} \mapsto [K]$, such that $f_{-1}(x_j) \neq f_1(x_j)$ for any $j\in[d]$, 
and for any $\bmsigma \in \{\pm 1\}^d$, there exists a policy $\pi \in \Pi$ such that $\pi(x_j) = f_{\sigma_j}(x_j)$.


To establish the lower bound, we construct $2^d$ distributions for $\calZ_T$,
where each $\bmsigma \in \{\pm 1\}^d$ induces a joint distribution of $\calZ_T$. Fix a $\bmsigma \in \{\pm 1\}^d$. For each $t\in [T]$,
let $X_t$ be independently and uniformly generated from $\{x_1,x_2,\ldots,x_d\}$; conditional on $X_t = x_j$, $W_t$ is chosen to be $f_{1}(x_j)$ 
with probability~(w.p.) $g_t$ and other arms w.p.~$\big(1-g_t\big)/(K-1)$ (thus the distribution of $X_t$ and $W_t$ does not depend on $\bmsigma$).
We now proceed to specify the set of reward distributions for the $K$ arms that depend on $\bmsigma$.
For any $j\in[d]$ and some (small) $\Delta > 0$ to be specified later, conditional on $X_t = x_j$, 

\begin{align*}
&\mbox{arm $f_1(x_j)$: } Y_t\big(f_1(x_j)\big) \sim \textcolor{black}{M}\cdot \textnormal{Bern}\Big(\frac{1+\sigma_j\cdot \Delta}{2}\Big),\\
&\mbox{arm $f_{-1}(x_j)$: } Y_t\big(f_{-1}(x_j)\big) \sim  \textcolor{black}{M}\cdot \textnormal{Bern}\Big(\frac{1}{2}\Big),\\
&\mbox{arm $k$: } Y_t(k) = 0, \mbox{ for any }k \in [K]\backslash \big\{f_1(x_j), f_{-1}(x_j)\big\},
\end{align*}

where $\textnormal{Bern}(q)$ denotes the Bernoulli distribution with parameter $q$. By construction, conditional on $X_t = x_j$,
the optimal arm is $f_{\sigma_j}(x_j)$; since $\{x_1,\ldots,x_d\}$ is shattered by $\Pi$, there exists a policy $\pi^{\bmsigma,*}\in\Pi$
that selects the optimal arm for any $x_j\in\{x_1,\ldots,x_d\}$.
It can be easily verified that the distributions constructed satisfy Assumption~\ref{assumption:dgp}.
Let $\p_{\bmsigma}(\cdot)$ and $\mathbb{E}_{\bmsigma}[\cdot]$ refer to the distribution and expectation taken under 
the joint distribution induced by $\bmsigma$,
and let $\p$ be the mixture distribution  uniformly
drawn from $\{\p_{\bmsigma}\}_{\bmsigma \in \{\pm 1\}^d}$.
For any policy $\hat{\pi}$ learned from $\mathcal{Z}_T\sim \p$,
\begin{align}
\label{eq:lowbnd1}
    \bE_{\calZ_T\sim 
    \mathbb{P}}\big[R(\hat{\pi})\big] 
    \ge &\dfrac{1}{2^d}\sum_{\bmsigma \in \{\pm 1\}^d}\bE_{\bmsigma}\big[Q(\pi^{\bmsigma, *})-Q(\hat{\pi})\big] \nonumber\\
    \stackrel{(i)}{=} &\dfrac{1}{2^d}\sum_{\bmsigma \in 
    \{\pm 1\}^d} \frac{1}{d}\sum^{d}_{j=1}\bE_{\bmsigma}\big[Q(\pi^{\bmsigma,*})-Q(\hat{\pi}) \given X = x_j\big] \nonumber\\
    \stackrel{(ii)}{\geq} & \frac{ \textcolor{black}{M}\Delta}{d2^{d+1}} \sum_{\bmsigma \in \{\pm 1\}^d} \sum^{d}_{j=1}
    \mathbb{E}_{\bmsigma}\Big[\one \big\{\hat{\pi}(x_j) \neq f_{\sigma_j}(x_j)\big\} \Big]\nonumber\\
    = & \frac{\textcolor{black}{M}\Delta}{d2^{d+1}} \sum_{\bmsigma \in \{\pm 1\}^d} \sum^{d}_{j=1}
    \p_{\bmsigma}\big(\hat{\pi}(x_j) \neq f_{\sigma_j}(x_j) \big).
\end{align}
Above, step (i) uses the tower property; in step (ii), the optimal policy is in $\Pi$ since $\{x_1,\ldots,x_d\}$
is shattered by $\Pi$. 
For a fixed $\bmsigma\in\{\pm 1\}^d$, let $M_j(\bmsigma)\in \{\pm 1\}^d$ be the vector that differs from $\bmsigma$ only in element $j$:
$[M_j(\bmsigma)]_j = -\bmsigma_j$ and  $[M_j(\bmsigma)]_i = \bmsigma_i$ for all $i \neq j$. Equipped with the notation, we have
\begin{align}
\label{eq:lowbnd1.5}
    \eqref{eq:lowbnd1}= & \frac{\textcolor{black}{M}\Delta}{d2^{d+1}} \sum_{j=1}^d\sum_{\bmsigma: \sigma_j = 1}
    \Big(\p_{\bmsigma}\big(\hat{\pi}(x_j) \neq f_1(x_j)\big)
    + \p_{M_j(\bmsigma)}\big(\hat{\pi}(x_j) \neq  f_{-1}(x_j)\big)\Big)\nonumber\\
    \stackrel{(i)}{\ge} & \frac{\textcolor{black}{M}\Delta}{d2^{d+1}} \sum_{j=1}^d\sum_{\bmsigma: \sigma_j = 1}
    \Big(\p_{\bmsigma}\big(\hat{\pi}(x_j) = f_{-1}(x_j)\big)
    + 1 - \p_{M_j(\bmsigma)}\big(\hat{\pi}(x_j) =   f_{-1}(x_j)\big)\Big)\nonumber\\
    \stackrel{(ii)}{\ge} &\frac{\textcolor{black}{M}\Delta}{d 2^{d+1}}\sum^d_{j=1}\sum_{\bmsigma: \sigma_j = 1}
    \Big(1 - \textsf{TV}\big(\p_{\bmsigma},\p_{M_j(\bmsigma)}\big) \Big)\nonumber\\
    \stackrel{(iii)}{\ge} &\dfrac{\textcolor{black}{M}\Delta}{d2^{d+2}}\sum^d_{j=1}\sum_{\bmsigma:\sigma_j=1}
    \exp\Big(-\mathrm{D}_{\mathrm{KL}}\big(\p_{\bmsigma} \, \| \,\p_{M_j(\bmsigma)}\big)\Big),
\end{align}
where $\textsf{TV}(P,Q)$ denotes the total variation distance between two 
distributions $P$ and $Q$, and $\textnormal{D}_{\textnormal{KL}}(P\,\|\,Q)$ is the 
KL-divergence between $P$ and $Q$; step (i) is because $f_1(x_j) \neq f_{-1}(x_j)$, and 
$\{\hat{\pi}(x_j) = f_{-1}(x_j)\}\subset \{\hat{\pi}(x_j) \neq f_1(x_j)\}$;
step (ii) follows from the definition of the
total variation distance and step (iii) is a result of Lemma~\ref{lemma:tv_exp} stated below.
\begin{lemma}[\citet{tsybakov2008introduction}, Lemma 2.6]
\label{lemma:tv_exp}
Let $P$ and $Q$ be any two probability measures on the same measurable
space. Then 
\begin{align*}
    1 - \textnormal{\textsf{TV}}(P,Q) \ge \dfrac{1}{2}\exp
    \big(-\textnormal{D}_{\mathrm{KL}}(P\|Q) \big).
\end{align*}
\end{lemma}
The KL-divergence between $\p_{\bmsigma}$ and $\p_{M_j(\bmsigma)}$ can be directly computed as,
\begin{align}
\label{eq:lowbnd2}
    \mathrm{D}_{\mathrm{KL}}\big(\p_{\bmsigma}\|\p_{M_j(\bmsigma)}\big) = &
    \bE_{{\bmsigma}}\bigg[\log  \dfrac{\mathrm{d}\!\p_{\bmsigma}}
    {\mathrm{d}\!\p_{M_j(\bmsigma)}}\big(X_1,W_1,Y_1,\ldots,X_T,W_T,Y_T \big)\bigg]\nonumber\\
    = & \bE_{{\bmsigma}}\bigg[\sum^T_{t=1} \log \dfrac{\mathrm{d}\!\p_{\bmsigma}(Y_t\mid W_t,X_t)}
    {\mathrm{d}\!\p_{M_j(\bmsigma)}(Y_t\mid W_t, X_t)}\bigg]\nonumber\\ 
    \stackrel{(i)}{=} & \sum^T_{t=1}\frac{1}{d}\sum_{i=1}^d \bE_{\bmsigma}\bigg[  \log
    \dfrac{\mathrm{d}\!\p_{\bmsigma}(Y_t\mid W_t, X_t)}
    {\mathrm{d}\!\p_{M_j(\bmsigma)}(Y_t\mid W_t,X_t)} \Biggiven X_t = x_i\bigg]\nonumber\\
    \stackrel{(ii)}{=} & \sum^T_{t=1}\frac{1}{d} \bE_{\bmsigma}\bigg[  \log
    \dfrac{\mathrm{d}\!\p_{\bmsigma}(Y_t\mid W_t, X_t)}
    {\mathrm{d}\!\p_{M_j(\bmsigma)}(Y_t\mid W_t,X_t)} \Biggiven X_t = x_j\bigg]\nonumber\\
    = & \sum^T_{t=1}\frac{1}{d}\bE_{\bmsigma}\bigg[ \mathbf{1} \{W_t = f_{1}(x_j)\} \cdot  
    \Delta\log\Big(\dfrac{1+\Delta}{1-\Delta}\Big) \Biggiven X_t = x_j\bigg]\nonumber\\ 
    \stackrel{(iii)}{\le}& \frac{3\Delta^2}{d} \cdot
    \sum^T_{t=1}\bE_{\bmsigma}\bigg[\mathbf{1} \big\{W_t = f_{1}(x_j)\big\}\Biggiven X_t = x_j\bigg]\nonumber\\
    \stackrel{(iv)}{=} &  \frac{3\Delta^2}{d} \cdot \sum_{t=1}^T g_t,
\end{align}
where step (i) uses the linearity and the tower property of expectation; step (ii) is because
$\p_{\bmsigma}$ differs from $\p_{M_j(\bmsigma)}$ only when $X = x_j$; \textcolor{black}{step (iii) follows 
from that $x\log(\frac{1+x}{1-x})\le 3x^2$ for $x \in [0,\frac{1}{3}]$};  step (iv) is by design of the sampling mechanism 
of $\{W_t\}_{t=1}^T$.
Combining~\eqref{eq:lowbnd1.5} and~\eqref{eq:lowbnd2}, we have that
\textcolor{black}{\begin{align*}
    \bE_{\calZ_T\sim 
    \mathbb{P}}\big[Q(\pi^*)
    - Q(\hat{\pi})\big] \geq \dfrac{M\Delta}{8}\exp \Big(- \frac{3\Delta^2}{d}\cdot \sum_{t=1}^T g_t\Big). 
\end{align*}}
\textcolor{black}{If $d\leq  \frac{1}{9}\sum_{t=1}^T g_t$, letting $\Delta = \sqrt{d\big/\big(\sum_{t=1}^T g_t\big)}$ yields the desired lower bound. 
Otherwise, choose $\Delta=1/3$, and we have 
\begin{align*}
     \bE_{\calZ_T\sim 
    \mathbb{P}}\big[Q(\pi^*)
    - Q(\hat{\pi})\big] \geq \dfrac{M}{24} 
    \exp \Big(- \frac{1}{3d} \cdot \sum_{t=1}^T g_t\Big)\geq \frac{M}{24e^3},
\end{align*}}
which concludes the proof.
 
A direct consequence of Theorem~\ref{theorem:lower_bound} is the following corollary.
\textcolor{black}{\begin{corollary}
\label{corollary:lower_bound}
Given $\alpha \in [0, 1)$, let $\mathcal{P}(\{t^{-\alpha}\})$ be the collection of all laws of data generating process for which Assumption \ref{assumption:dgp} holds 
with the assignment probability lower bound sequence $g_t=t^{-\alpha}$. 
For any $\hat{\pi}\in\Pi$ learned from $\calZ_T$, there is
\begin{align*}
\sup_{ \mathbb{P}\in\mathcal{P}(\{t^{-\alpha}\})}~\bE_{\calZ_T \sim \mathbb{P}}\big[Q(\pi^*) - Q(\hat{\pi})\big]
\ge \frac{M}{8e^3} \cdot \min\big(1/3, \sqrt{(1-\alpha) \ndim(\Pi)} \cdot T^{\frac{\alpha-1}{2}}\big).
\end{align*}
\end{corollary}}
The proof is completed by noticing that $\sum_{t=1}^T 
g_t \leq {T^{1-\alpha}}/{(1-\alpha)}$. 

\begin{remark}
The overlap
condition assumed in literature that studies \textbf{i.i.d.}~data~\citep{imbens2004nonparametric,athey2021policy,zhou2022offline} is a special case of  Corollary~\ref{corollary:lower_bound} with $\alpha=0$, which informs that 
the regret of any learned policy is at least $\Omega(T^{-{1}/{2}})$ under the overlap condition.
\end{remark}

\section{Regret Upper Bound}
\label{sec:upper_bnd}
In this section, we introduce our offline policy learning algorithm based on the 
generalized AIPW estimator, 
followed by a regret analysis of the algorithm and a discussion on the choice of weights. 

\subsection{The Policy Learning Algorithm}
Our proposed algorithm consists of two steps.
First, it estimates the value of a policy $\pi$
via reweighting the AIPW scores, where the 
\textcolor{black}{AIPW score} for $t\in[T]$ is defined as follows
\begin{align}
\label{eq:aipw_score}
\widehat{\Gamma}_t(\pi)  = 
\widehat{\mu}_t\big(X_t; \pi(X_t)\big)+ 
\frac{\one \big\{W_t=\pi(X_t)\big\}}{e_t\big(X_t;\pi(X_t)\big)}\cdot 
\Big(Y_t - \widehat{\mu}_t\big(X_t; \pi(X_t)\big)\big).
\end{align}
Above, $\widehat{\mu}_t(x;w)$ is an estimator of the  expectation $\mu(x;w)$ using  $\calH_{t-1}$, and $e_t(X_t;w)$ is the 
assignment probability that is also computed based on $\calH_{t-1}$ and is known to the algorithm by Assumption~\ref{assumption:dgp}.
The algorithm then outputs the policy maximizing the estimated policy value. 

We start from some properties of the AIPW scores that will be useful later.
\begin{proposition}[\citet{hadad2021confidence,zhan2021off}]
\label{proposition:gammahat}
The AIPW score $\widehat{\Gamma}_t(\pi)$ has
the following two properties.
\begin{enumerate}[label = (\alph*)]
    \item Conditional unbiasedness: 
      $\bE\big[\widehat{\Gamma}_t(\pi) \mid \calH_{t-1}, X_t \big] =Q \big(X_t, \pi(X_t)\big)$.
    \item Bounded conditional variance: there exist positive constants $L, U >0$, such that
    \begin{align*}
      L\cdot \bE \big[e_t\big(X_t; \pi(X_t)\big)^{-1} \biggiven \calH_{t-1}\big] \le
      \mathrm{Var}\big( \widehat{\Gamma}_t(\pi) \mid \calH_{t-1} \big)  \le
    U\cdot \bE \big[e_t(X_t; \pi(X_t))^{-1} \biggiven \calH_{t-1}\big].
    \end{align*}
\end{enumerate}
\end{proposition}

A direct consequence of Proposition \ref{proposition:gammahat}(a) is that any weighted average
of $\widehat{\Gamma}_t(\pi)$ is also unbiased; Proposition \ref{proposition:gammahat}(b)
reveals that the conditional variance of the AIPW score $\widehat{\Gamma}_t(\pi)$ scales with $\mathbb{E}\big[e_t(X_t; \pi(X_t))^{-1}\mid \calH_{t-1}\big]$. As 
is often the case in adaptive experiments, $e_t(X_t;w)$ goes to zero for some suboptimal action $w$ as $t$ increases;
consequently,
the term $e_t(X_t; \pi(X_t))^{-1}$ may
go to infinity, and the 
variance of $\hat{\Gamma}_t(\pi)$ may explode.  
To offset the potentially large variance of  $\widehat{\Gamma}_t$, 
we further introduce a weight $h_t$ (the choice of which will be discussed soon) 
to balance the variance of these AIPW scores. This gives our generalized AIPW estimator: 
\begin{equation}
\label{eq:estimator}
    \widehat{Q}_T(\pi) \stackrel{\Delta}{=} \frac{\sum^T_{t=1} h_t \widehat{\Gamma}_t(\pi)}{\sum^T_{t=1}h_t}.
\end{equation}
Our algorithm then selects the policy that maximizes the above estimator:
\begin{align}
\label{eq:learned_policy}
    \hat{\pi} = \argmax_{\pi \in \Pi}~\widehat{Q}_T(\pi).
\end{align}
\begin{remark}
Unlike the choice of 
weights in~\citet{luedtke2016statistical,hadad2021confidence,zhan2021off}, 
here $\{h_t\}_{t\in[T]}$ are pre-specified and \emph{not} adaptive. 
In fact, the choice of $h_t$ should take into account the worst-case variance of $\widehat{\Gamma}_t$ over the  policy 
class and all possible data realizations (see more details in Section \ref{section:weights}).
\end{remark}

\subsection{Regret Analysis}
\textcolor{black}{Below we  state  the main condition 
on our weighting scheme.
\begin{assumption}
\label{assumption:weights}
The weights $\{h_t\}$ used in \eqref{eq:estimator} satisfy that
\begin{equation*}
L_T(h,g) 
\stackrel{\Delta}{=}
\frac{\sum^T_{t=1} h_t^4 /g_t^3}
{\big(\sum_{t=1}^T h_t^2/g_t\big)^{2}}   \rightarrow 0,
\quad \mbox{as}\quad T\rightarrow +\infty.
\end{equation*}
\end{assumption}
This assumption specifies a  regularity condition on the weights, 
which controls the higher moments of the estimator $\widehat{Q}_T(\pi)$ 
in \eqref{eq:estimator} in the worst cases and thus enables us 
to introduce martingale concentration results to prove the uniform 
convergence of $\widehat{Q}_T(\pi)$.  An analogous assumption is 
also required for weighting schemes used in policy inference with adaptive data  \citep{hadad2021confidence,zhan2021off}.
We now present the bound on the expected regret
for the policy obtained via~\eqref{eq:learned_policy}.}

\textcolor{black}{\begin{theorem}
\label{theorem:upper_bound}
Suppose Assumptions~\ref{assumption:dgp} and \ref{assumption:weights} hold.
For $T$ such that $L_T(h,g) < 1/8$,
the expected regret of the policy $\hat{\pi}$ 
given by \eqref{eq:learned_policy}  can be upper bounded as,
\begin{equation*}
\label{eq:regret_upper_bound}
\mathbb{E}\big[R(\hat{\pi})\big] \leq 
100M\sqrt{K} \cdot 
\Bigg(19 \kappa(\Pi) +  
7\sqrt{\log\bigg(\tfrac{\sum^T_{t=1}h_t/g_t}{ \sqrt{\sum^T_{t=1} h_t^2/g_t}}\bigg)}
+29\Bigg)
\cdot \frac{\sqrt{\sum_{t=1}^T h_t^2/g_t}}{\sum_{t=1}^T h_t}
+ 4M L_T(h,g).
\end{equation*}
\end{theorem}}


\begin{remark}
When $g_t$ is a positive constant (which corresponds to the overlap condition), 
the expected regret bound with uniform weights ($h_t=1$) is $\tilde{O}(\kappa(\Pi)\cdot T^{-{1}/{2}})$; this recovers the 
rate achieved by \cite{zhou2022offline} with \textbf{i.i.d.}~data. 
\end{remark}


Here, we only discuss the proof of Theorem~\ref{theorem:upper_bound}
at a high level and defer the details to Section~\ref{section:proof}. 
The main idea is to  bound the regret of $\hat{\pi}$  by the worst-case estimation error of $\widehat{Q}_T$ over $\Pi$:
\begin{equation*}
\label{eq:regret}
\begin{split}
   R(\hat{\pi}) = Q(\pi^*) - Q(\widehat{\pi}) &= \big( Q(\pi^*) - \widehat{Q}_T(\pi^*) \big) + 
    \big( \widehat{Q}_T(\pi^*) - \widehat{Q}_T(\widehat{\pi}) \big) + \big( \widehat{Q}_T(\widehat{\pi}) - Q(\widehat{\pi}) \big) \\
    &\le  2 \max_{\pi \in \Pi}~|Q(\pi) - \widehat{Q}_T(\pi)|,
\end{split}
\end{equation*}
where the inequality uses the fact that $\hat{\pi}$ maximizes $\widehat{Q}_T(\pi)$. It thus suffices to bounding the quantity 
$\max_{\pi \in \Pi} |Q(\pi) - \widehat{Q}_T(\pi)|$. With $h_t$ being  independent of policy $\pi$ and data realization,
we have
\begin{align*}
    &\max_{\pi\in\Pi}~\big|Q(\pi) - \widehat{Q}_T(\pi)\big| \le 
    \Big(\sum_{t=1}^T h_t\Big)^{-1} \cdot \max_{\pi\in\Pi}~\Big|\sum_{t=1}^T h_t \cdot \big(\widehat{\Gamma}_t(\pi)-Q(\pi)\big)\Big|\\
    = & \Bigbracket{\sum_{t=1}^T h_t}^{-1} \cdot \max_{\pi\in\Pi}~\Bigbar{\sum_{t=1}^T h_t \cdot  
    \big(\widehat{\Gamma}_t(\pi)-Q(X_t,\pi) + Q(X_t,\pi)-Q(\pi)\big)}\\
    \le & \Bigbracket{\sum_{t=1}^T h_t}^{-1} \cdot
    \bigg(\max_{\pi\in\Pi}~\Bigbar{\sum_{t=1}^T h_t \cdot \big(\widehat{\Gamma}_t(\pi)-Q(X_t,\pi)\big)}
    + \max_{\pi\in\Pi}~\Bigbar{\sum_{t=1}^T h_t \cdot \big(Q(X_t,\pi)-Q(\pi)\big)}\bigg).
\end{align*}
Define the $\sigma$-field $\mathcal{F}_t \stackrel{\Delta}{=} 
\sigma(\calH_{t}, X_{t+1})$.  Then by Proposition \ref{proposition:gammahat}(a) the term 
$\sum_{t=1}^T h_t (\widehat{\Gamma}_t(\pi)-Q(X_t,\pi))$ is a martingale difference sequence w.r.t.~the filtration $\{\mathcal{F}_t\}_{t \ge 1}$;
the term $\sum_{t=1}^T h_t( Q(X_t,\pi)-Q(\pi))$ is an empirical process with \textbf{i.i.d.}~random variables.
We shall establish uniform concentration results for these two terms in Section \ref{section:proof} separately. 

\subsection{Choice of \texorpdfstring{$h_t$}{weights}}
\label{section:weights}
We proceed to discuss how to choose weights $h_t$ to satisfy Assumption \ref{assumption:weights} and sharpen the regret bound  established in Theorem~\ref{theorem:upper_bound}. 
Two scenarios are considered: one with assignment probability lower bound $g_t$ disclosed and the other without; we summarize the procedure in Algorithm \ref{algorithm}.

\begin{algorithm}[htb]
\caption{Policy Learning via Generalized AIPW Estimator}
\label{algorithm}
\textbf{Input:} dataset $\{(X_t, W_t, Y_t)\}_{t=1}^T$; policy class $\Pi$.\\

\For{$t=1,\dots, T$}
{   
\begin{enumerate}
\item Fit plug-in estimator $\hat{\mu}_t(\cdot;w)$ for $\mu(\cdot;w)$  using data $\{(X_s, W_s, Y_s)\}_{s=1}^{t-1}$, for all $w\in\calW.$
\item Construct the AIPW estimator $\widehat{\Gamma}_t(\pi)=\widehat{\mu}_t\big(X_t; \pi(X_t)\big)+ 
\frac{\one\big \{W_t=\pi(X_t)\big\}}
{e_t\big(X_t;\pi(X_t)\big)}\cdot \Big(Y_t - \widehat{\mu}_t\big(X_t; \pi(X_t)\big)\Big)$.
\end{enumerate}
}

\If{$g_t$ is known}
{
Set $h_t=g_t$.
}
\Else
{
Set $h_t=1$.
}

Construct generalized AIPW estimator $ \widehat{Q}_T(\pi) =\sum^T_{t=1} h_t \widehat{\Gamma}_t(\pi)\big/\sum^T_{t=1}h_t.$\\
\textbf{Return:} $\widehat{\pi} = \argmax_{\pi \in \Pi}~\widehat{Q}_T(\pi)$.
\end{algorithm}

\subsubsection{Scenario I: \texorpdfstring{$g_t$}{exploration lower bound} is known.} 
\textcolor{black}{With the knowledge of $g_t$, our goal is to analytically solve the \emph{optimal} weight $h^*_t$ that minimizes 
the regret bound presented in Theorem~\ref{theorem:upper_bound}, which we shall soon show to be $h_t^*\propto g_t$.
When $h_t \propto g_t$, Assumption~\ref{assumption:weights} reduces to that 
$\sum_{t=1}^T g_t \rightarrow + \infty$, which is naturally satisfied 
when the learning problem is ``feasible'' (note that when 
$\lim\!\sup_{T\rightarrow +\infty} \sum^T_{t=1} g_t < +\infty$, 
by the lower bound no learning algorithm can do better
than a constant error in the worst case).
Further, the  policy $\hat{\pi}$ obtained with  $h^*_t$  is minimax optimal---the expected regret upper 
bound is $\tilde{O}\big(\kappa(\Pi)\cdot (\sum_{t=1}^T g_t)^{-1/2}\big)$, 
matching the exact lower bound  in Theorem \ref{theorem:lower_bound}
up to logarithmic factors.}

\textcolor{black}{To solve for $h_t^*$, we firstly minimize the term $(\sum_{t=1}^Th_t^2/g_t)/(\sum_{t=1}^T h_t)^2$ in~\eqref{eq:regret_upper_bound}; as we shall soon see, this minimizer   also minimizes the term $L_T(h,g)$ in~\eqref{eq:regret_upper_bound}.
Let $\tilde{h}_t=h_t/(\sum_{s=1}^T h_s)$ be normalized weights.  We  rewrite the original problem into  the  following
convex optimization problem:
\begin{equation*}
\label{eq:optimization_optimal_weights}
    \begin{split}
        \min_{\{ \tilde{h}_t\}_{t\in[T]}}  \quad
        & \sum_{t=1}^T \tilde{h}_t^2/g_t\\
        \mbox{such that}\quad& \sum_{t=1}^T{\tilde{h}_t}=1;\\
        \quad & \tilde{h}_t\geq 0, \quad t \in [T].
    \end{split}
\end{equation*}
The above problem  has a unique minimizer 
$\tilde{h}^*_t=g_t/\big(\sum_{s=1}^T g_s\big)$ 
(see Appendix~\ref{appendix:solve_h} for details). Plugging $\tilde{h}^*$ into $L_T(h,g)$, we have
\begin{equation*}
    L_T(\tilde{h}^*,g) = \frac{1}{\sum_{t=1}^T g_t} \leq \frac{\sum_{t=1}^T h_t^4/g_t^3}{(\sum_{t=1}^T h_t^2/g_t)^2} = L_T(h, g), \mbox{ for any positive weights }\{h_t\},
\end{equation*}
where the inquality is by Cauchy-Schwartz inequality. This choice of $\tilde{h}^*$ yields a minimax optimal regret bound summarized in the  corollary below.}

\textcolor{black}{\begin{corollary}
\label{corollary:optimal_regret}
Suppose that Assumption~\ref{assumption:dgp} holds for assignment probability lower bound $\{g_t\}_{t\in[T]}$,
and that $\sum^T_{t=1} g_t \rightarrow +\infty$ as $T \rightarrow +\infty$.
The optimal weights $\{h_t^*\}_{t\in[T]}$ satisfy $h_t^*\propto g_t$ for each $t\in[T]$, 
with which  the policy $\hat{\pi}$ in \eqref{eq:learned_policy}   achieves minimax optimal regret. 
Specifically, for any $T$ such that $\sum^T_{t=1} g_t \ge 8$,
the expected regret of the policy $\hat{\pi}$ 
given by \eqref{eq:learned_policy}  can be upper bounded as,
\begin{equation*}
\label{eq:optimal_weighting_regret}
\mathbb{E}\big[R(\hat{\pi})\big] \leq 
100M\sqrt{K} \cdot \Bigg(19\kappa(\Pi) + 7\sqrt{\log\bigg(\tfrac{T}{\sqrt{\sum^T_{t=1}g_t}}\bigg)} + 30\Bigg)
\cdot \frac{1}{\sqrt{\sum_{t=1}^T g_t}}.
\end{equation*}
\end{corollary}
The corollary is proven by letting $h_t = g_t$ in 
Theorem~\ref{theorem:upper_bound} and noting 
that $L_T(h,g) = (\sum^T_{t-1}g_t)^{-1} \le (\sum^T_{t=1} g_t)^{-1/2}$.}


\subsubsection{Scenario II: \texorpdfstring{$g_t$}{exploration lower bound} is unknown.} 
In practice, one may not have access to  the assignment probability lower bound $g_t$. 
We claim in such settings, uniform weighting with $h_t=1$ can  also be effective. 
\textcolor{black}{In this case, Assumption~\ref{assumption:weights} reduces to that 
\begin{align*}
\frac{\sum^T_{t=1} 1/g_t^3}{\big(\sum^T_{t=1} 1/g_t\big)^2} \rightarrow 0
\quad\mbox{as}\quad T\rightarrow +\infty,
\end{align*}
which holds in many cases such as $g_t = t^{-\alpha}$ for some $\alpha \in [0,1)$.
The following corollary characterizes the regret incurred by our estimator
with uniform weights.}

\textcolor{black}{\begin{corollary}
\label{corollary:uniform_weighting_regret}
Suppose  Assumption~\ref{assumption:dgp} holds and 
that 
\begin{align*}
\frac{\sum^T_{t=1} 1/g_t^3}{\big(\sum^T_{t=1} 1/g_t\big)^2} 
\rightarrow 0 \quad \mbox{as} \quad T \rightarrow +\infty.
\end{align*}
The expected regret incurred by policy $\hat{\pi}$ in \eqref{eq:learned_policy} with uniform weights $h_t=1$  can be bounded as
\begin{equation*}
\label{eq:uniform_upper_bound}
\mathbb{E}\big[R(\hat{\pi})\big] \leq 
100M\sqrt{K} \cdot \Bigg(19 \kappa(\Pi) + 7\sqrt{\log\big( \sum^T_{t=1} 1/g_t\big)}+29\Bigg)
\cdot \frac{\sqrt{\sum_{t=1}^T 1/g_t}}{T} + 4 M \cdot \frac{\sum^T_{t=1}1/g_t^3}{(\sum^T_{t=1}1/g_t)^2}.
\end{equation*}
\end{corollary}}

\textcolor{black}{In general, the regret bound  $\tilde{O}\big(\kappa(\Pi)\cdot\frac{\sqrt{\sum_{t=1}^T g_t^{-1}}}{T} + \frac{\sum^T_{t=1}1/g_t^3}{(\sum^T_{t=1}1/g_t)^2}\big)$ 
yielded by  uniform weighting is looser than the minimax regret 
$\tilde{O}\big(\kappa(\Pi)\cdot(\sum_{t=1}^Tg_t)^{-1/2}\big)$ yielded by optimal weighting, 
which can be verified by noticing that $\frac{\sqrt{\sum_{t=1}^T g_t^{-1}}}{T}\geq(\sum_{t=1}^Tg_t)^{-1/2}$ as a result of Cauchy-Schwarz inequality. However, in some cases, these two 
achieve the same  regret decay rate---both are minimax optimal, as illustrated in the following example.}


\subsubsection{A case study.}
\label{section:case_study}
\textcolor{black}{To provide  intuition,  consider a special case
where the assignment probability lower bound decays polynomially,
and in specific we let $g_t = t^{-\alpha}$ for some $\alpha\in[0,1)$. 
We consider weights  $h_t=t^{-\beta}$ for some nonnegative $\beta$ and  
study how $\beta$ affects the regret of $\hat{\pi}$ obtained via \eqref{eq:learned_policy}.
Theorem \ref{theorem:lower_bound} shows that the expected regret (in terms of $T$)
is lower bounded by $\Omega(T^{(\alpha-1)/2})$.
Assumption~\ref{assumption:weights} holds for any $\beta < \frac{\alpha+1}{2}$;
then with Theorem \ref{theorem:upper_bound}, the expected regret upper bound with $h_t = t^{-\beta}$ for $\beta\in(0, \frac{\alpha+1}{2})$ is $\tilde{O}(\kappa(\Pi)\cdot T^{(\alpha-1)/2}+ T^{4\beta-2\alpha-2})$. In particular, when  $\beta\leq \frac{5\alpha+3}{8}$, our algorithm achieves  expected regret bounds of $\tilde{O}(\kappa(\Pi)\cdot T^{(\alpha-1)/2})$, which matches the exact lower bound and is
thus minimax optimal.
Notably,  uniform weighting (which does not require $g_t$ known) achieves the minimax optimal regret upper bound, and this bound  inflates the one obtained by optimal weights (which needs $g_t$ to be disclosed)  by a  factor of $(1-\alpha^2)^{-1/2}$. In this regard, when the assignment probability lower bound does not decay too fast in the sense that $g_t=t^{-\alpha}$ for some $\alpha\in[0,1)$, the minimax optimality of our algorithm is agnostic to the knowledge of  $g_t$.}


\section{Proof of the Upper Bound}
We now establish the regret bound given in Theorem \ref{theorem:upper_bound}. The problem has been decoupled to showing two uniform 
concentration results for: (i) the martingale difference
sequence $\sum_{t=1}^T h_t \big(\widehat{\Gamma}_t(\pi) - Q(X_t, \pi)\big)$, where we condition on  an event that the quadratic variation of AIPW scores $\widehat{\Gamma}_t$ is well regularized across the policy class, which happens with high probability under Assumption \ref{assumption:weights}; and (ii)  the sum of independent variable sequence 
$\sum_{t=1}^T h_t\big(Q(X_t, \pi)-Q(\pi)\big)$, where we apply standard techniques in analyzing empirical processes with \iid~data. 

\label{section:proof}
\subsection{Uniform Concentration of Martingale Difference Sequences}
Define $\M_T(\pi) \stackrel{\Delta}{=} 
\sum^T_{t=1} h_t \big(\widehat{\Gamma}_t(\pi) - Q(X_t, \pi)\big)$. 
Our goal is to show that with high probability, $\max_{\pi\in\Pi}|\M_T(\pi)|$ is small. 
\textcolor{black}{Particularly, we shall restrict our analysis  on the event below:
\begin{equation*}
\calB_T \stackrel{\Delta}{=}\bigg\{  \sup_{\pi\in\Pi}\sum_{t=1}^Th_t^2\cdot\big(\hgamma_t(\pi)
- Q(X_t;\pi)\big)^2 \leq 10KM^2 \cdot \sum_{t=1}^T\frac{h_t^2}{g_t}\bigg\},
\end{equation*}
In the following, 
we write $\chg \stackrel{\Delta}{=} 10KM^2 \cdot \sum^T_{t=1} h_t^2/g_t$ for notation convenience.  
On the event~$\calB_T$, the  quadratic variation of $\M_T(\pi)$ is 
controlled, which is critical in showing a fast uniform concentration 
rate of $\M_T$. The following lemma quantifies the probability
of  the event $\calB_T$. 
\begin{lemma}
\label{lemma:high_probability_event}
Under Assumption \ref{assumption:dgp},
$P( \calB_T)> 1- 2\cdot L_T(h,g)$.
\end{lemma}
The proof is deferred to Appendix \ref{appendix:proof_of_high_probability_event}. Lemma \ref{lemma:high_probability_event} immediately implies that under Assumption \ref{assumption:weights}, event $\calB_T$ happens with high probability.
Moving on, we shall show that $\M_T$ concentrates uniformly at a fast rate when $\calB_T$ happens.}

The traditional symmetrization technique is useful for proving
uniform concentration results with \textbf{i.i.d.}~data, but is not directly applicable given the adaptive nature of our data. Motivated by
\cite{rakhlin2015sequential}, we leverage a sequential analog of the symmetrization technique to obtain the uniform concentration for martingale empirical process. 
Note that results in \cite{rakhlin2015sequential}, which require  difference elements to be bounded, cannot be directly invoked in our setting where  each element may diverge. \textcolor{black}{To still allow for  sequential uniform concentration, we take a different route and condition on event $\calB_T$, such that the sum of quadratic terms is under control.}

We present the proof of uniform concentration of $\M_T(\pi)$ in two steps. First, we connect the martingale 
empirical process with a tree Rademacher process (defined shortly afterwards), and next 
we connect the tree Rademacher process with the entropy integral of the policy class $\kappa(\Pi)$. 

\paragraph{Step 1: Connecting the Martingale Empirical Process with a Tree Rademacher Process.}
We start by defining the notion of a tree, following~\citet{rakhlin2015sequential}. 
\begin{definition}
\label{defn:tree}
A $\calZ$-valued tree $\z$ of depth $T$ is a rooted complete binary tree with nodes labeled by elements of $\calZ$. 
\end{definition}
A tree $\z$ is identified with a sequence of 
labeling functions $(\z_1,\ldots, \z_T)$,
where $\z_i:\{\pm 1\}^{i-1} \mapsto \calZ$ 
labels the nodes on the $i$-th level.
To be more speficic, $\z_1$ refers to 
the root node; $\z_i$ for $i>1$ refers to
the node on the $i$-th level of the 
tree with the following rule: 
for a $\{\pm 1\}$-valued sequence of 
length $i-1$, $\z_i$ maps the sequence to a node
by following a path on the tree, with $-1$ referring to ``left'' and $+1$ to ``right''. As an example, Figure~\ref{fig:illustratin_tree} plots
a tree of depth $3$, where $\z_3(-1,-1)$ corresponds to the blue node,
and $\z_3(1,-1)$ is the red node.
\begin{figure}[ht]
\centering
    \begin{tikzpicture}[
        bluecirc/.style={circle, draw=gray!80, fill=gray!5, thick, minimum size=5mm},
        redcirc/.style={circle, draw=blue!60, fill=blue!5, thick, minimum size=5mm},
        greencirc/.style={circle, draw=red!60, fill=red!5, thick, minimum size=5mm},
    ]   
    \node[bluecirc] (root) {};
    \node (l1_sep) [below=15mm of root] {};
    \node[bluecirc] (child1) [left=15mm of l1_sep] {};
    \node[bluecirc] (child2) [right=15mm of l1_sep] {};
    \node (l2_sep1) [below=10mm of child1] {};
    \node[redcirc] (child11) [left=3mm of l2_sep1] {};
    \node[bluecirc] (child12) [right=3mm of l2_sep1] {};
    \node (l2_sep2) [below=10mm of child2] {};
    \node[greencirc] (child21) [left=3mm of l2_sep2] {};
    \node[bluecirc] (child22) [right=3mm of l2_sep2] {};
    \node [above left=6mm and 9mm of l1_sep] {\small $\epsilon_1 = -1$};
    \node [above right=6mm and 9mm of l1_sep] {\small $\epsilon_1 = 1$};
    \node [above right=3mm and 2mm of l2_sep2] {\small $\epsilon_2 = 1$};
    \node [above left=3mm and 2mm of l2_sep2] {\small $\epsilon_2 = -1$};
    \node [above right=3mm and 2mm of l2_sep1] {\small $\epsilon_2 = 1$};
    \node [above left=3mm and 2mm of l2_sep1] {\small $\epsilon_2 = -1$};
    \draw[-] (root) -- (child1);
    \draw[-] (root) -- (child2);
    \draw[-] (child1) -- (child11);
    \draw[-] (child1) -- (child12);
    \draw[-] (child2) -- (child21);
    \draw[-] (child2) -- (child22);
    \end{tikzpicture}
    \vspace{1em}
    \caption{Illustration of a tree of depth $3$. The blue
    node corresponds to $\z_3(-1,-1)$, and the red node corresponds
    to $\z_3(1,-1)$.}
    \label{fig:illustratin_tree}
\end{figure}
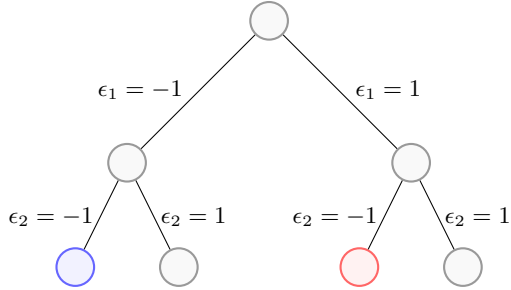

We proceed to state the definition of a \emph{tree Rademacher process}:
\begin{definition}[Definition 2 of~\citet{rakhlin2015sequential}]
Suppose $\calZ$ is the sample space. 
Let $\{f_{\pi}: \mathcal{Z} \rightarrow \bR \mid \pi \in \Pi\}$ 
be a class of functions indexed by
$\Pi$ and $\epsilon_1,\ldots,\epsilon_T$ be independent Rademacher random variables such that 
$\p(\epsilon_t = 1) = \p(\epsilon_t=-1) = {1}/{2}$. Given a $\mathcal{Z}$-valued
tree $\z$ of depth $T$, we define the following stochastic process as a tree Rademacher process (indexed by $\Pi$):
\begin{align*}
    f_{\pi} \rightarrow \sum^T_{t=1} \epsilon_t \cdot 
    f_{\pi}\big(\z_t(\epsilon_1,\ldots,\epsilon_{t-1})\big)
\end{align*}
\end{definition}

\textcolor{black}{Defining $Z_t \stackrel{\Delta}{=} (W_t,Y_t,\{e_t(X_t;w)\}_{w\in\calW}, \{\hat{\mu}_t(X_t;w)\}_{w\in\calW})$,
we write $f(h_t,X_t,Z_t;\pi) = h_t\cdot\hat{\Gamma}_t(\pi)$.} Let $\mathcal{Z}$ be the space $Z_t$ lives in and 
$\z$ be a $\calZ$-valued binary tree. For notational simplicity, we write  $\bx=\{x_{1}, \dots, x_T\}$ to denote
realized values of the covariates and $z_t(\epsilon) = 
\z_t(\epsilon_1,\ldots,\epsilon_{t-1})$ to denote a node at depth $t$.
Following~\citet{rakhlin2015sequential}, we  introduce a \emph{decoupled tangent sequence} (which is  similar to 
the symmetrized sequence in the \iid~case and is defined in Definition~\ref{def:dts}) with respect to the data sequence $(Z_1,\ldots,Z_T)$. 
\textcolor{black}{Similar to analyzing the martingale empirical process on event $\calB_T$, we also restrict our analysis of the tree Rademacher process on the event $\calB_T(\bx,\z)$, which is defined  with respect to a realization of covariates $\bx$ and a tree $\z$ as follows:
\begin{align*}
  \mathcal{B}_T(\mathbf{x},\z) \stackrel{\Delta}{=} 
\Big\{\epsilon \in \{\pm 1\}^T: 
  \sup_{\pi'} \sum^T_{t=1} \big(f(h_t,X_t,z_t(\epsilon);\pi')
- h_t Q(X_t;\pi')\big)^2 \le 2\chg\Big\}.
\end{align*}
Lemma~\ref{lemma:bound_empirical_process_prob}  
shows that the tail  of  $\sup_{\pi\in\Pi} |\M_T(\pi)|$  on the event $\calB_T$
is bounded by the tail of the supremum of tree Rademacher process on the event $\calB_T(\bx,\z)$.
\begin{lemma}
\label{lemma:bound_empirical_process_prob} 
Suppose that  Assumption \ref{assumption:dgp} holds, 
and that $L_T(h,g) \le 1/8$. 
Then for any $\eta> \sqrt{8C_T(h,g)/KT^2}$, it holds that
\begin{align*}
&\quad \mathbb{P}\Big(\sup_{\pi\in\Pi} \big|\M_T( \pi)\big|\ge \eta T,
~\calB_T
\,\big|\, \bx\Big) 
\le  4\cdot\sup_{\z}\,\mathbb{P}_{\epsilon}
\Big(\sup_{\pi\in\Pi}\big| \sum^T_{t=1} \epsilon_t
f(h_t,X_t, z_t(\epsilon);\pi)\big|\ge \dfrac{\eta T}{4},
~\calB_T(\bx,\z)\,\Big|\, \bx,\z
\Big),
\end{align*}
where 
$\z$ is a $\calZ$-valued tree of depth $T$, and $\p_{\epsilon}$ denotes the distribution of the Rademacher random variables 
$\{\epsilon_t\}^T_{t = 1}$.
\end{lemma}}
The proof  is deferred to Appendix~\ref{appendix:proof_lemma_bound_empirical_process_prob}. 

\vspace{0.3cm}
\paragraph{Step 2: Connecting the Tree Rademacher Process with Policy Class Complexity.} 
We now proceed to connect the tail bound of the supremum of  tree Rademacher process with the complexity of the policy class.
Note that this Rademacher process is defined on a tree, and the Hamming distance in 
Definition~\ref{definition:hamming} cannot be directly applied here to characterize Rademacher complexity.
Alternatively we adopt the notion of distance between policies on the tree process defined in
\cite{rakhlin2015sequential} and modify it slightly for notational convenience in the proof.
\begin{definition}
\label{def:cover}
Given covariates $\bx$ and a 
tree $\z$ of depth $T$:
\begin{enumerate}[label = (\alph*)]
\item The $\ell_2$ distance between 
two policies $\pi_1,\pi_2\in\Pi$ 
w.r.t.~$\bx$, $\z$ and a sequence
of $\{\pm 1\}$-valued random variables
$\epsilon_{1:T}$ is defined as 
\[
\ell_2(\pi_1,\pi_2; \z, \bx, \epsilon_{1:T}) 
\stackrel{\Delta}{=}\sqrt{\frac{\sum_{t=1}^T 
\Big(f\big(X_t,z_t(\epsilon);\pi_1\big)
- f\big(X_t,z_t(\epsilon);\pi_2\big) \Big)^2}{\textcolor{black}{16\chg}}}.
\]

\item A set \textcolor{black}{$\mathcal{S}\subset\Pi$} is a (sequential)
$\eta$-cover of a policy class $\Pi$ under the $\ell_2$-distance 
w.r.t.~$\bx$ and $\z$, if for any $\pi\in\Pi$
and any
\textcolor{black}{$\epsilon_{1:T}\in \mathcal{B}_T(\mathbf{x},\z)$},
there exists some $s\in \mathcal{S}$ such that   
$\ell_2(s,\pi; \z,\bx,\epsilon_{1:T})\le \eta$.
    
\item The $\eta$-covering number of a policy class $\Pi$ 
w.r.t.~$\bx$ and $\z$ is defined as 
\begin{align*}
N_2(\eta, \Pi; \z, \bx) = 
\min\big\{|S|: S \mbox{ is an }\eta
\mbox{-cover}\mbox{ of }\Pi\mbox{ under the }
\ell_2 \mbox{ w.r.t.~}\z\mbox{ and } \bx\big\}.
\end{align*}
\end{enumerate}
\end{definition}
\vspace{0.3cm}
The
Hamming distance is connected with the $\ell_2$ distance as follows.
\begin{lemma}
\label{lemma:l2_distance} Under Assumption~\ref{assumption:dgp},
for any realization of covariates $\bx$, 
any tree $\z$ of depth $T$
and for any $\eta>0$, we have 
$N_2(\eta, \Pi; \z, \bx)\leq N_\ham(\eta^2, \Pi)$.
\end{lemma}
The proof of Lemma~\ref{lemma:l2_distance} is provided in Appendix \ref{appendix:l2_distance}.
We are now ready to bound the tail  of supremum of tree Rademacher process using the entropy integral defined under Hamming distance. 

\textcolor{black}{\begin{lemma}
\label{lemma:bound-Rademacher-complexity} 
Under Assumptions~\ref{assumption:dgp}, 
consider a realization of covariates $\bx$ 
and a tree $\z$ of depth $T$. Given any 
$\delta \in (0,1)$, 
\begin{align*}
&\PP_{\epsilon}\Big(\max_{\pi\in\Pi}\Bigbar{\sum_{t=1}^T 
\epsilon_t f(h_t,X_t, z_t(\epsilon); \pi)  } 
\ge \zeta ,~\calB_T(\bx,\z)\,\big|\, \bx,\z \Big) \le \delta,\\
 \mbox{where}\quad &\zeta =  24\sqrt{\chg} 
\Big( 2\sqrt{2} + 2\sqrt{2}\kappa(\Pi) + 
\sqrt{\log\bigbracket{5/(3\delta)}} + 1/\sqrt{T}\Big).
\end{align*}
\end{lemma}}

We refer readers to Appendix \ref{appendix:proof_bound-Rademacher-complexity} for the complete proof of 
Lemma~\ref{lemma:bound-Rademacher-complexity}. We here present the three key parts in our proof to provide some intuition. 

\begin{itemize}
    \item {\em Part I: Policy decomposition}. \label{step_1}
    Let
    $J=\lceil \log_2(T) \rceil$. Given $\epsilon=\epsilon_{1:T}$, we define a sequence of projection operators $A_0,A_1,\dots, A_J$,
    \textcolor{black}{where each $A_j$ maps a policy $\pi$ to its $j$-th approximation $A_j(\pi;\epsilon)$.}
    As $j$ increases from $0$ to $J$, the approximation becomes finer:  $\{A_0(\pi;\epsilon)\}$ is a 
    singleton indicating the coarsest approximation, and $A_J$ refers to the finest approximation. 
   The construction of such a sequence will be discussed in detail in  Appendix \ref{appendix:proof_bound-Rademacher-complexity}.  We decompose the tree 
   Rademacher process as follows: 
    \begin{align}
    \label{eq:process_decomposition_main}
    \sum_{t=1}^T \epsilon_t f\big(x_t, z_t(\epsilon); \pi\big) =&
    \underbracket[0.5pt]{\sum_{t=1}^T \epsilon_t \Big(\sum_{j=1}^{J}f\big(x_t,z_t(\epsilon); A_{j}(\pi;\epsilon)\big) 
    -f\big(x_t, z_t(\epsilon); A_{j-1}(\pi;\epsilon)\big)
    + f\big(x_t,z_t(\epsilon);A_0(\pi;\epsilon)\big)\Big)}_{\scriptsize \text{term (i): effective}}\nonumber \\
    & + \underbracket[.5pt]{\sum_{t=1}^T \epsilon_t \Big( f\big(x_t, z_t(\epsilon); \pi\big)-
    f\big(x_t,  z_t(\epsilon); A_{J}(\pi;\epsilon) \big)\Big)}_{{\scriptsize \text{term (ii): negligible}}}
    \end{align}
\textcolor{black}{Above, term (i) is of the order $\sqrt{\chg}$
 with high probability, and terms (ii) is negligible w.r.t.~the
    first term, which are shown separately in the following steps.}
    
    \item \emph{Part II: The Effective Term.} \label{step_3} 
    \textcolor{black}{Given  a realization of the covariates $\bx$ and a tree $\z$ of depth $T$,  for any $\delta>0$, with probability at least $1-\delta$, either $\epsilon\notin \calB_{T}(\bx,\z)$, or there is
    \begin{align*}
    &\max_{\pi\in\Pi}\Bigbar{\sum_{t=1}^T \epsilon_t \Bigbracket{ \sum_{j=1}^{\underline{J}}f\big( x_t, z_t(\epsilon); 
    A_{j}(\pi;\epsilon)\big) -f\big(x_t, z_t(\epsilon); A_{j-1}(\pi;\epsilon)\big) + f\big(x_t, z_t(\epsilon);A_0(\pi;\epsilon)\big)}  }\\
    &\qquad \qquad \qquad<  24\sqrt{\chg}\Big( 2\sqrt{2} + 2\sqrt{2}\kappa(\Pi) + 
    \sqrt{\log\bigbracket{5/(3\delta)}}\Big).
    \end{align*}
    where the probability is w.r.t.~the distribution of the Rademacher sequence $\epsilon$.}
    
\item \emph{Part III: The Negligible Terms.} \label{step_2}
\textcolor{black}{Term (ii) in \eqref{eq:process_decomposition_main}  is 
negligible with respect to $\sqrt{\chg}$ on the event $\calB_T(\bx,\z)$, and specifically,
\begin{align*}
\sup_{\pi\in\Pi}   
\bigg|\sum_{t=1}^T \epsilon_t \cdot \Big( f\big(x_t, z_t(\epsilon);\pi) 
-f\big(x_t, z_t(\epsilon); A_J(\pi;\epsilon)\big) \Big)\bigg| 
\le  4\sqrt{\chg/T}.
\end{align*}}
\end{itemize}

\endproof

\subsection{Uniform Concentration of Independent Difference Sequences} 

We now establish concentration results for $\sum_{t=1}^T h_t\big(Q(X_t,\pi)-Q(\pi)\big)$. We 
note that covariates $X_{1:T}$ are exogenous; thus $\{Q(X_t, \pi)\}_{t=1}^T$ are bounded \iid~random 
variables. We shall apply the standard toolkit for the uniform concentration of \iid~data.

\begin{lemma}
\label{lemma:iid_concentration}
Under Assumption \ref{assumption:dgp},  with probability at least $1-\delta$, 
\begin{equation*}
\label{eq:iid_concentration}
    \max_{\pi\in\Pi}~\Big|{\sum_{t=1}^Th_t\big(Q(X_t,\pi)-Q(\pi)\big)}\Big| 
    \leq 8\sqrt{2}M\textcolor{black}{\sqrt{\sum_{t=1}^T h_t^2}}\Big(13 + 4\kappa(\Pi) 
    + \sqrt{\log(1/\delta)} + 1/\sqrt{T} \Big).
\end{equation*}
\end{lemma}
The proof is deferred to Appendix~\ref{appendix:iid_concentration} for details. At a high level, 
we shall decompose $\sum_{t=1}^T h_t\big(Q(X_t,\pi)-Q(\pi)\big)$ via a sequence of policy approximation
operators, and bound each difference with standard uniform concentration techniques for \iid~data.


\subsection{Expected Regret Upper Bound}
\textcolor{black}{We are now ready to put the pieces together and prove Theorem~\ref{theorem:upper_bound}. Suppose that Assumptions \ref{assumption:dgp} and \ref{assumption:weights} hold.
Consider a $T$ such that $L_T(h,g) \le 1/8$.
\begin{align*}
\EE\big[R(\hat{\pi})\big] = \EE\big[R(\hat{\pi}) \cdot \one \{\calB_T\}\big]
+ \EE\big[R(\hat{\pi})\cdot \one \{\calB_T^c\} \big]
\le &\EE\big[R(\hat \pi) \cdot \one\{\calB_T\}\big] + 2M \PP(\calB_T^c)\\
\le &\EE\big[R(\hat \pi) \cdot \one\{\calB_T\}\big] + 4ML_T(h,g),
\end{align*}
where the last inequality is a result of Lemma~\ref{lemma:high_probability_event}.
Next, using the decomposition introduced earlier,
\begin{align}
\label{eq:expected_regret}
\EE\big[R(\hat \pi) \cdot \one \{\calB_T\}\big] 
\le & 2\Big(\sum^T_{t=1} h_t\Big)^{-1} \cdot \Bigg\{ 
\EE\bigg[\sup_{\pi \in \Pi} \Big|\sum_{t=1}^T h_t \big(\hgamma_t(\pi) - Q(X_t;\pi)\big)\Big| \cdot \one \{\calB_T\}\bigg]\notag\\
&\qquad \qquad \qquad \qquad +\EE\bigg[\sup_{\pi \in \Pi} \Big|\sum^T_{t=1} h_t \big(Q(X_t;\pi) - Q(\pi)\big)\Big|\bigg]\Bigg\}.
\end{align}
For the first expectation,  
letting  $\bar{\zeta} = 24 \sqrt{\chg} \cdot (2\sqrt{2} + 2\sqrt{2} \kappa(\Pi) 
+ \sqrt{\log \Big(\frac{2M \sum^T_{t=1}h_t/g_t}{ \sqrt{\chg}}\Big)} + 1/\sqrt{T} )$, we have
\begin{align*}
\EE\Big[\sup_{\pi\in\Pi} |\M_T(\pi)| \cdot \one\{\calB_T\}\Big]
\le & 4\bar{\zeta} + \EE\Big[\sup_{\pi\in\Pi} \big|\M_T(\pi)\big|
\cdot \one\big\{\calB_T,\sup_{\pi\in\Pi}|\M_T|\ge 4\bar{\zeta} \big\}\Big] \\
\le & 4\bar{\zeta} + 4M\Big(\sum^T_{t=1}\frac{h_t}{g_t}\Big) \cdot 
\PP\big(\calB_T,\sup_{\pi\in\Pi} |\M_T|\ge 4\bar{\zeta}\big)\\
\le & 4\bar{\zeta} + 16 \sqrt{\chg},
\end{align*}
where in the last step we apply Lemma~\ref{lemma:bound_empirical_process_prob} and Lemma~\ref{lemma:bound-Rademacher-complexity}
with $\eta = 4\bar{\zeta} / T$ and $\delta = \frac{1}{M} \frac{\sqrt{\chg}}{\sum^T_{t=1} h_t/g_t}$ (it can be checked that 
$\eta > \sqrt{8\chg / (KT^2)}$).}

\textcolor{black}{For the second expectation, we take $\delta = \sqrt{\chg}/(M\sum^T_{t=1}h_t)$ in Lemma~\ref{lemma:iid_concentration} and get
\begin{align*}
& \EE\bigg[\sup_{\pi \in \Pi} \Big|\sum_{t=1}^T h_t\big(Q(X_t;\pi) - Q(\pi)\big) \Big|\bigg]\\
  \le & 8\sqrt{2}M \sqrt{\sum^T_{t=1} h_t^2} \cdot \bigg(13+4\kappa(\Pi)+\sqrt{\log\Big(\tfrac{M\sum^T_{t=1}h_t}{\sqrt{\chg}}\Big)} + 1/\sqrt{T}\bigg)
  +2\sqrt{\chg}.
\end{align*}}

\textcolor{black}{Summing the two terms up, we have
\begin{align*}
\eqref{eq:expected_regret} \le 
& \frac{M\sqrt{K\sum^T_{t=1}h_t^2/g_t}}{\sum^T_{t=1}h_t}
\cdot \bigg(2200+ 1900 \kappa(\Pi) + 630\sqrt{\log\Big(\tfrac{\sum^T_{t=1} h_t/g_t}{\sqrt{\sum^T_{t=1} h_t^2/g_t}}\Big)} + \frac{630}{\sqrt{T}}\bigg).
\end{align*}}

\textcolor{black}{The proof is hence completed.}

\endproof

\section{Simulations: Policy Learning with Decision Trees}
\label{sec:simulation}
In this section, we provide experimental evidence on the effectiveness of Algorithm \ref{algorithm}, using both synthetic datasets and  classification datasets from OpenML \citep{OpenML2013}. We investigate 1) how offline learning compares with its online counterpart when there is model misspecification; and 2) how different choices of weights $h_t$ influences the regret of offline-learned policy. Throughout the experiments, we use linear models to fit the nuisance estimator $\mu_t$ on data $\calH_{t-1}$.\footnote{Reproduction code can be found at \url{https://github.com/gsbDBI/PolicyLearning}.}

\paragraph{Policy Class.} \textcolor{black}{Exact policy learning via maximizing  policy value estimation  generally leads to a nonconvex optimazation problem and can be infeasible for arbitrary policy classes. We hereby focus on a policy class of decision trees with \emph{fixed} depth, which has a finite entropy integral \citep{zhou2022offline}. To learn the policy that maximizes generalized AIPW estimator, we apply a publicly available solver \texttt{PolicyTree} that finds the global optimum in polynomial runtime via an exhaustive and unconstrained tree search \citep{sverdrup2020policytree}.}
 Algorithm \ref{algorithm:policytree} adapts the software to our problem setting  with customized inputs, so that weights $h_t$ are incorporated in the  value estimator. 
\begin{algorithm}
\caption{Policy Learning via Generalized AIPW Estimator: Invoking \texttt{PolicyTree}}
\label{algorithm:policytree}
\textbf{Input:} dataset $\{(X_t, W_t, Y_t)\}_{t=1}^T$; weights $\{h_t\}_{t=1}^T$; decision tree depth $L$.

\For{$t=1,\dots, T$}
{\begin{enumerate}
\item Fit plug-in estimator $\hat{\mu}_t(\cdot;w)$ for $\mu(\cdot;w)$  using data $\{(X_s, W_s, Y_s)\}_{s=1}^{t-1}$ for all $w\in\calW$.
\item Construct the AIPW estimator $\widehat{\Gamma}_t(\pi)=\widehat{\mu}_t\big(X_t; \pi(X_t)\big) 
+ \frac{\one \{W_t=\pi(X_t)\}}{e_t(X_t;\pi(X_t))}\cdot \Big(Y_t - \widehat{\mu}_t\big(X_t; \pi(X_t)\big)\Big)$.
    \item 
Construct reweighted AIPW estimator $\Tilde{\Gamma}_t = \frac{h_t}{\sum_{s=1}^T h_s}\widehat{\Gamma}_t$.
\end{enumerate}
}

Call \texttt{PolicyTree} with input $\big(\{X_t\}_{t=1}^T, \{\Tilde{\Gamma}_t\}_{t=1}^T, L\big)$, and obtain the learned policy $\widehat{\pi}$.

\textbf{Return:} $\widehat{\pi}$.
\end{algorithm}
\paragraph{Data-Collection Agent.} At each time, the experimenter first computes each arm's preliminary assignment probabilities $\{\bar{e}_t(X_t;w)\}_{w\in\calW}$ based on past observations via a Linear Thompson sampling agent \citep{agrawal2013thompson}; then a lower bound $g_t=t^{-\alpha}/K$ (with $\alpha=0.5$) is imposed:  arms with  $\bar{e}_t(X_t;w)<g_t$  have assignment probability $e_t(X_t;w)=g_t$; others will be shrunk by setting $e_t(X_t;w)=g_t+c(\bar{e}_t(X_t;w)-g_t)$, where $c$ ensures $\sum_{w\in\calW}e_t(X_t;w)=1$. \textcolor{black}{This type of flooring scheme is a generalization of commonly-enforced overlap practice  in randomized controlled trials and has been increasingly used in adaptive experimentation. The floor allows for diminishing exploration on   suboptimal arms, but imposes a positive probability of sampling  each arm to facilitate post-experiment analyses (which often require non-zero assignment probabilities everywhere when using methods based on inverse probability weighting) \citep{offer2021optimal}.}

\subsection{Synthetic Data}
\label{section:synthetic}

We consider a contextual bandit problem with two arms. At each time, the experimenter observes a covariate $X_t\in \mathbb{R}^3$ that  is \iid~ sampled from Uniform$[-2,2]^3$.  The outcome model only depends on the first coordinate of the covariate: given  $x=(x_1, x_2, x_3)$, arm $1$ has conditional mean $\mu_1(x)=x_1^2$, and arm $2$ has conditional mean $\mu_2(x)=2-x_1^2$; see the left panel in Figure \ref{fig:synthetic} for illustration. The observed response is perturbed by \textbf{i.i.d.} standard Gaussian noise.


\begin{figure}
    \centering
    \includegraphics[width=\textwidth]{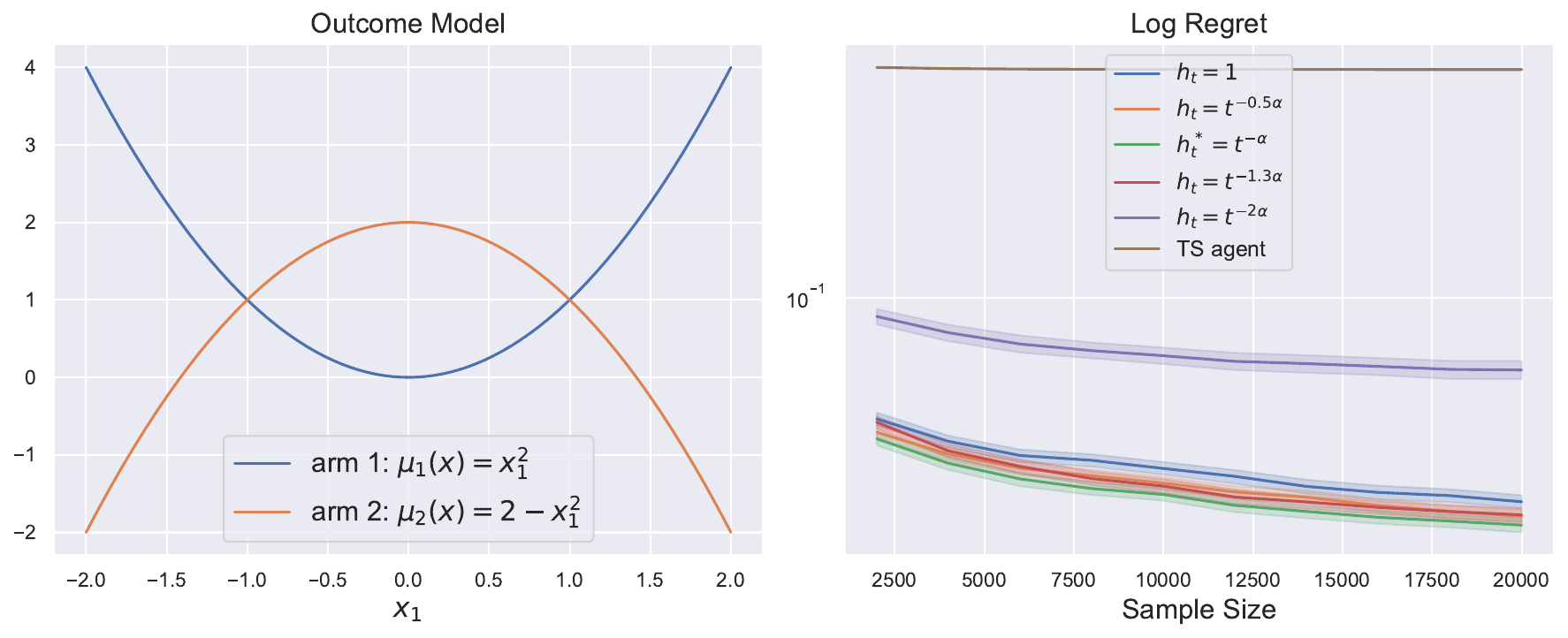}
    \caption{A synthetic example. A Thompson sampling agent collects data with assignment probability  lower bound $g_t=t^{-\alpha}/2$ ($\alpha=0.5$). The left panel demonstrates the arm outcome models, which  only depend on the first coordinate. The right panel shows the out-of-sample regrets of policy learned with different choices of weights $h_t$ and the data-collection agent. With model misspecification, the data-collection agent fails to learn the optimal policy and has the largest regret. Among different weights used for  offline  learning, the policy learned with optimal weights $h_t=t^{-\alpha}$ has the smallest regret.  Error bars are $95\%$ confidence intervals averages across $1000$ replications.}
    \label{fig:synthetic}
\end{figure}

To evaluate learned policies, we in addition sample $100,000$ observations (covariates and potential outcomes) from the same underlying distribution as test data to calculate the regret. The right panel in Figure \ref{fig:synthetic} demonstrates that given each sample size, the out-of-sample regrets (on the test data) of  the current data-collection agent (note that this agent updates its policy with growing sample size) and  of the policies obtained via Algorithm \ref{algorithm:policytree} with different choices of $h_t$.

We first compare online learning  with  offline learning  when there is model misspecification. Both the data-collection agent and the nuisance component $\hat{\mu}_t$ in the AIPW scores assume a linear outcome model, while the true $\mu$ is quadratic. The data-collection agent thus fails to learn the optimal policy; in fact, as shown in Figure \ref{fig:synthetic}, its out-of-sample regret decays much slower  than the regret of any policy learned via Algorithm \ref{algorithm:policytree}.

\textcolor{black}{Figure \ref{fig:synthetic} also demonstrates that performances of  different choices of weights are consistent with our analysis in Section \ref{section:case_study}, where the optimal weight $h^*_t=t^{-\alpha}$ (which minimizes the theoretical regret bound in Theorem \ref{theorem:upper_bound}) achieves the smallest regret. Besides,   weights $h_t=t^{-\beta}$ with $\beta<\frac{5\alpha+3}{8}$ yield  the same regret decay rate as $h_t^*$; this rate is faster than the regret obtained by setting $\beta= 2\alpha$, which choice of weights does not satisfy  Assumption \ref{assumption:weights} and thus does not have  guaranteed   regret decay in Theorem \ref{theorem:upper_bound}. 
These findings are aligned with our recommendation of weights in Algorithm \ref{algorithm}---when the assignment probability lower  bound $g_t$ is unknown, one can choose $h_t=1$ to achieve a reasonable regret decay rate, which is minimax optimal when $g_t$ decays slower than $\Theta(t^{-1})$.}

\subsection{Multi-class Classification Data}
\label{section:classification}
\begin{table}
    \centering
    \begin{tabular}{lcp{1cm}lcp{1cm}lc}
     \toprule
      Features   & Count & & Classes   & Count && Observations  & Count \\
     \midrule
     $<10$   & 24  & &  $2$   & 50 && $<5$k   & 22\\
     $\geq 10$, $<40$ & 46 &&$\geq 3$, $< 6$ & 21 && $\geq 5$k, $< 20$k & 20\\
     $\geq 40$ & 12 & &$\geq 6$ & 11  &&   $\geq 20$k & 40\\
  \bottomrule
    \end{tabular}
    \hfill
\caption{Characteristics of $82$ public OpenML datasets used for sequential classification in Section \ref{section:classification}.}
    \label{tab:dataset}
\end{table}

We adapt $82$ multi-class classification datasets from OpenML \citep{OpenML2013} into contextual bandit problems of sequential classification, following literature \cite{dudik2011doubly,dimakopoulou2017estimation,su2020doubly}. Specifically, each class represents an arm, and each feature vector denotes a covariate that is sampled uniformly from the data; the counterfactual outcomes of different arms  correspond to the one-hot encoding of the associated label. That is, we set the expected potential outcome for a given ``arm'' (a possible label) to  one if the arm is the same as the label, and zero otherwise. The observed outcome is perturbed with a standard Gaussian noise.  We can simulate an online learning algorithm with this artificial definition of arms and outcomes. 
In the experiment, we again use a floored Thompson sampling agent to collect data. See Appendix \ref{section:dataset} for the list of datasets. Table \ref{tab:dataset} summarizes the statistics of the datasets.

\begin{figure}
    \centering
    \includegraphics[width=\textwidth]{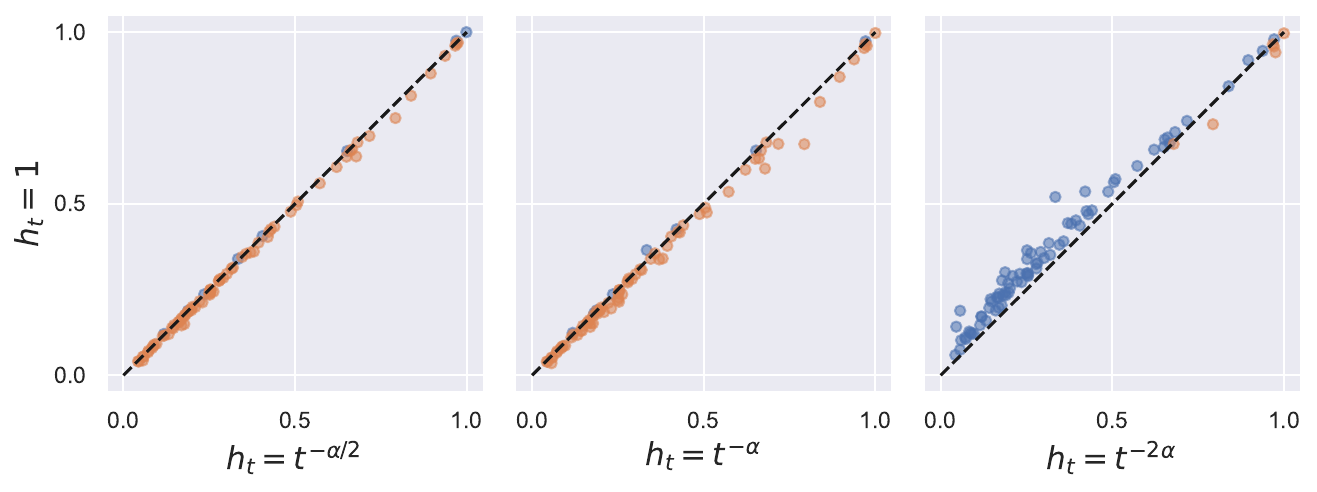}
    \caption{Regret of policy learned with different choices of weights $h_t$ on $82$ OpenML datasets. Each point denotes a dataset. A linear Thompson sampling agent with assignment probability lower bound $g_t=t^{-\alpha}/K$ ($\alpha=0.5$ and $K$ is number of classes/arms) conducts sequential classification for each dataset. The $x,y$-coordinates are the normalized regret of policies learned by setting $h_t=1$  and $h_t=t^{-\beta}$ with $\beta\in \{\frac{1}{2}\alpha, \alpha,  2\alpha \}$. Orange and blue points are those with $x<y$ and $x>y$ respectively. Policy learned with optimal weights $h_t=t^{-\alpha}$ performs best. }
    \label{fig:classification}
\end{figure}

\begin{table}
    \centering
    \begin{tabular}{p{6cm}cccc}
    \toprule
     Weighting    &  $h_t=1$ & $h_t=t^{-\alpha/2}$ & $h_t=t^{-\alpha}$ & $h_t=t^{-2\alpha}$  \\
    \midrule
       Averaged regret over $82$ datasets  & 0.313 & 0.307 & 0.301 & 0.353\\
       Median regret among $82$ datasets & 0.225 & 0.219 & 0.215 & 0.281\\
       Number of datasets on which such weighting achieves smallest regret & 5 & 10 & 64 & 3 \\
      \bottomrule
    \end{tabular}
    \caption{Summary statistics of regret of policy learned with different choices of weights $h_t$ on $82$ OpenML datasets. The assignment probability lower bound is $g_t=t^{-\alpha}/K$, where $K$ is the number of arms (classes). Optimal weighting with $h_t=t^{-\alpha}$ performs the best in terms of the average  and median of regrets over all datasets; it also achieves the smallest regret in most datasets ($64$ out of $82$.)}
    \label{tab:classification}
\end{table}

Figure \ref{fig:classification} demonstrates regret of policies learned with different choices of weights $h_t$ on various datasets. Particularly, we use uniform weights $h_t=1$ as a baseline, and compare it with other choices of weights $h_t=t^{-\beta}$ with $\beta=\{\alpha/2, \alpha, 2\alpha\}$. Each point represents a dataset, and its $x,y$ coordinates are normalized regrets using the uniform weights $h_t=1$ and the weights  $h_t=t^{-\beta}$, with the normalization term being the largest regret among all datasets.  One can see that the best performance is achieved at optimal weights $h_t=t^{-\alpha}$, while the worst one is given by weights $h_t=t^{-2\alpha}$.  We also observe that in the first two panels of Figure \ref{fig:classification}, setting $h_t=\{t^{-\alpha/2},t^{-\alpha}\}$  yields smaller regret than uniform weighting $h_t=1$ in most datasets (most points are in orange in these two panels), but the improvement is mild. \textcolor{black}{This again agrees with our analysis in Section \ref{section:case_study}, which shows   these three choices of weighting $h_t=\{1,t^{-\alpha/2},t^{-\alpha}\}$ have the same minimax regret decay  rate $O(T^{(\alpha-1)/{2}})$.} As such, we again recommend using uniform weighting when the assignment probability lower bound $g_t$ is unknown, as presented in Algorithm \ref{algorithm}.

We further provide  auxiliary summary statistics in Table \ref{tab:classification}, where for each weighting scheme, we list the average and median of regrets over all datasets; we also compute the number of datasets on which each weighting scheme achieves the smallest regret. Again, we find that optimal weighting with $h_t=t^{-\alpha}$ has the best performance.

\section{Discussion}
In this paper, we propose an approach to policy learning with adaptively collected data. Our main result is that, in the regime where assignment probabilities are bounded below by a nonincreasing positive sequence, one can leverage generalized AIPW estimators  to approximate policy value, and the policy that maximizes the estimates within a pre-specified class is asymptotically optimal. Our approach is built upon the semiparametric literature and does not require knowledge of outcome model. \textcolor{black}{Particularly, when equipped with the knowledge of lower bound on  assignment probabilities, our algorithm achieves  rate-optimal guarantees for minimax regret.}

A number of interesting research directions remain open. A natural extension is to adapt weights $h_t$ to a smaller subset of policy classes, or further, the policy currently being evaluated. Here we choose $h_t$ to offset the worst-case variance in the value estimation (and thus constant over the policy class and for any data realization); as a result, the effective sample size may not grow as fast as the actual sample size. However, if one can  adaptively choose $h_t$ with respect to the evaluated policy and the data collected,  as proposed in \cite{hadad2021confidence}, it may result in better empirical performance with improved effective sample size. Another potential line of research is to accommodate our offline learning framework to batch learning in the online setting. Most of the online learning literature that has established regret bounds is built upon functional form assumptions about the outcome model, while the generalized AIPW estimator does not require such knowledge, suggesting its potential to innovate the design of adaptive experiments.

%
%
%

\section*{Acknowledgements}
The authors would like to thank Vitor Hadad, David A. Hirshberg, Stefan Wager, and Ruoxuan Xiong for helpful discussions. The authors are also grateful for the generous support provided by Golub Capital Social Impact Lab. S.A. acknowledges generous support from the Office of Naval Research grant N00014-19-1-2468. R.Z. acknowledges generous support from the PayPal Innovation Fellowship.





\bibliographystyle{apalike}
\bibliography{reference}

\newpage
\appendix

\section{Auxiliary Definitions and Lemmas}
\label{appendix:addtional_lemmas}

\begin{definition}[Decoupled tangent sequence]
\label{def:dts}
Fix a sequence of random variables $\{Z_t\}_{t\geq 1}$ adapted to the filtration $\{\F_{t}\}_{t\geq 1}$. 
A sequence of random variables $\{Z'_t\}_{t\geq 1}$ is said to be a decoupled sequence tangent to $\{Z_t\}_{t\geq 1}$ 
if for each $t$, conditioned on $Z_1, \dots, Z_{t-1}$, the random variables $Z_t$ and $Z'_{t}$ are independent 
and identically distributed. 
\end{definition}

\textcolor{black}{We here provide more detail in constructing the decoupled tangent sequence.
First note that  $e_t$ and $\hat{\mu}_t$ are measurable w.r.t. $\F_{t-1}$, since they are fit with historical data $\calH_{t-1}$. 
We then construct $Z'_t = (W'_t, Y'_t, \{e_t(X_t;w)\}_{w\in\calW}, \{\hat{\mu}_t(X_t;w)\}_{w\in\calW})$, where $W'_t\in\calW$ is sampled from a categorical distribution specified by the assignment probabilities $\{e_t(X_t;w)\}_{w\in\calW}$; and $Y'_t=\mu(X_t; W'_t) +\epsilon'_t$ with $\epsilon'_t$ independently and identically distributed as $\epsilon_t$. 
By definition, the random variables $(Z'_t)_{t\geq 1}$ are conditionally independent given the $\big\{\big(\{e_t(X_t;w)\}_{w\in\calW}, \{\hat{\mu}_t(X_t;w)\}_{w\in\calW}\big)\big\}$, which can be used to induce the master $\sigma$-field in Proposition 6.1.5. of  \cite{victor1999decoupling}. Particularly, we note that  $(Z'_t)_{t\geq 1}$ are conditionally independent given $(Z_t)_{t\geq 1}$.}

\begin{definition}[$\tilde{\ell}_2$ distance]
\label{defn:iid_dist}
Given a realization of the covariates $x_{1:T}$, we define
\begin{enumerate}[label = (\alph*)]
        \item the inner product distance between two policies $\pi_1,\pi_2\in\Pi$ w.r.t.~$x_{1:T}$  as 
    \begin{equation*}
        \tilde{\ell}_2(\pi_1,\pi_2;x_{1:T}) \stackrel{\Delta}{=}\sqrt{\frac{\sum_{t=1}^T \big|h_tQ(x_t;\pi_1) - h_tQ(x_t;\pi_2)\big|^2}{ 4M^2\textcolor{black}{\sum_{t=1}^T h_t^2}}}.
    \end{equation*}
    \item a set $\mathcal{S}$ is a (sequential) $\eta$-cover of a policy class $\Pi$ under the $\tilde{\ell}_2$-distance 
    w.r.t.~$\{x_{1:T}\}$, if for any $\pi \in \Pi$, there exists some $s\in \mathcal{S}$ such that
    \begin{align*}
        \tilde{\ell}_2(s,\pi; x_{1:T})\le \eta.
    \end{align*}
    
    \item the covering number of a policy class $\Pi$ w.r.t.~$\{x_{1:T}\}$  as 
    \begin{align*}
    N_{\tilde{\ell}_2}(\eta, \Pi; 
    x_{1:T}) = \min\big\{|S|: S \mbox{ is an }\eta\mbox{-cover}\mbox{ of }\Pi\mbox{ under the }\tilde{\ell}_2 \mbox{ w.r.t.~} x_{1:T}\big\}.
    \end{align*}
\end{enumerate}
\end{definition}


\textcolor{black}{\begin{lemma}[Theorem 3.26 of \cite{bercu2015concentration}]
\label{lemma:bt_ineq}
Let $\{M_n\}_{n \ge 1}$ be a square integrable 
martingale w.r.t.~the filtration $\mathcal{F}_{n \ge 1}$
such that $M_0 = 0$. Define 
\begin{align*}
[M]_n \stackrel{\Delta}{=} \sum^n_{k=1} \big(M_k - M_{k-1}\big)^2,
\quad
\langle M \rangle_n \stackrel{\Delta}{=}
\sum^n_{k=1}\mathbb{E}\big[(M_k - M_{k-1})^2 \mid \mathcal{F}_{k-1}\big].
\end{align*}
Then for any positive $x$ and $y$,
\begin{align*}
\PP(M_n \ge x, [M]_n + \langle M \rangle_n \le y)
\le \exp\Big(-\frac{x^2}{2y}\Big).
\end{align*}
\end{lemma}}

\section{Proof of Main Lemmas}
\subsection{Proof of Lemma \ref{lemma:high_probability_event}}
\label{appendix:proof_of_high_probability_event}

\textcolor{black}{For notational simplicity, we define 
\[
\hgamma_t(X_t;w) = \hat{\mu}(X_t;w)
+ \frac{\one \{W_t = w\}}{e_t(X_t;w)}
\cdot \big(Y_t - \hat{\mu}(X_t;w)\big).
\]
Then for any $\pi \in \Pi$,
\begin{align}
\label{eq:mom_eq1}
&\PP\bigg(\sup_{\pi \in \Pi} \sum^T_{t=1} 
h_t^2 \cdot \big(\hgamma_t(\pi)
- Q(X_t;\pi)\big)^2 \ge \chg \bigg)
{\le} &\PP\bigg( \sum^T_{t=1} 
\sum_{w\in\calW}h_t^2\cdot \big(\hgamma_t(X_t;w)
- \mu(X_t;w)\big)^2 \ge \chg \bigg).
\end{align}
Next, we note that
\begin{align*}
\bE\Big[\big(\hgamma_t(X_t;w) - \mu(X_t;w)\big)^2 
\Biggiven \calH_{t-1}\Big]
= \bE\Big[\frac{\big(1 - e_t(X_t;w)\big)
\cdot \big(\hat{\mu}(X_t;w) - \mu(X_t;w)\big)^2 
+\sigma^2}{e_t(X_t;w)} \Biggiven \calH_{t-1}\Big]
\le \frac{5M^2}{g_t},
\end{align*}
and consequently,
\begin{align*}
\sum_{t=1}^T\sum_{w\in\calW}h_t^2\bE\Big[\big(\hgamma_t(X_t;w) - \mu(X_t;w)\big)^2 
\Biggiven \calH_{t-1}\Big] \leq \frac{\chg}{2}.
\end{align*}
Letting $D_t(X_t) \stackrel{\Delta}{=} 
\sum_{w\in\calW} \big(\hgamma_t(X_t;w) - \mu(X_t;w)\big)^2$,
we have
\begin{align}\label{eq:mom_eq2}
\eqref{eq:mom_eq1}
\le & 
\PP\bigg(\sum^T_{t=1} h_t^2 \cdot \Big(D_t(X_t) - \EE\big[D_t(X_t) \given \calH_{t-1}\big]\Big) \ge \frac{\chg}{2}\bigg)\notag\\
\le & \frac{4}{\chg^2} 
\EE\Bigg[\bigg|\sum^T_{t=1} h_t^2 \cdot \Big(D_t(X_t) - \EE\big[D_t(X_t) \given \calH_{t-1}\big]\Big)\bigg|^{2}\Bigg]\notag\\
= & \frac{4}{\chg^{2}} \cdot 
\sum^T_{t=1} h_t^4 \cdot \EE\Big[\textnormal{Var}\big(D_t(X_t) \given \calH_{t-1}\big)\Big],
\end{align}
where in the last inequality we use the property of a martingale.
Above, the variance term can be further bounded as
\begin{align*}
\EE\bigg[\Big|D_t(X_t) - \EE\big[D_t(X_t) \given \calH_{t-1}\big] \Big|^2\bigg]
\stackrel{\rm (i)}{\le} & 2 \cdot \Big(\EE\big[D_t(X_t)^{2}\big] 
+ \EE\big[\EE[D_t(X_t)\given \calH_{t-1}]^{2}\big]\Big)\\
\stackrel{\rm (ii)}{\le} & 4 \EE\big[D_t(X_t)^{2}\big]\\
= & 4 \EE\bigg[\Big(\sum_{w \in \calW} \big(\hgamma_t(X_t;w) - \mu(X_t;w)\big)^2\Big)^2\bigg]\\
\stackrel{\rm (iii)}{\le} & 4K \sum_{w \in \calW} \EE\Big[\big(\hgamma_t(X_t;w) - \mu(X_t;w)\big)^4\Big]\\
\le & \frac{48K^2M^4}{g_t^{3}},
\end{align*}
where step (i) is because $(a+b)^2 \le 2(a^2+b^2)$;
step (ii) is due to Jensen's inequality and 
step (iii) follows from the Cauchy-Schwarz inequality.
Combining everything, we have
\begin{align*}
  \eqref{eq:mom_eq2} \le 192K^2M^4 \cdot \frac{\sum^T_{t=1} h_t^4/g_t^3}{\chg^2}
  \le 2 \cdot \frac{\sum^T_{t=1}h_t^4 /g_t^3}{\big(\sum^T_{t=1} h_t^2/g_t\big)^2}.
\end{align*}}

\subsection{Proof of Lemma \ref{lemma:bound_empirical_process_prob}}
\label{appendix:proof_lemma_bound_empirical_process_prob}

Recall that \textcolor{black}{$Z_t=(W_t,Y_t, \{ e_t(X_t;w)\}_{w\in\calW}, \{\hat{\mu}_t(X_t;w)\}_{w\in\calW})$}.
Let $Z_1', Z_2',\ldots,Z_T'$ be a decoupled sequence tangent 
to $Z_1,Z_2,\ldots,Z_T$ conditional on 
$X_{1:T}$ as defined in Definition~\ref{def:dts}, 
where \textcolor{black}{$Z_t' = (W_t',Y_t',  \{ e_t(X_t;w)\}_{w\in\calW}, \{\hat{\mu}_t(X_t;w)\}_{w\in\calW})$}.
By Chebychev's inequality, for any $\eta > 0$
\begin{align*}
\mathbb{P} \bigg(\Big|\sum^T_{t=1} 
f(h_t,X_t, Z_t';\pi) - h_t \cdot Q(X_t;\pi)\Big|
\ge \frac{T \eta}{2} \Biggiven Z_{1:T},X_{1:T}\bigg) \le &
\dfrac{\mathbb{E}\Big[\Big(\sum^T_{t=1}
f(h_t,X_t,Z_t';\pi) -h_t \cdot Q(X_t;\pi)\Big)^2 
\biggiven Z_{1:T},X_{1:T}\Big]}{T^2\eta^2/4}.
\end{align*}
Conditional on $Z_{1:T}$ and $X_{1:T}$, 
$Z_s'$ is independent of $Z_t'$ for 
any $t\neq s$. Consequently, 
\begin{align*}
\mathbb{E} \bigg[\Big(\sum^T_{t=1} 
f(h_t,X_t,Z_t';\pi)-h_t \cdot Q(X_t;\pi) \Big)^2 
\biggiven Z_{1:T},X_{1:T}\bigg] 
= & \sum^T_{t=1} \mathbb{E} 
\bigg[\Big(f(h_t,X_t, Z_t';\pi) - h_t Q(X_t;\pi)\Big)^2 
\Biggiven Z_{1:T},X_{1:T} \bigg].
\end{align*}
For a fixed $t\in[T]$, we expand the summand as
\textcolor{black}{\begin{align*}
&\bE\Big[\big(f(h_t,X_t, Z_t';\pi) - h_t \cdot Q(X_t;\pi) \big)^2
\biggiven Z_{1:T},X_{1:T} \Big] \notag \\
= & h_t^2 \cdot \bE\bigg[\Big(\hat{\Gamma}'_t\big(X_t;\pi(X_t)\big) - \mu\big(X_t,\pi(X_t)\big) \Big)^2 
\biggiven Z_{1:T},X_{1:T}\bigg] \nonumber\\
= & h_t^2 \cdot \bE\Bigg[\bigg(\frac{\one \{W'_t = \pi(X_t)\}}{e_t(X_t;\pi(X_t))} 
\cdot \big(Y'_t - \mu(X_t,\pi(X_t))\big) \notag\\
  &\qquad \qquad \qquad+ \Big(1 - \frac{ \one \{W'_t = \pi(X_t)\}}{e_t(X_t;\pi(X_t))}\Big)
\cdot \Big(\hat{\mu}_t\big(X_t;\pi(X_t)\big) - \mu\big(X_t,\pi(X_t)\big)\Big) \bigg)^2 
\biggiven Z_{1:T},X_{1:T}\Bigg] \nonumber\\
    = & \frac{h_t^2 \sigma^2}{e_t(X_t;\pi(X_t))} + \frac{h_t^2\big(1-e_t(X_t;\pi(X_t))\big) 
    \big(\hat{\mu}_t(X_t;\pi(X_t)) - \mu(X_t,\pi(X_t))\big)^2}{e_t(X_t;\pi(X_t))}\notag\\
      \le & \frac{5M^2h_t^2}{g_t^2}.
\end{align*}
Combining the above calculation, we have that
\begin{align}
\label{eq:bound_tail}
    \mathbb{P}  \bigg(\Big|\sum^T_{t=1} f(h_t,X_t, Z_t';\pi) - h_t Q(X_t;\pi)\Big| \ge \frac{T \eta}{2} \Biggiven Z_{1:T},X_{1:T}\bigg) 
    \le & \dfrac{20M^2 \cdot \sum_{t=1}^Th_t^2/g_t}{\eta^2T^2} \le \frac{2C_T(h,g)}{K\eta^2T^2}.
\end{align}
Similar to the proof of Lemma~\ref{lemma:high_probability_event}, we can bound the quadratic variation.
\begin{align}\label{eq:second_order_markov}
&\PP\bigg(\sup_{\pi \in \Pi} 
\sum^T_{t=1} \Big(f(h_t,X_t,Z_t';\pi) - h_t Q(X_t; \pi)\Big)^2  
\ge \chg \,\Big|\, X_{1:T},Z_{1:T}\bigg)\notag\\
\le  & \PP\bigg(
  \sum^T_{t=1}\sum_{w\in\calW} h_t^2 \cdot \Big(
\widehat{\Gamma}\big(X_t;w\big) - \mu\big(X_t;w\big) \Big)^2 \ge \chg
\Biggiven X_{1:T},Z_{1:T}\bigg).
\end{align}}

\textcolor{black}{Next, since the conditional expectation of 
$\sum_{w\in\calW}\big(\widehat{\Gamma}'(X_t;w) - \mu(X_t;w)\big)^2$
can be bounded deterministically as below:
\begin{align*}
\sum^T_{t=1}\sum_{w\in\calW}
h_t^2 \cdot \bE\Big[\big(\hgamma'(X_t;w) - \mu(X_t;w)\big)^2 
\biggiven X_{1:T},Z_{1:T}\Big]
\le \sum^T_{t=1} \sum_{w\in\calW}
\frac{5M^2 h_t^2}{e_t(X_t,w)}
\le \frac{\chg}{2},
\end{align*}
we subsequently have
\begin{align}
\label{eq:second_order_step2}
&\eqref{eq:second_order_markov}
\le \PP \bigg\{\sum^T_{t=1} \sum_{w \in \calW}
h_t^2 \Big(\big(\hgamma'^2(X_t;w) - \mu(X_t;w)\big)^2
- \bE\big[(\hgamma'(X_t;w) - \mu(X_t;w))^2 
\given X_{1:T},Z_{1:T}\big] \Big) \ge \frac{\chg}{2}
\Biggiven X_{1:T},Z_{1:T}\bigg\}\notag\\
\le & \frac{4}{\chg^2}
\cdot \bE\Bigg[\bigg(
\sum^T_{t=1} \sum_{w \in \calW}
h_t^2 \Big(\big(\hgamma'^2(X_t;w) - \mu(X_t;w)\big)^2
- \bE\big[(\hgamma'(X_t;w) - \mu(X_t;w))^2 
\given X_{1:T},Z_{1:T}\big] \Big)\bigg)^2
\Biggiven X_{1:T},Z_{1:T} \Bigg]\notag\\
= & \frac{4}{\chg^2}
\cdot \sum^T_{t=1} \bE\Bigg[\bigg(
\sum_{w \in \calW}
h_t^2 \Big(\big(\hgamma'^2(X_t;w) - \mu(X_t;w)\big)^2
- \bE\big[(\hgamma'(X_t;w) - \mu(X_t;w))^2 
\given X_{1:T},Z_{1:T}\big] \Big)\bigg)^2
\Biggiven X_{1:T},Z_{1:T} \Bigg] \notag \\
\le &\frac{4K}{\chg^2}
\cdot \sum^T_{t=1} \sum_{w \in \calW}
h_t^4\bE\bigg[
\Big(\big(\hgamma'^2(X_t;w) - \mu(X_t;w)\big)^2
- \bE\big[(\hgamma'(X_t;w) - \mu(X_t;w))^2 
\given X_{1:T},Z_{1:T}\big] \Big)^2 
\Biggiven X_{1:T},Z_{1:T} \bigg],
\end{align}
where the equality is due to the
independence between $Z'_{1:T}$ conditional
on $X_{1:T}$ and $Z_{1:T}$. For a given $t\in [T]$
and $w\in\calW$,
\begin{align*}
&\bE\bigg[
\Big(\big(\hgamma'^2(X_t;w) - \mu(X_t;w)\big)^2
- \bE\big[(\hgamma'(X_t;w) - \mu(X_t;w))^2 
\given X_{1:T},Z_{1:T}\big] \Big)^2 
\Biggiven X_{1:T},Z_{1:T} \bigg] \\
= & \textnormal{Var}
\Big(\big(\hgamma'(X_t;w) - \mu(X_t;w))^2
\given X_{1:T},Z_{1:T} \Big) 
\le \bE\Big[\big(\hgamma'(X_t;w) - 
\mu(X_t;\mu)\big)^4\biggiven X_{1:T},Z_{1:T} \Big]
\le \frac{48M^4}{e_t^3(X_t,w)}.
\end{align*}
Consequently, we have
\begin{align}
\label{eq:second_order_step3}
\eqref{eq:second_order_step2}
\le & \frac{192 KM^4}{\chg^2} \sum^T_{t=1}
\sum_{w\in\calW} \frac{h_t^4}{e_t^3(X_t;w)}
\le \frac{192 K^2M^4}{\chg^2} \sum^T_{t=1}
\frac{h_t^4}{g_t^3}.
\end{align}
With condition that $(\sum^T_{t=1} h_t^4/g_t^3)/(\sum^T_{t=1}h_t^2/g_t)\leq 1/8$,
 the right-hand side 
of~\eqref{eq:second_order_step3} is bounded by $1/4$. 
Choose $\eta > \sqrt{8C_T(h,g)/KT^2}$, then the right-hand side of \eqref{eq:bound_tail} is bounded by $1/4$. Collectively,  
we have for any $\pi\in \Pi$,
\begin{align*}
\mathbb{P}  \bigg(\Big|\sum^T_{t=1} f(h_t,X_t, Z_t';\pi) - 
h_t \cdot Q(X_t; \pi)\Big| < \frac{T \eta}{2},~
\sup_{\tilde{\pi} \in \Pi}\sum^T_{t=1} 
\big(f(h_t,X_t,Z_t';\tilde{\pi}) - h_t Q(X_t; \tilde{\pi})\big)^2  
< \chg \Biggiven Z_{1:T}, X_{1:T}\bigg) \ge \frac{1}{2},
\end{align*}
Given $Z_1,\ldots,Z_T$ and $X_1,\ldots,X_T$, let $\pi^*$ denote the policy that maximizes 
$\big|\sum^T_{t=1} f(h_t,X_t,Z_t;\pi)- h_t Q(X_t,\pi)\big|$, from the above we have,
\begin{align*}
\PP \bigg( & \Big|\sum^T_{t=1} f(h_t,X_t,Z_t';\pi^*) 
- h_t Q(X_t, \pi^*)\Big| < \frac{T \eta}{2},\\
& \qquad \qquad \qquad
\sup_{\pi \in \Pi}\sum^T_{t=1} 
\big(f(h_t,X_t,Z'_t;\pi) 
- h_t \cdot Q(X_t;\pi) \big)^2 
< \chg \Biggiven Z_{1:T},X_{1:T} \bigg) 
\ge \frac{1}{2}.
\end{align*}
Define the events 
\begin{align*}
&A = \bigg\{\sup_{\pi\in\Pi} 
\Big|\sum^T_{t=1} f(h_t,X_t,Z_t;\pi) - h_t \cdot Q(X_t; \pi)\Big| \ge \eta T \bigg\},\\
&B = \bigg\{ \sup_{\pi \in \Pi} \sum^T_{t=1}
\Big(f(h_t,X_t,Z_t;\pi) - h_t Q(X_t;\pi)\Big)^2 \leq \chg
\bigg\}.
\end{align*}
By the tower property, we have that
\begin{align*}
\frac{1}{2} \le \mathbb{P}  
\bigg(\Big|\sum^T_{t=1} f(h_t,X_t,Z_t';\pi^*) - 
h_t \cdot Q(X_t, \pi^*)\Big| < \frac{T \eta}{2},
~\sup_{\pi \in \Pi}\sum^T_{t=1} \Big(f(h_t,X_t,Z'_t;\pi) - h_t \cdot Q(X_t;\pi) \Big)^2 
< \chg
\Biggiven A\cap B, X_{1:T}\bigg).
\end{align*}
This implies that
\begin{align*}
&\dfrac{1}{2} \mathbb{P}\bigg( 
\Big|\sum^t_{t=1}h_t\cdot \big(\hat{\Gamma}_t(X_t;\pi^*) -
Q(X_t;\pi^*)\big) \Big|\ge \eta T,
~\sup_{\pi\in\Pi} \sum^T_{t=1} h_t^2 \cdot \big(f(h_t,X_t,Z_t;\pi) - Q(X_t;\pi)\big)^2
\le \chg
\Biggiven X_{1:T} \bigg)\\
\le & \mathbb{P}  \bigg(\Big|\sum^T_{t=1} f(h_t,X_t,Z_t';\pi^*) - 
h_tQ(X_t,\pi^*)\Big| < \frac{T \eta}{2},
~\sup_{\pi \in \Pi}\sum^T_{t=1} \big(f(h_t,X_t,Z'_t;\pi) 
- h_t \cdot Q(X_t;\pi) \big)^2 
\le \chg
\Biggiven A\cap B ,X_{1:T}\bigg)\\
&~~~~\times \mathbb{P}\bigg(\Big|\sum^T_{t=1}f(h_t, X_t, Z_t;\pi^*) - h_t \cdot Q(X_t; \pi^*) \Big|\ge \eta T
,~ \sup_{\pi \in \Pi}\sum^T_{t=1} \Big(f(h_t,X_t,Z_t;\pi) - h_t \cdot Q(X_t;\pi)\Big)^2
\le \chg
\Biggiven X_{1:T}\bigg)\\
= & \mathbb{P}  \bigg(\Big|\sum^T_{t=1} f(h_t,X_t,Z_t';\pi^*) - 
h_tQ(X_t,\pi^*)\Big| < \frac{T \eta}{2},
~\sup_{\pi \in \Pi}\sum^T_{t=1} \big(f(h_t,X_t,Z'_t;\pi) 
- h_t \cdot Q(X_t;\pi) \big)^2 
\le \chg,\\
&~~~~\Big|\sum^T_{t=1}f(h_t, X_t, Z_t;\pi^*) - h_t  Q(X_t; \pi^*) \Big|\ge \eta T
,~ \sup_{\pi \in \Pi}\sum^T_{t=1} \Big(f(h_t,X_t,Z_t;\pi) - h_t \cdot Q(X_t;\pi)\Big)^2
\le \chg
\Biggiven X_{1:T}\bigg).
\end{align*}
The above can be further bounded by
\begin{align}
\label{eq:sym_step1}
\PP\bigg(& \Big|\sum^T_{t=1}f(h_t,X_t,Z_t;\pi^*) - f(h_t,X_t,Z'_t;\pi^*) \Big|
\ge \frac{\eta T}{2},
~ \sup_{\pi \in \Pi}\sum^T_{t=1} 
\big(f(h_t,X_t,Z_t;\pi) - h_tQ (X_t;\pi)\big)^2 \le \chg, \notag\\
& \qquad \qquad \sup_{\pi \in \Pi}\sum^T_{t=1} 
\big(f(h_t,X_t,Z_t';\pi) - h_t  
Q(X_t;\pi)\big)^2 \le \chg
\Biggiven X_{1:T}\bigg)\notag\\
\le \PP\bigg(& \sup_{\pi \in \Pi}\Big|\sum^T_{t=1}f(h_t,X_t,Z_t;\pi) - f(h_t,X_t,Z'_t;\pi) \Big|\ge \frac{\eta T}{2},
~ \sup_{\pi \in \Pi}\sum^T_{t=1} \big(f(h_t,X_t,Z_t;\pi) 
- h_t  Q(X_t;\pi) \big)^2\le \chg,\notag\\
& \qquad \qquad\sup_{\pi \in \Pi}\sum^T_{t=1} 
\big(f(h_t,X_t,Z_t';\pi) - h_t  Q(X_t;\pi)\big)^2 \le \chg
\Biggiven X_{1:T}  \bigg) \notag\\
  \le \PP\bigg(& \sup_{\pi \in \Pi}\Big|\sum^T_{t=1}f(h_t,X_t,Z_t;\pi) - f(h_t,X_t,Z'_t;\pi) \Big|\ge \frac{\eta T}{2},\notag\\
  & \qquad \qquad \qquad   \sup_{\pi \in \Pi}\sum^T_{t=1} \big(f(h_t,X_t,Z_t;\pi) - h_tQ(X_t;\pi)\big)^2  
+ \big(f(h_t,X_t,Z_t';\pi) - h_t Q(X_t;\pi)\big)^2  \le 2\chg
\Biggiven X_{1:T} \bigg) 
\end{align}
Since conditional on $Z_{1:T-1}$ and $X_{1:T}$, $Z_T$ is independent of and identically distributed as $Z_T'$,
\begin{align*}
\eqref{eq:sym_step1} 
= & \mathbb{E}_{\epsilon_T} \Bigg[\mathbb{P}_{Z_T,Z_T'}
\bigg(\sup_{\pi \in \Pi}\Big|\sum^{T-1}_{t=1} 
f(h_t,X_t,Z_t';\pi)-f(h_t,X_t,Z_t;\pi) + \epsilon_T \cdot
\big(f(h_t,X_T,Z_T';\pi) - f(h_t,X_T,Z_T;\pi)\big) \Big| 
> \frac{T\eta}{2},\\
& \qquad 
\sup_{\pi \in \Pi}\sum^T_{t=1} 
\big(f(h_t,X_t,Z_t;\pi) - h_t \cdot Q(X_t;\pi)\big)^2  
+ \big(f(h_t,X_t,Z_t';\pi) - h_t \cdot Q(X_t;\pi) \big)^2 \le 2\chg
\Biggiven Z_{1:T-1},Z'_{1:T-1},X_{1:T}\bigg) \Bigg]\\
\le & \sup_{z_T,z_T'} \mathbb{E}_{\epsilon_T} 
\Bigg[\mathbf{1}\bigg\{\sup_{\pi \in \Pi}\Big|\sum^{T-1}_{t=1} 
f(h_t,X_t,Z_t';\pi)-f(h_t,X_t,Z_t;\pi) + \epsilon_T 
\big(f(h_t,X_T,z_T';\pi) - f(h_t,X_T,z_T;\pi) \big) \Big| > \frac{T\eta}{2} \bigg\} \Bigg]\\
&\qquad  \qquad \qquad 
\times \mathbf{1}\bigg\{\sup_{\pi \in \Pi}\sum^T_{t=1} 
\big(f(h_t,X_t,Z_t;\pi) - h_t \cdot Q(X_t;\pi) \big)^2  
+ \big(f(h_t,X_t,Z_t';\pi) - h_tQ(X_t;\pi) \big)^2 \le 2\chg\bigg\},
\end{align*}
where we use Fubini's theorem in the inequality. Continue doing this for $T-1$, we have that
\begin{align*}
& \mathbb{P}\bigg(\sup_{\pi \in \Pi}
\Big|\sum^T_{t=1} f(h_t,X_t,Z_t';\pi)-f(h_t,X_t,Z_t;\pi) \Big| > \frac{T\eta}{2},\\
& \qquad \qquad \sup_{\pi \in \Pi}
\sum^T_{t=1} \big(f(h_t,X_t,Z_t;\pi) - h_t \cdot Q(X_t;\pi)\big)^2  
+ \big(f(h_t,X_t,Z_t';\pi) - h_t \cdot Q(X_t;\pi)\big)^2 \le 2\chg
\Biggiven Z_{1:T-2},Z'_{1:T-2},X_{1:T}\bigg) \\
= & \bE\Bigg[\p\bigg(\sup_{\pi \in \Pi}
\Big|\sum^T_{t=1} f(h_t,X_t,Z_t';\pi)-f(h_t,X_t,Z_t;\pi)) \Big| > \frac{T\eta}{2},
~\sup_{\pi \in \Pi}\sum^T_{t=1} \Big(f(h_t,X_t,Z_t;\pi) - h_tQ(X_t;\pi)\Big)^2\\ 
&\qquad \qquad \qquad + \Big(f(h_t,X_t,Z_t';\pi) - h_tQ(X_t;\pi)\Big)^2 \le 2\chg
\Biggiven Z_{1:T-1},Z'_{1:T-1},X_{1:T}\bigg)\Biggiven Z_{1:T-2},Z'_{1:T-2},X_{1:T}\Bigg]\\
\le& \bE\Bigg[\sup_{z_T,z_T'}\mathbb{E}_{\varepsilon_T} 
\bigg[\mathbf{1} \Big\{\sup_{\pi \in \Pi}\Big|\sum^{T-1}_{t=1} 
f(h_t,X_t,Z_t';\pi)-f(h_t,X_t,Z_t;\pi) 
+ \epsilon_T \big(f(h_t,X_T,z_T';\pi) - f(h_t,X_T,z_T;\pi) \big) \Big| 
> \frac{T\eta}{2} \Big\} \bigg]\\
&\qquad \times \mathbf{1} \bigg\{ 
\sup_{\pi \in \Pi}\sum^T_{t=1} \Big(f(h_t,X_t,Z_t;\pi) - h_tQ(X_t;\pi)\Big)^2  
+ \Big(f(h_t,X_t,Z_t';\pi) - h_tQ(X_t;\pi)\Big)^2 \le 2\chg \bigg\}
\Biggiven Z_{1:T-2},Z'_{1:T-2},X_{1:T}\Bigg]\\
\le &\sup_{z_{T-1},z'_{T-1}}\bE_{\epsilon_{T-1}}\sup_{z_T,z_T'}\mathbb{E}_{\epsilon_T} \Bigg[\mathbf{1}\bigg\{
\sup_{\pi \in \Pi}\Big|\sum^{T-2}_{t=1} f(h_t,X_t,Z_t';\pi)-f(h_t,X_t,Z_t;\pi) + \sum^T_{t=T-1}\epsilon_t \big(f(h_t,X_t,z_t';\pi) - 
f(h_t,X_t,z_t;\pi) \big) \Big| > \frac{T\eta}{2}\bigg\} \Bigg]\\
&\qquad \times \mathbf{1} \bigg\{ 
\sup_{\pi \in \Pi}\sum^T_{t=1} \Big(f(h_t,X_t,Z_t;\pi) - h_tQ(X_t;\pi)\Big)^2  
+ \Big(f(h_t,X_t,Z_t';\pi) - h_tQ(X_t;\pi)\Big)^2 \le 2\chg \bigg\}.
\end{align*}
Repeating the above steps for $T-2,\ldots,1$, we have that
\begin{align*}
\eqref{eq:sym_step1} \le &
\sup_{z_1,z'_1}\mathbb{E}_{\epsilon_1}\sup_{z_2,z_2'}\mathbb{E}_{\epsilon_2}\cdots 
\sup_{z_T,z_T'}\mathbb{E}_{\epsilon_T}\mathbf{1}\bigg\{\sup_{\pi \in \Pi}\Big|\sum^T_{t=1}\epsilon_t
\big(f(h_t,X_t,z_t';\pi)-f(h_t,X_t,z_t;\pi)\big) \Big| > \frac{T\eta}{2}\bigg\} \\ 
&\qquad \times \mathbf{1} \bigg\{ 
\sup_{\pi \in \Pi}\sum^T_{t=1} 
\big(f(h_t,X_t,z_t;\pi) - h_tQ(X_t;\pi)\big)^2  
+ \big(f(h_t,X_t,z_t';\pi) - h_tQ(X_t;\pi)\big)^2
\le 2\chg \bigg\}\\
\le & 2 \cdot \sup_{z_1}\mathbb{E}_{\epsilon_1}\sup_{z_2}\mathbb{E}_{\epsilon_2}\cdots 
\sup_{z_T}\mathbb{E}_{\epsilon_T}\mathbf{1}\bigg\{\sup_{\pi \in \Pi}\Big|\sum^T_{t=1}\epsilon_t
f(h_t,X_t,z_t;\pi) \Big| > \frac{T\eta}{4} \bigg\}\\
&  \qquad \times \mathbf{1} \bigg\{ 
\sup_{\pi \in \Pi}\sum^T_{t=1} \Big(f(h_t,X_t,z_t;\pi) - h_tQ(X_t;\pi)\Big)^2  
\le 2\chg \bigg\}.
\end{align*}
Note here the  $t$-th supremum is taken over $z_t\in \mathcal{Z}$.
Assuming the above supremum can all be attained---so that 
at $t$ the maximum $z^*_t$ depends on
$\epsilon_1,\ldots,\epsilon_{t-1}$---the above probability can essentially be written as
\begin{align*}
&2 \mathbb{P}_{\varepsilon}\bigg(\sup_{\pi\in\Pi}\Big| \sum^T_{t=1} \epsilon_t\cdot 
f(h_t,X_t, z^*_t(\epsilon_1,\ldots,\epsilon_{t-1};X_{1:T});\pi)
\Big|\ge \dfrac{T \eta}{4},\\
&\qquad \qquad \qquad \qquad \sup_{\pi \in \Pi}\sum^T_{t=1} 
\Big(f(h_t,X_t,z_t^*(\epsilon_1,\ldots,\epsilon_{t-1});\pi) - h_tQ(X_t;\pi)\Big)^2  
\le 2\chg \bigg)\\
\le & 2\sup_{\z}\mathbb{P}_{\epsilon}\bigg(\sup_{\pi\in\Pi}\Big| 
\sum^T_{t=1} \epsilon_t\cdot f\big(X_t,z_t(\epsilon_1\ldots,\epsilon_{t-1};X_{1:T});\pi\big) 
\Big|\ge \dfrac{T \eta}{4},\\
&\qquad \qquad \qquad \qquad 
\sup_{\pi \in \Pi}\sum^T_{t=1} \Big(f(h_t,X_t,z_t(\epsilon_1,\ldots,\epsilon_{t-1});\pi)
- h_tQ(X_t;\pi)\Big)^2  \le 2\chg \bigg),
\end{align*}
where $\z$ is a $\calZ$-valued tree of depth $T$.
If the supremum cannot be obtained, a limiting argument can be applied to show the same conclusion.
Finally, we conclude that 
\begin{align*}
& \mathbb{P}\bigg(\sup_{\pi\in\Pi} 
\Big|\sum^t_{t=1}h_t\big(\hat{\Gamma}_t(X_t;\pi) -Q(\pi)\big) \Big|
\ge \eta T,
~\sup_{\pi \in \Pi}\sum^T_{t=1} 
\Big(f(h_t,X_t,Z_t;\pi) - h_t Q(X_t;\pi)\Big)^2 \le \chg
\biggiven X_{1:T}\bigg) \\
\le & 
4\cdot \sup_{\z}\mathbb{P}_{\varepsilon}\bigg(\sup_{\pi\in\Pi}\Big| \sum^T_{t=1} \epsilon_t 
  \cdot f\big(X_t,z_t(\epsilon_1,\ldots,\epsilon_{t-1};\pi)\big) 
\Big|\ge \dfrac{T \eta}{4},\\
& \qquad \qquad \qquad \qquad 
\sup_{\pi \in \Pi}\sum^T_{t=1} \Big(f(h_t,X_t,z_t(\epsilon_1,\ldots,\epsilon_{t-1});\pi)
- h_tQ(X_t;\pi)\Big)^2  \le 2\chg\bigg).
\end{align*}}

\subsection{Proof of Lemma \ref{lemma:l2_distance}}
\label{appendix:l2_distance}
For $\eta>0$, let $N_0 = N_\ham(\eta^2, \Pi)$. Without loss of generality we assume 
$N_0<\infty$, otherwise the argument trivially holds. \textcolor{black}{Fix a realization of the covariates 
$\bx$ and a tree $\z$. 
When $\mathcal{B}_T(\bx,\z)$ is empty, the result 
trivially holds. We assume in the following 
that $\mathcal{B}_T(\bx,\z)$ is non-empty.}
Next, consider the following optimization problem:
\begin{align*}
\sup_{\pi_{a,t},\pi_{b,t}, \epsilon \in \mathcal{B}_T(\mathbf{x},\z)}
\Big|f\big(x_t,z_t(\epsilon);\pi_{a,t}\big) 
- f\big(x_t,z_t(\epsilon);\pi_{b,t}\big)\Big|.
\end{align*}
Let $(\pi_{a,t}^*,\pi_{b,t}^*, \epsilon_t^*)$ denote the policies and the Rademacher sequence  that attains the supremum (we assume without
loss of generality that the supremum is attainable; otherwise we can simply apply a limiting 
argument). 
For a positive integer $m$, define
\begin{equation*}
n_t = \Big \lceil\frac{ m \big(f(h_t,x_t, z_t(\epsilon_t^*);\pi_{a,t}^*) -
f(h_t,x_t, z_t(\epsilon^*_t);\pi_{b,t}^*)\big)^2}{16\textcolor{black}{\chg}}\Big\rceil,
\end{equation*}
and
\begin{align*}
\{\Tilde{x}_1, \dots, \Tilde{x}_n\}=\{x_1, \dots,x_1, x_2,\dots, x_2, \dots, x_T, \dots, x_T\},
\end{align*} 
where $x_t$ appears $n_t$ times. Let $\mathcal{S} = \{\pi_1,\dots, \pi_{N_0}\}$ be the set of 
$N_0$ policies that $\eta^2$-covers 
$\Pi$ w.r.t~the Hamming distance defined with~$\{\tilde{x}_1,\ldots,\tilde{x}_n\}$. 
Consider now an arbitrary policy $\pi\in\Pi$ . By the definition of a covering set,
there exists $\pi' \in \mathcal{S}$ such that
\begin{align*}
    \ham(\pi,\pi' ; \tilde{x}_{1:n}) \le \eta^2.
\end{align*}
Further fix an arbitrary Rademacher sequence \textcolor{black}{$\epsilon
\in  \mathcal{B}_T(\mathbf{x},\z)$}, and by the definition
of $n$, we have
\begin{align*}
n = & \sum_{t=1}^T \Big \lceil
\frac{m \big(f(h_t,x_t,z_t(\epsilon^*_t);\pi^*_{a,t}) -
f(h_t,x_t, z_t(\epsilon^*_t);\pi^*_{b,t})\big)^2}{16 \textcolor{black}{\chg}} \Big \rceil\\
\le & \sum_{t=1}^T 1 + \frac{3m}{16\textcolor{black}{\chg}} \cdot
\Big[\big (f(z_t(\epsilon^*_t);
\pi^*_{a,t}) - h_t \cdot Q(x_t;\pi^*_{a,t}) \big )^2
+ h_t^2 \cdot \big(Q(x_t;\pi^*_{a,t}) - Q(x_t;\pi^*_{b,t}) \big)^2\\
&\qquad\qquad \qquad \qquad
+ \big(h_t \cdot Q(x_t;\pi_{b,t}^*) - f(z_t(\epsilon^*_t);\pi^*_{b,t}\big)^2\Big]\\
\le & T + \frac{3m}{16\textcolor{black}{\chg}} 
\cdot \Big[
\sum^T_{t=1}\big (f(z_t(\epsilon^*_t);
\pi^*_{a,t}) - h_t \cdot Q(x_t;\pi^*_{a,t}) \big )^2
+ \big (f(z_t(\epsilon^*_t);
\pi^*_{b,t}) - h_t \cdot Q(x_t;\pi^*_{b,t}) \big )^2
+ 4M^2h_t^2\Big]\\
\le & m+T.
\end{align*}
 On the other hand, by the definition of the Hamming distance,  
 \begin{align*}
\ham(\pi, \pi' \mid \tilde{x}_{1:n}) = 
&\frac{1}{n}\sum_{i=1}^n \mathbf{1} 
\big\{\pi(\Tilde{x}_i)\neq \pi'(\Tilde{x}_i)\big\} \\
= &\frac{1}{n}\sum_{t=1}^T\Big \lceil \frac{m
\big(f(h_t,x_t, z_t(\epsilon^*_t);\pi^*_{a,t}) 
- f(h_t,x_t, z_t(\epsilon^*_t);\pi^*_{b,t}) \big)^2 }
{16\textcolor{black}{\chg}} \Big\rceil 
\mathbf{1}\big\{\pi(x_t)\neq \pi'(x_t)\big\} \\
\ge & \frac{m}{m+T}
\sum_{t=1}^T \frac{\big(f(h_t,x_t, z_t(\epsilon);\pi) - 
f(h_t,x_t,z_t(\epsilon);\pi')\big)^2}{16\textcolor{black}{\chg}} 
\mathbf{1}\big\{\pi(x_t)\neq \pi'(x_t)\big\}\\
\stackrel{(i)}{=} &  \frac{m}{m+T}\sum_{t=1}^T 
\frac{\big(f(h_t,x_t, z_t(\epsilon);\pi)
- f(h_t,x_t,z_t(\epsilon);\pi')\big)^2}{16\textcolor{black}{\chg}}\\
= & \frac{m}{m+T}\ell^2_{2}(\pi, \pi'; \z, \bx, \epsilon),
 \end{align*}
 where step (i) is because $f(h_t,x_t,z_t(\epsilon);\pi) = f(h_t,x_t,z_t(\epsilon);\pi')$ if
 $\pi(x_t) = \pi'(x_t)$. Additionally, with the choice of $\pi'$, we have that
 \begin{equation*}
   \eta^2\ge \ham(\pi, \pi' ; \tilde{x}_{1:n}) 
   \ge \frac{m}{m+T}\ell^2_{2}(\pi, 
      \pi';\z,x_{1:T},\epsilon),
 \end{equation*}
 which yields $\ell_2(\pi, \pi';\z,x_{1:T},\epsilon) \le \sqrt{1+T/m}\cdot\eta$. 
 In other words, for any $\epsilon\in \mathcal{B}_T(\mathbf{x},\z)$ and any $\pi\in\Pi$,
 there exists $\pi'\in S$  such that $\ell_2(\pi,\pi';\z,\bx,\epsilon) 
 \le \sqrt{1+T/m}\cdot\eta$. This says,
 \begin{equation*}
      N_2\Big(\sqrt{1+T/m} \cdot \eta,\Pi;\z, \bx\Big) 
      \le N_{\ham}(\eta^2,\Pi;\tilde{x}_{1:n}) \le N_\ham(\eta^2,\Pi).
 \end{equation*}
 Letting $m$ go to infinity, we arrive at
 \begin{align*}
     N_2(\eta,\Pi;\z, \bx) \le N_{\ham}(\eta^2,\Pi).
 \end{align*}

\subsection{Proof of Lemma \ref{lemma:bound-Rademacher-complexity}}
\label{appendix:proof_bound-Rademacher-complexity}
\subsubsection{Policy Decomposition.}
For notational brevity, we write $\epsilon_{1:T}$ as $\epsilon$ 
if no confusion can arise. 
With $J=\lceil \log_2(T) \rceil$, 
we construct a sequence of 
projection operators $A_0,A_1,\dots, A_J$,
where conditioning on realizations of $\bx, \z, \epsilon$, \textcolor{black}{each $A_j$ maps a 
policy $\pi$ to its $j$-th approximation $A_j(\pi;\epsilon)$}; as $j$ increases from
$0$ to $J$, the approximation becomes finer. 

Let $\eta_j = 2^{-j}$, and $S_j$ 
be a set of policies that $\eta_j$-covers 
$\Pi$ under the $\ell_2$ distance for $j = 0,1,\ldots,J$.
Now for any sequence $\epsilon \in \mathcal{B}_T(\bx,\z)$ and 
any policy $\pi \in \Pi$, there exists $\pi'\in S_j$ such that
$\ell_2(\pi, \pi';\z,\bx,\epsilon)\le \eta_j$.
By definition, we can choose
$S_j$ such that $|S_j|=N_2(\eta_j,\Pi;\z,\bx)$,
for any $j\in[J]$. 
We now proceed to construct the sequential approximation operators via a backward selection scheme.
For any policy $\pi\in\Pi$, fix a sequence $\epsilon \in \mathcal{B}_T(\bx,\z)$. 
For $j\in[J]$, we define the  
$j$-th approximation mapping $A_j(\pi;\epsilon)$ to be: 
\begin{align*}
     & A_j(\pi;\epsilon) = \argmin_{\pi'\in S_j}~
     \ell_2 \big(\pi, \pi';\z, \bx, \epsilon\big).
\end{align*}
By the definition of $S_j$, 
$\ell_2(\pi, A_j(\pi);\z,\bx,\epsilon) \le \eta_j$.
In particular, we define $A_0(\pi) = (0,0,\ldots,0)$. Note that $A_0(\pi)$ is not exactly an 
element in $\Pi$---it is not a policy---it however serves as a $1$-cover of $\Pi$:
for any $\pi\in\Pi$, any $\z$ and any $\bx$,
\begin{align*}
\ell_2\big(\pi,A_0(\pi);\z,\bx, \epsilon \big) 
= & \frac{\sqrt{\sum^T_{t=1} \Big( 
f\big(x_t,z_t(\epsilon);\pi\big) - 0\Big)^2}}
{4\textcolor{black}{\sqrt{\chg}}} \\
\stackrel{\rm (i)}{\le} &
\frac{\sqrt{\sum^T_{t=1}\big(f(h_t,x_t,z_t(\epsilon);\pi) - h_tQ(X_t;\pi)\big)^2}
+ \sqrt{\sum^T_{t=1} h_t^2 \cdot Q(x_t;\pi)^2}}{4\textcolor{black}{\sqrt{\chg}}}
\le 1,
\end{align*}
where step (i) is due to the triangular inequality.
Given the approximation operators, we decompose the tree Rademacher process as follows:
\begin{align}
\label{eq:decomposition}
    \sum_{t=1}^T \epsilon_t f\big(x_t, z_t(\epsilon); \pi\big) =
    &\underbracket[0.4pt]{\sum_{t=1}^T \epsilon_t \Big(\sum_{j=1}^{J}
    f\big(x_t,z_t(\epsilon); A_{j}(\pi;\epsilon)\big) 
    -f\big(x_t, z_t(\epsilon); A_{j-1}(\pi;\epsilon)\big)\Big) 
    }_{\scriptsize \text{term (i): effective}}\nonumber \\
    & + \underbracket[.4pt]{\sum_{t=1}^T \epsilon_t \Big( f\big(x_t, z_t(\epsilon); \pi\big)-f\big(x_t,  z_t(\epsilon); A_{J}(\pi;\epsilon)\big)\Big)}_{{\scriptsize \text{term (ii): negligible}}}.
\end{align}
\vspace{0.5cm}
\subsubsection{The Effective Term.}
We now focus on term (i) in~\eqref{eq:decomposition}. 
For $j \in \{1,\ldots, J\}$, any $\pi_s \in S_j$ 
and $\pi_r \in S_{j-1}$, we define a new tree 
$\w^{j,s,r}$, where \textcolor{black}{for a sequence $\epsilon \in \mathcal{B}_T(\bx,\z)$} 
and $t\in[T]$:
\begin{itemize}
\item \textcolor{black}{If there exists $\pi \in \Pi$ and $\epsilon' \in \mathcal{B}_T(\bx,\z)$ 
that matches with $\epsilon$ up to time $t$ such that 
$A_j(\pi;\epsilon') = \pi_s$ and $A_{j-1}(\pi;\epsilon') = \pi_r$, then 
        \begin{align*}
            \w^{j,s,r}_t(\epsilon) = f\big(x_t,z_t(\epsilon); \pi_s\big) - f\big(x_t,z_t(\epsilon);\pi_r\big).
        \end{align*}}
    \item Otherwise $\w_t^{j,s,r}(\epsilon)=0$.
\end{itemize}
We can check that $\w^{j,s,r}$ is well-defined.
Write the set of all such trees as $W_j \stackrel{\Delta}{=} \{\w^{j,s,r}: 1\le s\le |S_j|,
1 \le r \le |S_{j-1}| \}$ (where we enumerate the elements $S_j$ and $S_{j-1}$ in an arbitrary order); by definition, $|W_j| \le |S_j||S_{j-1}|$. With the new tree process, the sum of differences in term (i) can be bounded as
\begin{align*}
&\sum^T_{t=1} \epsilon_t \cdot  \Big(\sum^{J}_{j=1} 
f\big(x_t,z_t(\epsilon);A_j(\pi;\epsilon)\big) - f\big( x_t,z_t(\epsilon);A_{j-1}(\pi;\epsilon)\big)\Big)\\
= & \sum^J_{j=1} \sum^T_{t=1} \epsilon_t
\cdot \Big(f\big(x_t,z_t(\epsilon);A_j(\pi;\epsilon)\big)
- f\big(x_t,z_t;A_{j-1}(\pi;\epsilon)\big) \Big)\\
\le & \sum^{J}_{j=1} \sup_{\w \in W_j} \sum^T_{t=1} 
\epsilon_t \w_t(\epsilon).
\end{align*}
Next, for any $j\in \{1,\ldots,J\}$, 
any sequence $\epsilon \in \mathcal{B}_T(\bx,\z)$, 
and any $\w\in W_j$, if there exists
$\pi \in \Pi$ such that $A_j(\pi;\epsilon) 
= \pi_s$ and $A_{j-1}(\pi;\epsilon) = \pi_r$, then 
\textcolor{black}{\begin{align*}
\sqrt{\sum^T_{t=1} \w_t(\epsilon)^2} =& 
\sqrt{\sum^T_{t=1}\Big(f\big(x_t,z_t(\epsilon);A_j(\pi;\epsilon)\big) - 
f\big(x_t,z_t(\epsilon);A_{j-1}(\pi;\epsilon)\big)\Big)^2}\\
\stackrel{(i)}{\le} & \sqrt{\sum^T_{t=1}\Big(f\big(x_t,z_t(\epsilon);A_j(\pi;\epsilon)\big)
- f\big(x_t,z_t(\epsilon); \pi \big)\Big)^2}
+ \sqrt{\sum^T_{t=1}\Big(f\big(x_t,z_t(\epsilon); \pi\big) - 
f\big(x_t,z_t(\epsilon);A_{j-1}(\pi;\epsilon)\big)\Big)^2}\\
\stackrel{(ii)}{\le} & 4\sqrt{\chg}\cdot(\eta_j + \eta_{j-1}) 
= 12\sqrt{\chg}\cdot\eta_j.
\end{align*}
Above, step (i) is due to the triangle inequality and step (ii) is due to  the choice of $A_j(\pi)$ and $A_{j-1}(\pi)$. 
If no such $\pi$ exists, then there exists a time
$t_0 \le T$ such that $\w_{t_0-1}(\epsilon)\neq 0$ and  $\w_t(\epsilon) = 0$ for all $t \ge t_0$. 
When $t_0 = 1$, $\w(\epsilon)$ is trivially zero. When $t_0 >1$, by definition of $\w$, there exists a policy $\pi \in \Pi$ and  a 
sequence $\epsilon' \in \mathcal{B}_T(\bx,\z)$ that agrees with $\epsilon$ 
up to time $t_0-1$  such 
that $A_j(\pi;\epsilon') = \pi_s$ and $A_{j-1}(\pi;\epsilon') = \pi_r$.
Using the previous argument, we arrive at 
\begin{align*}
\sqrt{\sum^T_{t=1} \w_t(\epsilon)^2} \le \sqrt{\sum^T_{t=1} \w_t(\epsilon')^2}
\le 12\sqrt{\chg} \cdot \eta_j.
\end{align*}
Combining the above, we have for
any $j\in \{1, \dots, J\}$ and any 
(constant to be specified) $t_j > 0$,
\begin{align*}
& \p_{\epsilon}\bigg(\max_{\pi\in\Pi} 
\Big|\sum_{t=1}^T \epsilon_t \Big(f\big(x_t,z_t(\epsilon); A_{j}(\pi;\epsilon)\big) 
-f\big(x_t, z_t(\epsilon); A_{j-1}(\pi;\epsilon)\big)\Big) \Big|
\ge t_j, \epsilon \in \mathcal{B}_T(\bx,\z) \bigg)\\
\le &\p_{\epsilon}\Big(\max_{\w \in W_j} 
\Big| \sum^T_{t=1}\epsilon_t\w_t(\epsilon)\Big| \ge t_j,
\epsilon \in \mathcal{B}_T(\bx,\z) \Big)\\
\le & \PP_{\epsilon}
\Big(\max_{\w \in W_j} 
  \Big|\sum^T_{t=1} \epsilon_t \w_t(\epsilon) \Big| \ge t_j,
\max_{\w \in W_j} \sum^T_{t=1} \w_t(\epsilon)^2 \le 144 \chg \cdot \eta_j^2\Big)\\
\stackrel{\rm (i)}{\le} &  N_2^2(\eta_j,\Pi;\z,\bx) 
\cdot
\PP\Big(\sum^T_{t=1} \epsilon_t\w_t(\epsilon) \ge t_j, 
\max_{\w \in W_j} \sum^T_{t=1} \w_t^2(\epsilon) \le 144 \chg\cdot \eta_j^2\Big)\\
\le &  N_2^2(\eta_j,\Pi;\z,\bx) 
\cdot
\PP\Big(\sum^T_{t=1} \epsilon_t\w_t(\epsilon) \ge t_j, 
\sum^T_{t=1} \w_t^2(\epsilon) \le 144 \chg\cdot \eta_j^2\Big)\\
\stackrel{\rm (ii)}{\le} &  N_2^2(\eta_j,\Pi;\z,\bx) 
\exp\Big(- \frac{t_j^2}{576 \chg \cdot \eta_j^2}\Big)\\
\stackrel{\rm (iii)}{\le} &  N_\ham^2\big( \eta_j^2,\Pi \big) \cdot
\exp\Big(- \frac{t_j^2}{576 \chg \cdot \eta_j^2}\Big)
\end{align*}
Above, step (i) is by the union bound, step (ii)
follows from Lemma~\ref{lemma:bt_ineq} and 
step (iii) is a result of Lemma~\ref{lemma:l2_distance}.
For $\delta >0$, 
setting $t_j = 24\sqrt{\chg}\cdot\eta_j 
\sqrt{\log\big(5j^2N^2_\ham(\eta_j^2,\Pi)/3\delta \big)}$, we have
\begin{align*}
& \p\bigg(\max_{\pi\in\Pi} \Big|\sum_{t=1}^T \epsilon_t 
\Big(f\big(z_t(\epsilon); A_{j}(\pi;\epsilon)\big) 
-f\big(z_t(\epsilon); A_{j-1}(\pi;\epsilon)\big) \Big) \Big| 
\ge t_j,\epsilon \in \mathcal{B}_T(\bx,\z) \bigg) \le \frac{3\delta}{5 j^2}.
\end{align*}
Taking the union bound over $j \in \{1,\ldots,J\}$, we have
\begin{align*}
&\p\bigg(\max_{\pi\in\Pi}\Bigbar{\sum_{t=1}^T \epsilon_t \Bigbracket{ 
\sum_{j=1}^{J}f\big(x_t,z_t( \epsilon); A_{j}(\pi;\epsilon)\big) -
f\big(x_t,z_t(\epsilon); A_{j-1}(\pi;\epsilon)\big)}}\geq \sum_{j=1}^{J}t_j,
\epsilon \in \mathcal{B}_T(\bx,\z)\bigg)\\
\leq & \sum_{j=1}^{J} 
\p\Bigbracket{\max_{\pi\in\Pi}\Bigbar{\sum_{t=1}^T 
\epsilon_t \Bigbracket{ f\big(z_t(\epsilon); A_{j}(\pi;\epsilon)\big) -
f\big(z_t(\epsilon); A_{j-1}(\pi;\epsilon)\big)}}\ge t_j, 
\epsilon \in \mathcal{B}_T(\bx,\z)} \\
\le & \dfrac{3\delta}{5}\sum_{j=1}^{J} \dfrac{1}{ j^2} 
< \dfrac{3\delta}{5}\sum_{j=1}^{\infty}  \frac{1}{j^2} = \delta.
\end{align*}
Consequently, we conclude that with probability at least $1-\delta$,
on the event $\{\epsilon \notin \mathcal{B}_T(\bx,\z)\}$
\begin{align*}
&\max_{\pi\in\Pi}\Bigbar{\sum_{t=1}^T \epsilon_t \Bigbracket{ 
\sum_{j=1}^{J}f\big(x_t, z_t(\epsilon); A_{j}(\pi,\epsilon)\big)
    -f\big(x_t, z_t(\epsilon); A_{j-1}(\pi;\epsilon)\big)}}\\
    \le &   \sum_{j=1}^{J} t_j
    = \sum_{j=1}^{J}24 \sqrt{\chg} \cdot \eta_j
    \sqrt{2\log(j) + 2\log\big(N_\ham(\eta_j^2,\Pi)\big) + \log\big(5/(3\delta)\big)}\\
    \stackrel{(a)}{\le} & 24\sqrt{\chg} 
    \sum^{J}_{j=1} \eta_j\Bigbracket{\sqrt{2\log(j)} + 
    \sqrt{2\log \bigbracket{N_\ham(\eta_j^2,\Pi)}} + \sqrt{\log \bigbracket{5/{(3\delta)}}}}\\
    \stackrel{(b)}{\le} & 24\sqrt{\chg} \cdot\Big( 2\sqrt{2} + 2\sqrt{2}\kappa(\Pi) + 
    \sqrt{\log\bigbracket{5/(3\delta)}}\Big).
\end{align*}
Above, step (a) is because $\sqrt{x+y+z} \le
\sqrt{x}+\sqrt{y}+\sqrt{z}$ for $x,y,z\ge 0$; 
step (b) is due to the fact that 
$\sum \eta_j \sqrt{\log{j}} \le \sum \eta_j j \le 2$ and that $\sum \eta_j \sqrt{\log N_\ham(\eta_j^2,\Pi)} = \sum 2(\eta_j - \eta_{j+1})
\sqrt{\log N_\ham(\eta_j^2,\Pi)} \le  2\int^1_0 \sqrt{\log N_\ham(\eta_j^2,\Pi)}~\mbox{d}\epsilon = 2\kappa(\Pi)$. }

\vspace{0.5cm}
\subsubsection{The Negligible Term.}
\textcolor{black}{We proceed to show term (ii) in~\eqref{eq:decomposition} is negligible. 
On the event $\{\epsilon \in \mathcal{B}_T(\bx,z)\}$,
\begin{align*}
&\sup_{\pi\in\Pi}\Big|\sum_{t=1}^T \epsilon_t \Big( f\big(x_t, z_t(\epsilon); \pi\big) 
-f\big( x_t, z_t(\epsilon); A_J(\pi;\epsilon)\big) \Big)\Big|
\stackrel{\rm (i)}{\le} \sup_{\pi\in\Pi}\sqrt{T}
\cdot \sqrt{\sum^T_{t=1}\Big(
 f\big(x_t,z_t(\epsilon);\pi\big) - f\big(x_t,z_t(\epsilon);A_J(\pi,\epsilon)\big)\Big)^2}\\
& \stackrel{\rm (ii)}{\le}  4 \cdot 2^{-J}\cdot \sqrt{\chg T}
\le  4\sqrt{\chg}/\sqrt{T},
\end{align*}
where in step (a) we use the Cauchy-Schwarz inequality; step (b) is due to the choice of $A_J(\pi;\epsilon)$.}

\subsection{Proof of Lemma \ref{lemma:iid_concentration}}
\label{appendix:iid_concentration}
Let $\epsilon_1,\ldots,\epsilon_T$ denote a sequence of \iid~Rademacher random variables with 
$\mathbb{P}(\epsilon_t = 1) = \mathbb{P}(\epsilon_t = -1) = 1/2$ for $t\in [T]$. Let $X_1',\ldots,X_T'$
be an independent copy of $X_1,\ldots,X_T$. Then for any $\eta >0$, we have
\begin{align*}
    \mathbb{E}\bigg[\max_{\pi\in \Pi} \Big| \sum^T_{t=1} h_t Q(X_t;\pi) - h_t Q(\pi)\Big|\bigg]
    = & \mathbb{E}\bigg[\max_{\pi\in \Pi} \Big| \sum^T_{t=1} h_t Q(X_t;\pi) - h_t \mathbb{E}\big[Q(X_t';\pi)\big]\Big|\bigg]\\
    \le & \mathbb{E}\bigg[\mathbb{E}_{X_{1:T}'}\bigg[\max_{\pi\in \Pi} \Big| \sum^T_{t=1} h_t Q(X_t;\pi) - h_t Q(X_t';\pi) \Big|\bigg]\bigg]\\
    \le & \mathbb{E}\bigg[\max_{\pi\in \Pi} \Big| \sum^T_{t=1} \epsilon_t h_t \big(Q(X_t;\pi) - Q(X_t';\pi)\big) \Big| \bigg]\\
    \le & 2\cdot\mathbb{E}\bigg[\max_{\pi \in \Pi} \Big|\sum^T_{t=1} \epsilon_t h_t Q(X_t,\pi)\Big|\bigg]
\end{align*}
By the tower property, 
\begin{align*}
\mathbb{E}\bigg[\max_{\pi \in \Pi} \Big|\sum^T_{t=1} \epsilon_t h_t Q(X_t,\pi)\Big|\bigg] 
=  \mathbb{E}\bigg[\mathbb{E}_{\epsilon_{1:T}}\bigg[\max_{\pi \in \Pi} \Big|\sum^T_{t=1} \epsilon_t h_t Q(X_t,\pi)\Big|\bigg]\bigg].
\end{align*}
We now focus on bounding the inner expectation, and for brevity we write $\mathbb{E}_{\epsilon} = \mathbb{E}_{\epsilon_{1:T}}$.
To do this,  we shall use a distance between two policies (defined in Definition~\ref{defn:iid_dist}) as in the non \iid~case.
Given the $\tilde{\ell}_2$ distance, we construct the policy approximation operators $A_j:\Pi \mapsto \Pi$ via a backward selection scheme 
as before. Let $J =\lceil \log_2{T} \rceil$ and $\eta_j = 2^{-j}$ for $j = 0,1,\ldots, J$. For any covariate realization, let $S_j$
denote the smallest $\eta_j$-covering set of $\Pi$ w.r.t.~the $\tilde{\ell}_2$ distance. Then for any $\pi\in \Pi$, 
\begin{itemize}
    \item define $A_J(\pi) =\arg \min_{\pi'\in S_J} \tilde{\ell}_2(\pi, \pi';x_{1:T})$;
    \item for each $j = J-1, \dots, 1$, define $A_j(\pi) = \arg\min_{\pi'\in S_j}\tilde{\ell}_2\big(A_{j+1}(\pi), \pi';x_{1:T}\big)$.
    \item define $A_0(\pi) = (0,0,\ldots,0)$.
\end{itemize}
Similar to the non \iid~case, $A_0(\pi)$ is not exactly in $\Pi$. But $A_0(\pi)$ can serve as
a $1$-cover of $\Pi$ since for any $x_{1:T}$ and any $\pi\in \Pi$,
\begin{align*}
    \tilde{\ell}_2\big(\pi,A_0(\pi); x_{1:T}\big) = \frac{\sqrt{\big|\sum^T_{t=1}h_t^2Q^2(X_t;\pi)\big|}}{\textcolor{black}{2M\sqrt{\sum_{t=1}^T h_t^2}}} \le 1.
\end{align*}
With the sequence of policy approximation operators, we can decompose the inner expectation as follows:
\begin{align*}
    \mathbb{E}_{\epsilon}\bigg[\max_{\pi \in \Pi} \Big|\sum^T_{t=1} \epsilon_t h_t Q(X_t,\pi)\Big|\bigg]
    \le& \underbracket[0.4pt]{\mathbb{E}_{\epsilon}\bigg[\max_{\pi \in \Pi} \Big|\sum^T_{t=1} \epsilon_t \cdot \big(h_t Q(X_t,\pi) - h_t Q(X_t,A_J(\pi))\big)\Big|\bigg]}_{
    \scriptsize \textnormal{term (i): negligible}}\\
    & + \underbracket[0.4pt]{\mathbb{E}_{\epsilon}\bigg[\max_{\pi \in \Pi} \Big|\sum^T_{t=1} \epsilon_t
    \cdot \Big(\sum^J_{j=1} h_t Q\big(X_t,A_j(\pi)\big) -
    h_t Q\big(X_t,A_{j-1}(\pi)\big)\Big)\Big|\bigg]}_{\scriptsize \textnormal{term (ii): effective}}.
\end{align*}
We proceed to bound the above two terms separately.

\paragraph{The Negligible Term.}
Conditional on $X_{1:T}$, we have,
\begin{align*}
    &\max_{\pi\in\Pi} \bigg|\sum^T_{t=1} \epsilon_t h_t\Big( Q(X_t,\pi) - Q\big(X_t,A_J(\pi)\big)\Big)\bigg|\\
    \stackrel{(i)}{\le} & \max_{\pi \in \Pi}\sqrt{T} \sqrt{\textcolor{black}{\sum^T_{t=1} h_t^2}\cdot \Big(Q\big(X_t,\pi\big) -Q\big(X_t, A_J(\pi)\big) \Big)^2}
    =  2M\sqrt{T\textcolor{black}{\sum_{t=1}^T h_t^2}} \cdot \max_{\pi\in\Pi}   \tilde{\ell}_2\big(\pi,A_J(\pi);X_{1:T}\big)
    \stackrel{(ii)}{\le}  2M\textcolor{black}{\sqrt{\sum_{t=1}^T h_t^2/T}} ,
\end{align*}
where step (i) follows from Cauchy-Schwarz inequality and step (ii) is due to the choice of $A_J(\pi)$. 
Consequently,
\begin{align*}
    \mathbb{E}_{\epsilon}\bigg[\max_{\pi\in\Pi} \Big|\sum^T_{t=1} \epsilon_t h_t\Big( Q(X_t,\pi) - Q\big(X_t,A_J(\pi)\big)\Big)\Big|\bigg] 
    \le 2M\textcolor{black}{\sqrt{\sum_{t=1}^T h_t^2/T}}.
\end{align*}

\paragraph{The Effective Term.} 
For $j\in[J]$, a positive integer $k$, and a positive constant $t_{j,k}$ to be specified later,
\begin{align*}
    &\mathbb{P}_{\epsilon}\bigg(\max_{\pi\in\Pi} \Big|\sum^T_{t=1} \epsilon_th_t\Big(Q\big(X_t,A_j(\pi)\big) -Q\big(X_t,A_{j-1}(\pi)\big) \Big)\Big| \ge t_{j,k}\bigg)\\
    \stackrel{(i)}{\le} & 2|S_j|\cdot\exp\bigg(-\frac{t_{j,k}^2}{8\eta_{j-1}^2 M^2 \textcolor{black}{\sum_{t=1}^T h_t^2}} \bigg)\\
    \le & 2N_{\tilde{\ell}_2}(\eta_j,\Pi;x_{1:T}) \cdot \exp\bigg(-\frac{t_{j,k}^2}{8\eta_{j-1}^2 M^2 \textcolor{black}{\sum_{t=1}^T h_t^2}} \bigg),
\end{align*}
where the  inequality (i) is due to Hoeffding's inequality and the definition of $A_{j}(\pi)$ and $A_{j-1}(\pi)$.
Let $t_{j,k} = 2\sqrt{2}\eta_{j-1} M \sqrt{\sum_{t=1}^T h_t^2} \cdot \sqrt{\log\big(2^{k+2}j^2 N_{\tilde{\ell}_2}(\eta_j,\Pi;x_{1:T})\big)}$ and 
applying a union bound over $j\in[J]$, we obtain
\begin{align*}
    &\mathbb{P}_{\epsilon}\bigg(\max_{\pi\in\Pi}\Big|\sum^T_{t=1} \epsilon_t \Big(\sum^J_{j=1}
    h_t Q\big(X_t,A_j(\pi)\big) - h_t Q\big( X_t,A_{j-1}(\pi)\big)\Big) \Big| \ge \sum^J_{j=1} t_{j,k} \bigg)\\
    \le &\sum^J_{j=1} \mathbb{P}_{\epsilon}\bigg(\max_{\pi\in\Pi}\Big|\sum^T_{t=1} \epsilon_t \Big(
    h_t Q\big(X_t,A_j(\pi)\big) - h_t Q\big( X_t,A_{j-1}(\pi)\big)\Big) \Big| \ge t_{j,k} \bigg)\\
    \le & 2 \sum^J_{j=1} N_{\tilde{\ell}_2}(\eta_j,\Pi;x_{1:T}) \cdot \exp\bigg(-\frac{t_{j,k}^2}
    {8\eta_{j-1}^2M^2 \textcolor{black}{\sum_{t=1}^T h_t^2} } \bigg)
    \le \frac{1 }{2^{k+1} } \sum^J_{j=1} \dfrac{1}{j^2} 
    \le \frac{1}{2^k},
\end{align*}
where in the last inequality we use $\frac{\pi^2}{3}\leq 4$. As before, we connect the covering number under $\tilde{\ell}_2$-distance with that under the Hamming distance.
The connection is characterized by Lemma~\ref{lemma:tl2_distance}.
\begin{lemma}
\label{lemma:tl2_distance} 
Under Assumption \ref{assumption:dgp}, for any realization of covariates $x_{1:T}$ and for any $\eta>0$, we have $N_2(\eta, \Pi; x_{1:T})\leq N_\ham(\eta^2, \Pi)$.
\end{lemma}
We defer the proof of Lemma~\ref{lemma:tl2_distance} to Appendix \ref{appendix:proof_tl2_distance}. Using the connection, we have with probability at least $1-1/2^k$,
\begin{align*}
    &\max_{\pi\in\Pi}\Big|\sum^T_{t=1} \epsilon_t \Big(\sum^J_{j=1}
    h_t Q\big(X_t,A_j(\pi)\big) - h_t Q\big( X_t,A_{j-1}(\pi)\big)\Big) \Big| \\
    \stackrel{(i)}{\le} & \sum^J_{j=1} 
    4\sqrt{2}\eta_{j} M\textcolor{black}{\sqrt{\sum_{t=1}^T h_t^2}}\cdot \Big( \sqrt{(k+2)\log{2}} + \sqrt{2\log{j}} + \sqrt{\log{N_{\tilde{\ell}_2}\big(\eta_j,\Pi;x_{1:T}\big)}} \Big)\\
    \le & \sum^J_{j=1} 
    4\sqrt{2}\eta_{j} M\textcolor{black}{\sqrt{\sum_{t=1}^T h_t^2}}\cdot \Big( \sqrt{(k+2)\log{2}} + \sqrt{2\log{j}} + \sqrt{\log{N_{\ham}\big(\eta_j^2,\Pi;x_{1:T}\big)}} \Big)\\
    \stackrel{(ii)}{\le} & 4\sqrt{2}M\textcolor{black}{\sqrt{\sum_{t=1}^T h_t^2}} \cdot \Big(\sqrt{(k+2)\log{2}} +  2\sqrt{2} + 2\kappa(\Pi) \Big),
\end{align*}
where in step (i) we use that $\sqrt{a+b+c} \le \sqrt{a}+\sqrt{b}+\sqrt{c}$ for $a,b,c\ge 0$. Step (ii) is due to the fact that 
$\sum \eta_j \sqrt{\log{j}} \le \sum \eta_j j \le 2$ and that $\sum \eta_j \sqrt{\log N_\ham(\eta_j^2,\Pi)} = \sum 2(\eta_j - \eta_{j+1})
\sqrt{\log N_\ham(\eta_j^2,\Pi)} \le  2\int^1_0 \sqrt{\log N_\ham(\eta_j^2,\Pi)} = 2\kappa(\Pi)$. 

As a result, we can bound the expectation as:
\begin{align*}
    & \mathbb{E}_{\epsilon} \Bigg[ \max_{\pi\in\Pi}\bigg|\sum^T_{t=1} \epsilon_t \Big(\sum^J_{j=1}
    h_t Q\big(X_t,A_j(\pi)\big) - h_t Q\big( X_t,A_{j-1}(\pi)\big)\Big) \bigg|\Bigg] \\
     = & \int^{\infty}_0 \mathbb{P}_{\epsilon} \Bigg( \max_{\pi\in\Pi}\bigg|\sum^T_{t=1} \epsilon_t \Big(\sum^J_{j=1}
    h_t Q\big(X_t,A_j(\pi)\big) - h_t Q\big( X_t,A_{j-1}(\pi)\big)\Big) \bigg| > s\Bigg) \textnormal{d}s\\
    \le & 4\sqrt{2}M\textcolor{black}{\sqrt{\sum_{t=1}^T h_t^2}} \Big( 8\sqrt{\log{2}} + 4\sqrt{2} + 4 \kappa(\Pi) \Big)\\
    \le & 4\sqrt{2}M \textcolor{black}{\sqrt{\sum_{t=1}^T h_t^2}}\Big(13 +  4\kappa(\Pi)\Big),
\end{align*}
where in the last equality we use $ 8\sqrt{\log{2}} + 4\sqrt{2}< 13$. Combining the bounds for the negligible term and the effective bound, we have 
\begin{align*}
    \mathbb{E}_{\epsilon}\Big[\max_{\pi \in \Pi} \big|\sum^T_{t=1} \epsilon_t h_t Q(X_t,\pi)\big|\Big]
    \le \textcolor{black}{2M \sqrt{\sum_{t=1}^T h_t^2/T} + 4\sqrt{2}M\sqrt{\sum_{t=1}^T h_t^2}\Big(13 + 4\kappa(\Pi)\Big)}. 
\end{align*}
Therefore, 
\begin{align*}
     \mathbb{E}\bigg[\max_{\pi\in \Pi} \Big| \sum^T_{t=1} h_t Q(X_t;\pi) - h_t Q(\pi)\Big|\bigg] & \leq 2  \mathbb{E}_{\epsilon}\Big[\max_{\pi \in \Pi} \big|\sum^T_{t=1} \epsilon_t h_t Q(X_t,\pi)\big|\Big]\\
     &\leq \textcolor{black}{4M \sqrt{\sum_{t=1}^T h_t^2/T}+ 8\sqrt{2}M\sqrt{\sum_{t=1}^T h_t^2}\Big(13 + 4\kappa(\Pi)\Big)}. 
\end{align*}
Finally for any $\eta>0$,
\begin{align*}
    &\mathbb{P}\bigg(\max_{\pi\in \Pi} \big|\sum^T_{t=1} h_t \big(Q(X_t,\pi) - Q(\pi)\big) \big|
    \ge \textcolor{black}{8M\sqrt{\sum_{t=1}^T h_t^2/T} + 8\sqrt{2}M\sqrt{\sum_{t=1}^T h_t^2}\Big(13 + 4\kappa(\Pi)\Big) + \eta}\bigg)\\
    \le & \mathbb{P}\bigg(\max_{\pi\in \Pi} \big|\sum^T_{t=1}  h_t \big(Q(X_t,\pi)-Q(\pi)\big) \big| 
    - \mathbb{E}\Big[\max_{\pi\in \Pi} \big|\sum^T_{t=1}  h_t \big( Q(X_t,\pi)-Q(\pi)\big) \big|  \Big] \ge \eta\bigg)\\
    \le & \textcolor{black}{\exp\Big( -\dfrac{\eta^2 }{2M^2 \sum_{t=1}^T h_t^2}\Big)},
\end{align*}
where the last inequality is due to bounded difference inequality \citep{duchi2016lecture}. Choosing \textcolor{black}{$\eta=\sqrt{2\sum_{t=1}^T h_t^2}M\log(1/\delta)$} yields the desired result. 

\subsection{Proof of Lemma~\ref{lemma:tl2_distance}}
\label{appendix:proof_tl2_distance}
Consider $\eta>0$, and let $N_0 = N_\ham(\eta^2, \Pi)$. We assume without loss of generality that $N_0<\infty$.
Fix a realization of covariates $x_{1:T}$, and consider the following optimization problem
\begin{align*}
    \max_{\pi_{a,t}, \pi_{b,t}}~\big| h_tQ(x_t,\pi_{a,t}) - h_tQ(x_t,\pi_{b,t})\big|.
\end{align*}
Let $\pi_{a,t}^*$ and $\pi_{b,t}^*$ denote the policies that achieve the maximum (we assume without 
loss of generality that the maximum is attainable; otherwise we can simply apply a limiting argument).
We have
\begin{align*}
    \big|h_tQ(x_t,\pi_{a,t}^*) - h_tQ(x_t,\pi_{b,t}^*)\big| \le 2Mh_t.
\end{align*}
For a positive integer $m$, we define for any $t\in[T]$ that
\begin{align*}
    n_t =\Big \lceil\dfrac{m\big|h_tQ(x_t,\pi_{a,t}^*) - h_tQ(x_t,\pi_{b,t}^*)\big|^2}{4M^2 \textcolor{black}{\sum_{t=1}^T h_t^2}}\Big\rceil,
\end{align*}
and 
\begin{align*}
    \{\tilde{x}_1, \ldots, \tilde{x}_n\} = \{x_1,\ldots,x_1,x_2,\ldots,x_2,\ldots,x_T,\ldots,x_T\},
\end{align*}
where $x_t$ appears for $n_t$ times. Let $\mathcal{S} = \{\pi_1,\ldots,\pi_{N_0}\}$ denote the set of
$N_0$ policies that $\eta^2$-cover $\Pi$ w.r.t.~the Hamming distance defined under $\tilde{x}_1,\ldots,\tilde{x}_n$. 
Consider an arbitrary $\pi \in \Pi$. By the definition of a covering set, there exists $\pi' \in \mathcal{S}$, such that
\begin{align*}
    \ham(\pi,\pi' ;\tilde{x}_{1:n})\le \eta^2.
\end{align*}
By the definition of $n$, we have
\begin{align*}
    n = \sum^T_{t=1} \Big\lceil \dfrac{m\big|h_tQ(x_t,\pi_{a,t}^*) - h_tQ(x_t,\pi_{b,t}^*) \big|^2}{4M^2\textcolor{black}{\sum_{t=1}^T h_t^2}} \Big\rceil
    \le \sum^T_{t=1} \Big(\dfrac{m\big|h_tQ(x_t,\pi_{a,t}^*) - h_tQ(x_t,\pi_{b,t}^*) \big|^2}{4M^2\textcolor{black}{\sum_{t=1}^T h_t^2}} + 1\Big)
    \le m+T.
\end{align*}
On the other hand,
\begin{align*}
    \ham(\pi,\pi'; \tilde{x}_{1:n}) = & \frac{1}{n}\sum^n_{i=1} \mathbf{1}\big\{\pi(\tilde{x}_i) \neq \pi'(\tilde{x}_i)\big\}\\
    \ge & \frac{1}{m+T} \sum^T_{t=1} \frac{m\big|h_tQ(x_t,\pi_{a,t}^*) - h_tQ(x_t,\pi_{b,t}^*) \big|^2}{4M^2\textcolor{black}{\sum_{t=1}^T h_t^2}}
    \cdot\mathbf{1}\big\{\pi(x_t) \neq \pi'(x_t)\big\} \\
    \ge & \frac{1}{m+T} \sum^T_{t=1} \frac{m\big|h_tQ(x_t,\pi) - h_tQ(x_t,\pi') \big|^2}{4M^2\textcolor{black}{\sum_{t=1}^T h_t^2}}
    \cdot\mathbf{1}\big\{\pi(x_t) \neq \pi'(x_t)\big\}\\
    \ge & \frac{m}{m+T} \sum^T_{t=1} \frac{\big|h_tQ(x_t,\pi) - h_tQ(x_t,\pi') \big|^2}{4M^2\textcolor{black}{\sum_{t=1}^T h_t^2}}\\
    = & \dfrac{m}{m+T} \tilde{\ell}^2_2(\pi,\pi';x_{1:T}).
\end{align*}
Equivalent, for any $\pi\in\Pi$, there exists $\pi'\in \mathcal{S}$ such that $\tilde{\ell}_2(\pi,\pi';x_{1:T})\le \sqrt{1+T/m}\cdot \eta$.
Consequently, 
\begin{align*}
    N_{\tilde{\ell}_2}\bigg(\sqrt{\frac{m+T}{m}}\cdot \eta, \Pi;x_{1:T}\bigg) \le N_0 = N_\ham(\eta^2,\Pi). 
\end{align*}
Letting $m\rightarrow\infty$ completes the proof.

\section{Additional Details}
\subsection{Solving for the Optimal Weights}
\label{appendix:solve_h}
\textcolor{black}{The optimization problem in \eqref{eq:optimization_optimal_weights} is a convex one, and we can use the method of 
Lagrange multipliers to find the global minimum. We construct the Lagrangian as follows:
\begin{align*}
    \mathcal{L}(\tilde{h}_1,\ldots,\tilde{h}_t\lambda) 
    = \sum_{t=1}^T \tilde{h}_t^2/g_t 
    + \lambda \cdot \bbracket{\sum_{t=1}^T \tilde{h}_t-1},
\end{align*}
with the constraints $\tilde{h}_t\geq 0$ for $t\in[T]$.
The corresponding KKT condition can be written as:
\begin{align*}
    &\tilde{h}_t\ge 0, \mbox{ for } t\in[T],\\
       &\nabla_{\tilde{h}_t} \mathcal{L} =2\tilde{h}_t/g_t + \lambda = 0,\\
    &\nabla_{\lambda} \mathcal{L} = 1 -\sum_{t=1}^T\tilde{h}_t = 0.
\end{align*}
Solving the above equations, we arrive at the solution $\tilde{h}_t^*=g_t/(\sum_{s=1}^T g_s)$ for all $t\in[T]$.}

\subsection{Comparing with  \texorpdfstring{\cite{bibaut2021risk}}{Bibaut et al. (2021)}}
\label{appendix:compare_with_erm}
\textcolor{black}{We hereby provide more detailed comparison with \cite{bibaut2021risk}.  Their approach is a  variant of our algorithm, where they use uniform weights with $h_t=1$ and  IPW estimator with $\widehat{\mu}_t=0$ to estimate policy values. In general, uniform weighting with $h_t=1$ does not achieve minimax regret guarantee.   When  weights $\{h_t\}$ are fixed, different choices of nuisance component $\widehat{\mu}_t$---whether setting it to zero (IPW estimator) as in \cite{bibaut2021risk} or not as in our proposal---would not affect  regret rate. Yet this nuisance component choice  can largely influence the variance of policy value estimation,  which in turn affects the value of learned policy. In particular, suppose that  $\hat{\mu}_t \xrightarrow{a.s.}\mu_{\infty}$ for some $\mu_{\infty}(\cdot)$, our algorithm achieves a  regret bound of $\tilde{O}\big(\kappa(\Pi)\sqrt{M^2+\sup_{x,w}\|\mu_\infty(x,w)-\mu(x,w)\|^2}\cdot \frac{\sqrt{K\sum_{t=1}^T h_t^2/g_t}}{\sum_{t=1}^T h_t} +\frac{\sum^T_{t=1}h_t^4/g_t^3}{(\sum^T_{t=1}h_t^2/g_t)^2}\big)$. Thus a good estimation of $\mu$ would reduce the regret bound up to a constant factor. One possible fix to using  IPW estimator would be de-meaning the outcomes beforehand, but the effect of this data-preprocessing approach may be weakened  when there exists a fair amount of variation in outcomes across contexts and actions.}

\textcolor{black}{To further see the benefits of using nuisance components in the AIPW estimator, we compare with \cite{bibaut2021risk} empirically. 
Consider the same  data generating process  and the same policy class we provide in Section \ref{section:synthetic}. Fixing weights $h_t=1$, we compare the policies learned with the IPW estimator \citep{bibaut2021risk}, the AIPW estimator with correctly specified $\widehat{\mu}_t$ (which has  the same quadratic functional form in context as the underlying outcome model), and the AIPW estimator with mis-specified $\widehat{\mu}_t$ (which wrongly assumes the underlying outcome model to be linear).   Figure \ref{fig:compare_with_erm} shows that, as compared to using IPW estimator, policies learned with both AIPW estimators---even the one with mis-specified nuisance component---effectively  reduce the estimation variance and improve the value of learned policies noticeably.}

\begin{figure}
    \centering
    \includegraphics[width=.7\textwidth]{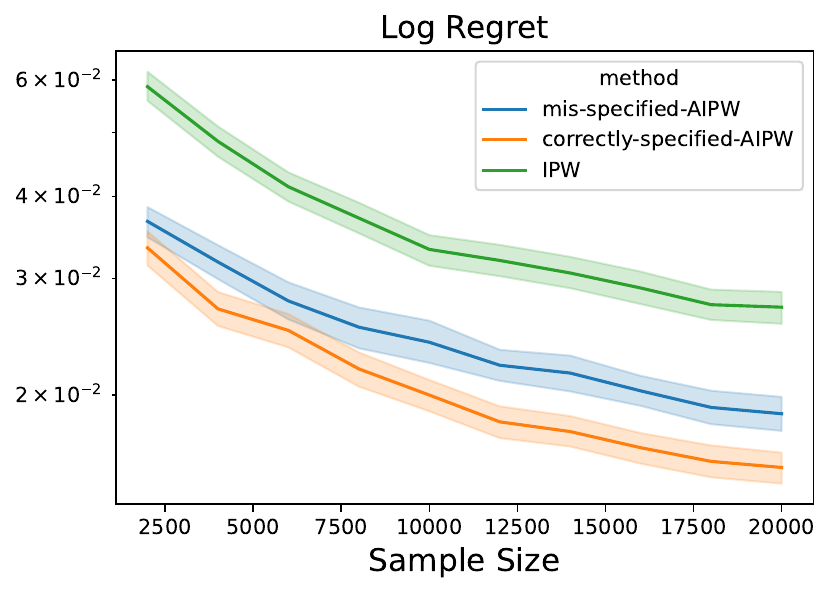}
    \caption{\textcolor{black}{Comparing with \cite{bibaut2021risk}. Policies learned with AIPW estimators  used in our algorithm  effectively reduce the variance and improve the value of learned policy.}}
    \label{fig:compare_with_erm}
\end{figure}

\subsection{Multi-class Classification Datasets}
\label{section:dataset}
The list of OpenML datasets used in Section \ref{section:classification} is
\begin{verbatim}
    'GAMETES_Epistasis_2-Way_20atts_0_4H_EDM-1_1',
       'BNG_sick_nominal_1000000', 'waveform-5000', 'BNG_credit-a_',
       'Long', 'banknote-authentication', 'page-blocks', 'BNG_mushroom_',
       'BNG(trains)', 'cmc', 'artificial-characters',
       'Click_prediction_small', 'BNG_lymph_10000_10_', 'BNG_kr-vs-kp_',
       'skin-segmentation', 'BNG_lymph_nominal_1000000_', 'allrep',
       'mfeat-morphological', 'satellite_image', 'BNG_cmc_nominal_55296_',
       'BNG_lymph_', 'jungle_chess_2pcs_endgame_elephant_elephant',
       'BNG_kr-vs-kp_5000_10_', 'BNG_vehicle_nominal_1000000_', 'wilt',
       'BNG_credit-a_nominal_1000000_', 'BNG_labor_nominal_1000000_',
       'eye_movements', 'Satellite', 'ringnorm', 'mammography',
       'delta_ailerons', 'SEA_50_', 'BNG_credit-g_nominal_1000000_',
       'PhishingWebsites', 'BNG_breast-cancer_nominal_1000000_', 'splice',
       'pendigits', 'volcanoes-a1', 'texture',
       'BNG_hepatitis_nominal_1000000_', 'SEA_50000_', 'BNG(ionosphere)',
       'BNG_breast-w_', 'BNG_colic_ORIG_nominal_1000000_',
       'cardiotocography', 'volcanoes-d4', 'BNG_labor_', 'volcanoes-b3',
       'dis', 'BNG_heart-statlog_', 'spambase', 'BNG_credit-g_',
       'jungle_chess_2pcs_endgame_panther_lion', 'ozone_level',
       'optdigits', 'BNG_cylinder-bands_nominal_1000000_', 'electricity',
       'BNG_tic-tac-toe_', 'kr-vs-kp', 'bank-marketing', 'satimage',
       'BNG(heart-statlog_nominal,1000000)', 'MagicTelescope',
       'COMET_MC_SAMPLE', 'BNG_colic_nominal_1000000_', 'volcanoes-a2',
       'yeast',
       'GAMETES_Heterogeneity_20atts_1600_Het_0_4_0_2_75_EDM-2_001',
       'BNG_cmc_', 'segment', 'BNG_vote_',
       'BNG_waveform-5000_nominal_1000000_', 'BNG_lymph_1000_10_',
       'BNG_hypothyroid_nominal_1000000_', 'mfeat-zernike', 'coil2000',
       'houses', 'eeg-eye-state', 'BNG_kr-vs-kp_5000_1_', 'car', 'pol'
\end{verbatim}

\end{document}